\documentclass[11pt]{article}
\usepackage[utf8]{inputenc}
\usepackage{amssymb}%
\usepackage{amsthm}%
\usepackage{amsmath}
\usepackage{amsfonts}
\usepackage{dsfont}
\usepackage{stmaryrd}
\usepackage[table,xcdraw]{xcolor}%
\usepackage[caption=false,font=footnotesize,labelfont=sf,textfont=sf]{subfig}
\usepackage[font=footnotesize]{caption}
\usepackage{graphicx}
\usepackage{authblk}
\usepackage{natbib}
\usepackage{hyperref}
\hypersetup{%
        colorlinks,
        breaklinks=true,
        plainpages=false,%
        citecolor=blue,
        linkcolor=blue,
        urlcolor=blue,
        bookmarksopen=true,%
        bookmarksnumbered=false,%
        bookmarksdepth=5%
}

\usepackage{fullpage}
\usepackage{xurl}

\newcommand{\norm}[1]{\left\lVert#1\right\rVert}
\newcommand{\parent}[1]{\left(#1\right)}
\newcommand{\llbrack}[1]{\left\llbracket#1\right\rrbracket}
\newcommand{\dotprod}[1]{\left<#1\right>}

\renewcommand{\brace}[1]{\left\{#1\right\}}
\renewcommand{\det}[1]{\left|#1\right|}

\newcommand{\argmin}[1]{\underset{#1}{\text{argmin}} \,}
\newcommand{\argmax}[1]{\underset{#1}{\text{argmax}} \,}
\renewcommand{\min}[1]{\underset{#1}{\text{min}} \,}
\renewcommand{\max}[1]{\underset{#1}{\text{max}} \,}

\newcommand{\pen}{\text{pen}}

\newcommand{\R}{\mathbb{R}}

\renewcommand{\L}{\mathcal{L}}
\renewcommand{\S}{\mathcal{S}}
\newcommand{\N}{\mathcal{N}}
\newcommand{\bN}{\mathbb{N}}
\newcommand{\indic}{\mathds{1}}
\newcommand{\tr}{\text{tr}}

\newcommand{\btheta}{\boldsymbol{\theta}}
\newcommand{\bthetay}{\boldsymbol{\theta}^{(y)}}
\newcommand{\bthetax}{\boldsymbol{\theta}^{(x)}}
\newcommand{\bthetaz}{\boldsymbol{\theta}^{(z)}}

\newcommand{\bthetaex}{\boldsymbol{\theta}^{(e(x))}}
\newcommand{\bthetaexk}{\boldsymbol{\theta}_k^{(e(x))}}

\newcommand{\bthetak}{\boldsymbol{\theta}_k}
\newcommand{\bthetayk}{\boldsymbol{\theta}_k^{(y)}}
\newcommand{\bthetaxk}{\boldsymbol{\theta}_k^{(x)}}
\newcommand{\bthetazk}{\boldsymbol{\theta}_k^{(z)}}

\newcommand{\balpha}{\boldsymbol{\alpha}}
\newcommand{\bbeta}{\boldsymbol{\beta}}
\newcommand{\bsigma}{\boldsymbol{\sigma}}
\newcommand{\bmu}{\boldsymbol{\mu}}

\newcommand{\bOmega}{\boldsymbol{\Omega}}
\newcommand{\btau}{\boldsymbol{\tau}}
\newcommand{\blambda}{\boldsymbol{\lambda}}
\newcommand{\bpi}{\boldsymbol{\pi}}

\newcommand{\by}{\boldsymbol{y}}
\newcommand{\bx}{\boldsymbol{x}}
\newcommand{\bz}{\boldsymbol{z}}

\providecommand{\keywords}[1]
{
  \small	
  \textbf{\textit{Keywords---}} #1
}

\newtheorem{theorem}{Theorem}
\newtheorem{proposition}[theorem]{Proposition}%
\newtheorem{lemma}[theorem]{Lemma}
\newtheorem{corollary}{Corollary}

\begin{document}

\title{Scalable Regularised Joint Mixture Models}

\author[1]{Thomas Lartigue}
\author[1,2]{Sach Mukherjee}
\affil[1]{Statistics and Machine Learning, German Center for Neurodegenerative Diseases, Bonn, Germany}
\affil[2]{MRC Biostatistics Unit, University of Cambridge, UK}
\date{}                     
\setcounter{Maxaffil}{0}
\renewcommand\Affilfont{\itshape\small}

\maketitle

\begin{abstract}%
In many applications, data can be heterogeneous in the sense of spanning latent groups with different underlying distributions. When predictive models are applied to such data the heterogeneity can affect both predictive performance and interpretability. Building on developments at the intersection of unsupervised learning and regularised regression, we propose an approach for heterogeneous data that allows joint learning of (i) explicit multivariate feature distributions, (ii) high-dimensional regression models and (iii) latent group labels, with both (i) and (ii) specific to latent groups and both elements informing (iii). The approach is demonstrably effective in high dimensions, combining data reduction for computational efficiency with a re-weighting scheme that retains key signals even when the number of features is large. We discuss in detail these aspects and their impact on modelling and computation, including EM convergence. The approach is modular and allows incorporation of data reductions and high-dimensional estimators that are suitable for specific applications. We show results from extensive simulations and real data experiments, including highly non-Gaussian data. Our results allow efficient, effective analysis of high-dimensional data in settings, such as biomedicine, where both interpretable prediction and explicit feature space models are needed but hidden heterogeneity may be a concern. 
\end{abstract}


\section{Introduction}

High-dimensional regression models -- including SVMs, the lasso (and its variants), and generic neural networks -- 
 typically assume that the distribution of responses $y$ given (high-dimensional) features $x$ is the same for all samples. However, in many contemporary applications, such homogeneity does not hold: rather, data samples may span multiple latent subgroups, with the regression function itself being subgroup-specific. In settings with data that is heterogeneous in this sense, 
 naive application of high-dimensional regression models may be suboptimal with respect to prediction and moreover can lead to parameter estimates that are misleading in terms of their scientific or domain interpretation.  
 
 Motivated by these concerns, \cite{perrakis2019latent} recently introduced a class of mixture models (``regularised joint mixtures" or RJMs) that couple together the conditional regression model for responses $y$ given features $x$ with a model on the marginal distribution of the features themselves. In their model, a latent subgroup indicator $z \in \{1, \ldots, K \}$ selects between groups, each of which has its own feature distribution and regression model. That is, for each latent subgroup $k$, these models posit a regression model of the form $p(y \mid x, z = k) = N(\beta_k^t x, \sigma_k^2)$ and a feature model $p(x \mid z = k ) = f(\cdot \mid \bthetaxk)$, with the subgroup indicators $z$ treated as latent. Estimation over all model parameters is carried out via an EM algorithm and this allows both signals -- any cluster-type signal in the feature space and potential differences in regression models -- to inform subgroup learning and estimation of subgroup-specific parameters.
 
While a number of related ideas are well known in the literature, including profile regression \citep[see][]{molitor2010bayesian, liverani2015premium}, mixtures of experts \citep[see][]{dayton1988concomitant, jacobs1991adaptive, jordan1994hierarchical, jacobs1997bias} and standard mixture models \citep{McLachlan_Peel_2000}, RJMs differ from these models in important respects. For example, the graphical model and conditional independence structure of the model is different from profile regression and the treatment of the two aspects of the model differs from simply carrying out clustering on a stacked vector of the form $(x,y)$ (RJM typically outperforms such models). As a consequence of these elements, RJM can simultaneously learn subgroup structure in a wide range of settings that challenge classical mixture models. (For a detailed comparison with related models in the literature we refer the interested reader to the RJM paper). 
 
 However, RJM does not scale to high-dimensional problems. This is due to two factors. First, in a high-dimensional regime (i.e. where 
 $x \in \mathbb{R}^p$, with $p$ large)  absent further assumptions the parameters 
 $\beta_k$ and $\bthetaxk$ are themselves high dimensional and, within the framework of RJM, must be estimated for each latent subgroup at each iteration. For computationally demanding high-dimensional estimators this quickly becomes infeasible for large $p$. Second, in the RJM model, as $p$ grows, the relative {\it influence} of the regression part of the model on the overall latent subgroup allocation shrinks: this behaviour is inevitable under the RJM model and means in practice that the influence of the response $y$ is essentially lost in the large-$p$ regime. Given that in most applications the response $y$ captures a particularly important output (and is not ``just another variable"), this can be undesirable and, depending on the nature and location of subgroup-defining signals, may mean that latent subgroup identification fails in the high-dimensional regime.

Motivated by these concerns, in this paper we present a scalable extension to the RJM model (``Scalable RJM" or S-RJM), which is geared towards high-dimensional problems. In a nutshell, the idea is to eschew working exclusively in the high-dimensional, ambient space and instead combine operations in a lower-dimensional space (via a data reduction step) with operations in the original space. This allows us to deliver high-dimensional parameter estimates but to side-step much of the computational burden of RJM. Furthermore, we introduce explicit weighting 
that controls the relative importance of the response $y$ (even as $p$ grows). The data reduction step is computationally convenient but a natural question is whether it can retain signals relevant for subgroup identification. Drawing on recent results in projection \citep{taschler2019model,lartigue2022unsupervised}, we consider a simple PCA step and show that this is indeed effective, provided the projection dimension $q <p$ is chosen appropriately (an aspect we also discuss and for which we propose an adaptive strategy). 
The underlying assumption is that identifying the cluster structure does not require operating in the ambient $p$-dimensional space but can instead be done via a suitably constructed lower dimensional summary. Scalable RJM is modular and allows users to leverage projections or embeddings that may be suitable in the application domain; to illustrate this generality we include also results using a nonlinear autoencoder. 

 Combining these elements gives an overall procedure that is capable of identifying latent structure in a range of settings and of effective high-dimensional parameter learning specific to latent subgroups, whilst scaling well computationally with the number of features $p$. We study the behaviour of the proposed Scalable RJM via extensive experiments on simulated data and real, biomedical data. These results demonstrate the effectiveness of the proposed approach and the fact that it can be used with essentially all tuning parameters set automatically (including number of groups/clusters and the target dimension $q$ of the projection).

Thus, the key contributions of this paper are as follows:
\begin{itemize}
\item We put forward methodology for joint learning of subgroup-specific regression and feature models that is scalable to very high-dimensional problems. The overall procedure in the end returns high-dimensional estimates, including sparsity patterns and feature distributions specific to latent groups (as well as estimates of the group labels) but requires only the same inputs (namely a feature matrix and response vector) as a conventional, homogeneous regression model.
\item We address the issue of balancing the relative contributions of the regression models and feature distributions to subgroup identification and allocation: this ensures that even for large $p$, the effect of the response and regression models is not lost.
\item We present extensive empirical results on both simulated and real, biomedical data. The simulated results include both Gaussian and non-Gaussian examples and the real examples involve high-dimensional, non-Gaussian data from cancer biology. These results show how the proposed methodology can be used in an ``out of the box" manner (with hyper-parameters, including projection dimension $q$ set automatically)
in complex, high-dimensional problems, where group-defining signals may lie in one of both of the regression models and feature distributions.
\end{itemize}

The code for the S-RJM algorithm, as well as the simulation framework that can reproduce the results of this paper, is publicly available at \url{https://github.com/tlartigue/Scalable-Regularised-Joint-Mixture-Models}.

The remainder of the paper is organised as follows: in Section \ref{sec:summary}, we provide a high-level summary of the problem and proposed methodology. In Section \ref{sec:model}, we present in more detail the RJM model of \cite{perrakis2019latent}, including regularisation, with some modifications needed for the current setting. In Section \ref{sec:algorithm}, we describe the EM algorithm developed to maximise the RJM likelihood, including the scalability modifications. In Section \ref{sec:theory} we prove a theoretical convergence result for the S-RJM-EM. Finally, in Section \ref{sec:experiments} we present experimental results to study the performance of the various S-RJM variants in comparison with several existing approaches. Different datasets are considered, ranging from simulated Gaussian, non-Gaussian and mixed data to real biomedical data (RNA-sequencing count data). 

\section{Summary of the model and methodology}\label{sec:summary}
In this section, we present a very high-level overview of the context and motivation for our work. We start by providing a very high level description of the mixed regression model of \cite{perrakis2019latent}. We then introduce the specific challenges that emerge in high dimensions, going on to sketch our proposed solutions to these issues.

\paragraph{Problem Statement.}
We consider a setting where we have an available data matrix $(\bx, \by) \in \R^{n \times (p+1)}$, comprising $n$ samples of a $p-$dimensional \textit{feature vector} $x$ and a corresponding real-valued \textit{response} or output $y$ (we focus on real-valued responses, but our approaches could be straightforwardly modified for discrete $y$). We want to learn a (linear) regression model between $x$ and $y$, as well as a 
multivariate distribution on $x$. Motivated by applications in biomedicine, where graphical models on features such as gene or protein levels can be important, we allow for graphical models describing covariance structure among the components of $x$. However, we suspect that there are latent subgroups in the data. A hierarchical model is necessary to properly account for these latent subgroups. Moreover, since the group labels are unknown at the outset, they need to be estimated via unsupervised learning techniques. To address this, we use a penalised mixture-likelihood maximisation problem which yields estimates for both the latent labels and the group-specific model parameters.

\paragraph{Generative model.}
Suppose there are $K \in \bN^*$ latent subgroups in the data.
To allow for subgroup-specific parameters, we consider a collection of $K$ sets of model parameters: 
$$(\bmu, \bOmega, \balpha, \bbeta, \bsigma) := (\mu_k, \Omega_k, \alpha_k, \beta_k, \sigma_k)_{k=1}^K \, ,$$ 
which will be more properly introduced in Section \ref{sec:model}. For each data point, a \textit{group-label} $z \in \llbrack{1, K}$ is generated from a probability vector $\bpi = (\pi_1, ..., \pi_K)$. An \textit{input} feature vector $x \in \R^p$ is then generated as a multivariate normal $x \sim \N(\mu_z, \Omega_z^{-1})$. Finally, a \textit{response} is generated as $y \sim \N(\alpha_z + \beta_z^t x, \sigma_z^2)$. See \figurename~\ref{fig:generative_model} for a graphical model summarizing this dependence.
\begin{figure}[tbhp]
    \centering
    \includegraphics[width=0.6\linewidth]{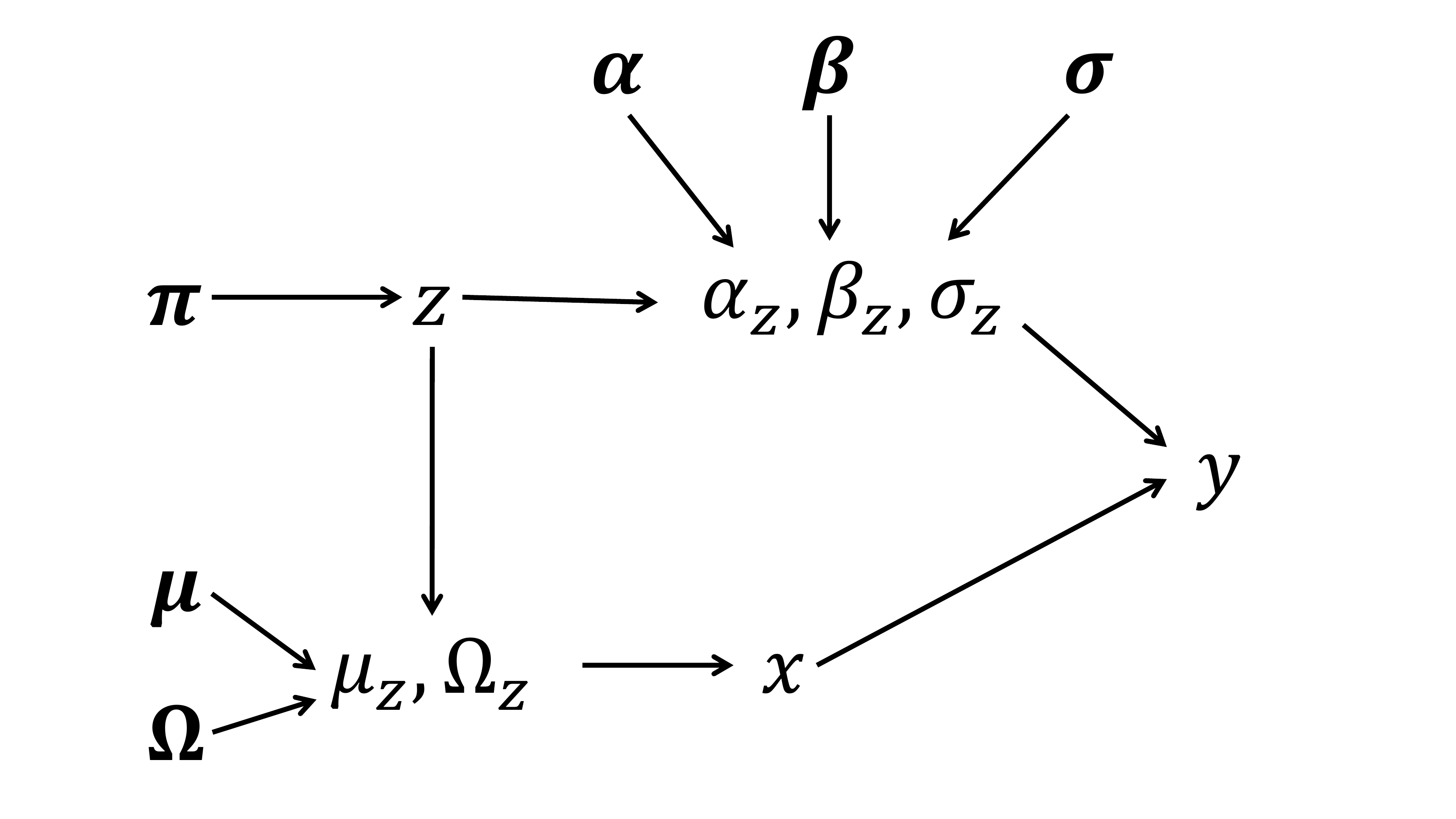}
    \caption{Generative graphical model.}
    \label{fig:generative_model}
\end{figure}

\paragraph{Regularised Joint Mixtures.}
The Regularised Joint Mixtures (RJM) approach of \citet{perrakis2019latent} proposes an algorithm to estimate the model parameter $\btheta := (\bpi, \bmu, \bOmega, \balpha, \bbeta, \bsigma)$ from an observed dataset of size $n\in \bN^*$, $\by:= (y_i)_{i=1}^n$, $\bx := (x_i)_{i=1}^n$. The labels $\bz := (z_i)_{i=1}^n$ are assumed latent (and not seen at all at the outset). Hence RJM maximises in $\btheta$ the observed likelihood $p(\by, \bx ; \btheta)$ with an EM algorithm.
\begin{equation}\label{eq:observed_likelihood}
\begin{split}
    p(\by, \bx ; \btheta) &= \prod_{i=1}^n p(y_i, x_i ; \btheta)\\
    &= \prod_{i=1}^n \sum_{k=1}^K  \varphi_1(y_i ; \alpha_k + \beta_k^t x_i, \sigma_k^2) \varphi_p(x_i; \mu_k, \Omega_k^{-1}) \tau_k  \, .
\end{split}
\end{equation}
Where $\varphi_l(\cdot \, ; \mu, \Sigma)$ is the probability density function (pdf) of a $l-$dimensional normal distribution with mean $\mu$ and covariance $\Sigma$. 

\paragraph{High-dimensional concerns.} Within the RJM model, the density on $(y, x)\in \R^{p+1}$ is jointly Gaussian. We draw attention to two issues that arise in the large $p$ regime. First, 
for the purpose of learning latent class labels 
the response $y$ 
is treated in a way 
as ``just another feature'' (alongside $p$ others). Hence, as $p \! = \! \mathrm{dim}(x)$ grows, the contribution of the term in $y$ to the likelihood \eqref{eq:observed_likelihood} becomes gradually overshadowed by the term in $x$ and for large $p$, $y$ becomes irrelevant in the latent class discovery process. On the other hand, in applied problems, $y$ is not ``just another feature" but rather an output that holds a special status in the given application. As a result, it may be desirable to ensure that the contribution of $y$ (and the associated regression models and differences between them) are not lost in the high dimensional regime. 
A second issue is computational. Since we are interested in problems, notably in biomedicine, where group-specific graphical models on $x$ are of interest, we need to estimate these group-specific models. However, in the RJM formulation this involves calls to 
high-dimensional estimators (e.g. a suitable graphical model estimator) for each group at each M-step. For large $p$ this can quickly become prohibitively expensive.

\paragraph{Scalable Regularised Joint Mixtures.} In this article, 
building on the RJM formulation,
we introduce several modifications that, among other things, 
allow for computationally efficient estimation, 
re-balancing of the contribution of $x$ and $y$ to the likelihood
and the ability to cope with non-Gaussian data. 
First, we map data $x$ into a lower dimensional space via an embedding $e : \R^p \rightarrow \R^q, q \ll p$.
This replaces the data matrix $\bx \in \R^{n \times p}$ by the reduced data $\bx^{(q)} \in \R^{n \times q}$ in the EM.
The output of this EM are the estimated labels $\widehat{\bz}=\widehat{\bz}(\bx^{(q)}, \by)$. Then, we go back to the ambient space $\R^p$, where we estimate the ambient parameter $\widehat{\btheta}$ through penalised maximum likelihood from the labelled dataset $(\bx, \by, \widehat{\bz})$. Alternatively, in the spirit of the EM algorithm, the estimated labels $\widehat{\bz}$ can be replaced by estimated label-probability for this final step. In addition to reducing the computational burden, this data reduction step also reduces the weight of $x$ in the computation of the likelihood. The assumption behind this reduction step is that working with the full dimensional $x$ is not necessary to properly identify the cluster structure, and that a lower dimension summary will carry enough group information (if the function $e$ and dimension $q$ are suitably chosen). When the latent subgroups are mean-separated, random linear projections are known to provide an effective reduction, with theoretical support via the well-known Johnson-Lindenstrauss lemma 
 \citep[see for instance][and subsequent works]{dasgupta1999learning}. 
Even without mean separation, covariance signal can still be recovered after projection using classical PCA, including in high dimensions, as studied in \citet{taschler2019model} and \citet{lartigue2022unsupervised}.
Following these works, we mainly focus on PCA as the reduction $e$ but emphasise that the scheme sketched above is modular and hence essentially any embedding suitable for a given application could be used (we show also below results from S-RJM using a nonlinear autoencoder).

As mentioned, the projection step reduces the weight of $x$ in the likelihood, but unless $q=1$ (which would usually be too restrictive) $x$ may still dominate $y$. The second modification of the algorithm we introduce is to add a re-balancing exponent $1/T$ on the likelihood in $x$. As a default, taking $T\! = \! q$ (with $q=p$ if $x$ was not projected)  ensures that there is equitable balance between $x$ and $y$ in the loss. Other exponents can be chosen in order to control the relative importance of the two terms. With data reduction $e$ 
and balancing, the new observed likelihood takes the form:
\begin{equation*}
 p(\by, \bx; \btheta) = \prod_{i=1}^n \sum_{k=1}^K  \varphi_1(y_i ; \alpha_k + \beta_k^t x_i, \sigma_k^2) \varphi_q^{\frac{1}{T}}(e(x_i); \mu_k, \Omega_k^{-1}) \tau_k  \, .
\end{equation*}
With a combination of data reduction and re-balancing, we can more efficiently deal with high-dimensional $x$, by reducing the computational complexity and enforcing a proper balance between $x$ and $y$ in the EM.
As noted above, inspired by recent results we focus mainly on PCA as a linear projection that can retain covariance signals. Even with this choice, the ideal embedding size $q$ is not obvious since lowering it can have a beneficial regularisation effect on $x$, but at the cost of a loss of information (this type of bias-variance trade-off under embedding is discussed in \cite{lartigue2022unsupervised} and we direct the interested reader to the reference for a fuller discussion).
Hence, we also propose an adaptive scheme that jointly tunes the number of latent subgroups $K$ and the embedding dimension $q$.

\section{RJM model}\label{sec:model}
In this section, we present in full details the RJM model, which we adapt from \cite{perrakis2019latent} with a some modifications in the model parametrisation and regularisation. 
\subsection{Basic model}
Let $E$ and $F$ be two sets, and $K$ be a positive integer. RJM is a parametric hierarchical model that describes the co-dependency between three random variables: an \textit{input} variable $x \in E$, a \textit{response} variable $y\in F$ and a \textit{group-label} variable $z\in \llbrack{1, K}$. The parameter space $\Theta$ of the RJM is a subset of $\R^d$ for a certain $d\in \bN^*$. Since there is a natural block structure on the parameter, the space $\Theta$ can itself be decomposed into three subsets: $\Theta = \Theta^{(x)} \times \Theta^{(y)} \times \Theta^{(z)}$. Then, the model parameter $\btheta \in \Theta$ can be written: $\btheta := (\bthetax, \bthetay, \bthetaz)$, with $\bthetax \in \Theta^{(x)}$, $\bthetay \in \Theta^{(y)}$ and $\bthetaz \in  \Theta^{(z)}$. Because of the hierarchical structure of RJM, each of block of the parameter can be also decomposed into $K$ components: $\bthetax:=(\bthetaxk)_{k=1}^K$, $\bthetay:=(\bthetayk)_{k=1}^K$ and $\bthetaz:=(\bthetazk)_{k=1}^K$. Within the RJM with parameter $\btheta$, the joint probability density function (pdf) on $(y, x, z)$ is:
\begin{equation} \label{eq:hierarchical_conditional_model}
\begin{split}
    p(y, x, z ; \btheta) &= p(y | x, z; \bthetay) p(x | z; \bthetax) p(z ; \bthetaz) \\
    &= \sum_{k=1}^K p(y | x, z=k; \bthetayk) p(x | z=k; \bthetaxk) p(z=k; \bthetazk) \indic_{z=k} \, .
\end{split}
\end{equation}
Throughout this paper, we consider that the variables $(x, y)$ are observed and that the group-label $z$ is hidden, as is the case in many applications. Hence, we are most interested in the marginalised version of \eqref{eq:hierarchical_conditional_model}, the \textit{mixture probability density function}:
\begin{equation} \label{eq:mixture_conditional_model}
    p(y, x; \btheta) = \sum_{k=1}^K p(y | x, z=k; \bthetayk) p(x | z=k; \bthetaxk) p(z=k; \bthetazk)\, .
\end{equation}
Let $n\in \bN^*$ be a sample size. Consider a dataset $(\bx, \by) \in E^n \times F^n$ made of $n$ pairs of inputs and responses: $\bx := (x_i)_{i=1}^n$ and $\by:= (y_i)_{i=1}^n$. Under the independent identically distributed (iid) assumption, we can define for any $\btheta\in \Theta$ the RJM \textit{observed likelihood} with this dataset:
\begin{equation*}
\begin{split}
    p(\by, \bx ; \btheta) &= \prod_{i=1}^n p(y_i, x_i ; \btheta)\\
    &= \prod_{i=1}^n \sum_{k=1}^K p(y_i | x_i, z_i=k; \bthetayk) p(x_i | z_i=k; \bthetaxk) p(z_i=k; \bthetazk) \, .
\end{split}
\end{equation*}
Where $\bz := (z_i)_{i=1}^n$ are the supposed labels associated with each pair $(x_i, y_i)$ in the hierarchical model. In most of this paper, we consider the particular case of a Hierarchical Conditional Gaussian Model where $y \in \R$, $\exists p \in \bN^*$ such that $x \in \R^p$, and the parameters can be written:
\begin{itemize}
    \item $\bthetay := (\balpha, \bbeta, \bsigma) := \parent{ \alpha_k, \beta_k, \sigma_k}_{k=1}^K \in \R^K \times \R^{K \times p} \times \R^K $,
    \item $\bthetax := (\bmu, \bOmega) := \parent{\mu_k, \Omega_k}_{k=1}^K \in \R^{K\times p} \times S_p^{++}(\R)^K$. Where $S_p^{++}(\R)$ is the space of real symmetric positive definite matrices, and
    \item $\bthetaz := \btau := (\tau_k)_{k=1}^K \in \S_K$. Where $\S_K := \brace{\btau \in [0, 1]^K | \sum_k \tau_k = 1}$ is the space of stochastic vectors of size $K$. 
\end{itemize}   
The conditional distributions are as follows:
\begin{itemize}
    \item $y | x, z = k$ follows a $\N(\alpha_k + \beta_k^t x, \sigma_k^2)$,
    \item $x | z = k$ follows a $\N(\mu_k, \Omega_k^{-1})$, and
    \item $z$ follows a categorical distribution on $\llbrack{1, K}$ with parameter $\btau$.
\end{itemize}
For any $q\in \bN^*$, we call $\varphi_q(. ; \mu, \Sigma)$ the probability density function of a $q-$dimensional Gaussian with mean $\mu$ and covariance $\Sigma$. Then, in the absence of any additional regularisation, we have: 
\begin{equation} \label{eq:hierarchical_conditional_gaussian_model}
    p(y, x ; \btheta) = \sum_{k=1}^K \varphi_1(y ; \alpha_k + \beta_k^t x, \sigma_k^2) \varphi_p(x; \mu_k, \Omega_k^{-1}) \tau_k \, ,
\end{equation}
and
\begin{equation} \label{eq:hierarchical_conditional_gaussian_density}
    p(\by, \bx; \btheta) = \prod_{i=1}^n \sum_{k=1}^K \varphi_1(y_i ; \alpha_k + \beta_k^t x_i, \sigma_k^2) \varphi_p(x_i; \mu_k, \Omega_k^{-1}) \tau_k \, .
\end{equation}

\subsection{Regularisation}\label{sec:regularisation}
RJM estimates the model parameter $\widehat{\btheta}$ by maximising the observed likelihood $p(\by, \bx; \btheta)$ \eqref{eq:hierarchical_conditional_gaussian_density}. As usual with mixture models the negative log-likelihood function $\btheta\mapsto - \ln p(\by, \bx; \btheta)$ is non-convex. In this work, we tackle this issue with the help of an Expectation Maximisation (EM) algorithm \citep{dempster1977maximum}. We add a penalty term $\pen(\btheta)$ to the negative log-likelihood. This regularisation is a useful tool to enforce desired properties and structure on the solution $\widehat{\btheta}$, as well as to ensure that the EM procedure is well behaved. The resulting objective of the RJM-EM can be written:
\begin{equation}\label{eq:penalised_objective}
    l(\btheta) := -\sum_{i=1}^n \ln p(y_i, x_i; \btheta) +\pen(\btheta)  \, .
\end{equation}
The final output of our method is an estimate $\widehat{\btheta}$, in the ambient space, with sparse structure on $\bbeta$ and $\bOmega$. However, running at each M-step a structured, sparse estimation of the matrix parameter $\bOmega$ is computationally costly, especially when the dimension $p$ grows. Since scalability with $p$ is the main focus of our work, we avoid costly sparsity-inducing penalties within the EM. Instead, we split the estimation in two steps:
\begin{itemize}
    \item \textbf{Lightweight iterative computations.} First, we run the EM to completion with only very cheap penalties. The goal of this step is solely to get the estimated group-labels. We conduct this step regardless of whether the EM is run on $x$ in the ambient space $\R^p$ or on an embedded $e(x)$ in a lower dimensional space $\R^q$. 
    \item \textbf{Final parameter estimation.} Once we have the estimated labels at hand, we run in the ambient space $\R^p$ a single final ``M-step", with any desired structure-inducing penalties, to get the definite estimate $\widehat{\btheta}$.
\end{itemize}
We now describe each of the regularisation terms, deferring presentation of exact details for the full algorithm to Section \ref{sec:algorithm}. First we discuss the ``cheap" penalties used within the EM, then the structure-inducing penalties used in the final estimation step. Throughout this section, we slightly abuse notation by using the function name``pen'' for different penalty functions. We do this with the understanding that the name of the argument makes it clear which penalty function we refer to.

\paragraph{Lightweight iterative computations.} 
For the block $\bthetaz = (\btau)$, we use a penalty function that prevents vanishing cluster (see Eq.~\eqref{eq:M_step_tau}) to avoid pathological behaviour of the EM. With $\rho \in \R^*_+$:
\begin{equation}\label{eq:penalty_tau}
    \pen(\tau_k|\rho) = -\rho \ln \tau_k  \, .
\end{equation}
For the bloc $\bthetay = (\balpha, \btheta, \bsigma)$, with $\gamma \in \R, \lambda_k \in \R^*_+$, we use the penalty functions:
\begin{equation}\label{eq:penalty_sigma}
     \pen(\sigma_k^2|\gamma) = (p+\gamma) \ln \sigma_k^2 \, ,
\end{equation}
\begin{equation}\label{eq:penalty_beta}
    \pen(\beta_k|\lambda_k) = \frac{\lambda_k}{\sigma_k^2} \norm{\beta_k}_1 \, .
\end{equation}
To put them in perspective with the Bayesian approach of \cite{perrakis2019latent}, these penalties can be viewed as corresponding to the following prior $p(\beta_k, \sigma_k | \lambda_k) = p(\beta_k| \sigma_k, \lambda_k) p(\sigma_k)$ on $\beta_k$ and $\sigma_k$:
\begin{equation*}\label{eq:prior_beta_sigma}
\begin{split}
    &p(\beta_k| \sigma_k, \lambda_k) = \parent{\frac{\lambda_k}{2 \sigma_k^2}}^p exp\parent{-\frac{\lambda_k}{\sigma_k^2}\norm{\beta_k}_1}\, , \\
    &p(\sigma_k) = \parent{\frac{1}{\sigma_k^2}}^{\gamma}  \, .
\end{split}
\end{equation*}
Where $p(\beta_k| \sigma_k, \lambda_k)$ is a proper density, but $p(\sigma_k)$ is not.

For the bloc $\bthetax = (\bmu, \bOmega)$, we use no regularisation. When we run RJM-EM in an embedding space of size $q\ll p$, 
further regularisation may not be needed for $\Omega_k$ (since the dimension is now greatly reduced, i.e. 
the projection already carries a regularisation effect). For the ambient space RJM-EM, we have to shrink the singular $\Omega_k$ at each step to avoid computational issues; for this purpose we use the \textit{Optimal Approximation Shrinkage} of \cite{chen2010shrinkage}. See Appendix \ref{app:oas} for more details on this shrinkage.

The total penalty can be written:
\begin{equation}\label{eq:full_penalty_function}
\begin{split}
    \pen(\btheta) &= - \ln p(\bbeta, \bsigma, \bOmega, \btau |\blambda, \delta, \rho)  \\
    &= \sum_{k=1}^K\parent{ \frac{\lambda_k}{\sigma_k^2}\norm{\beta_k}_1 + (p+\gamma) \ln \sigma_k^2  - \rho \ln \tau_k} + \text{cst} \\
    &\equiv \sum_{k=1}^K \pen(\bthetak) \, .
\end{split}
\end{equation}
We omit from $\pen(\bthetak)$ the constant term that does not intervene in the optimisation.

\paragraph{Final parameter estimation.} 
We now discuss the regularisation used in the final, ambient space, parameter estimation step.
Once the labels (or group-probabilities) have been estimated through the EM, we run one single estimation step, in the ambient $p-$dimensional space, to obtain the final estimates $\widehat{\bbeta}$ and $\widehat{\bOmega}$. At this stage, we are looking for structured sparse estimates of these parameters. Since the cluster discovery part of the problem has already been dealt with, what remains is a supervised hierarchical model estimation problem. There exists an abundant literature on this topic, both for the estimation of regression coefficients $(\beta_1, ..., \beta_K)$, see for instance Joint Lasso \citep{dondelinger2020}, as well as for the estimation of precision matrices $(\Omega_1, ...\Omega_K)$, such as the Joint Graphical Lasso \citep{danaher2014joint}, node-based Graphical Lasso \citep{mohan2014node} and others. Any of these methods can be plugged in at this stage, according to the preferences of the user. 

For the sake of simplicity, we do not consider in our experiments any method that enforces a common structure between latent subgroups through the use of joint penalties. Hence, the $K$ estimations of the group-specific parameters $\beta_k$ and $\Omega_k$ are separable and are run independently from one another. Call $\widehat{p}_{i, k}$ the estimated probability that data point $i$ belongs in the subgroup $k$ at the end of the EM. We introduce a group-weight $\widehat{w}_{i, k}$ that can be either continuous, $\widehat{w}_{i, k} := \widehat{p}_{i, k}$, or discrete, $\widehat{w}_{i, k} := \indic_{ \argmax{l= 1, ..., K} \widehat{p}_{i, l} = k }$. For the regression parameter $\beta_k$, we use the Lasso \citep{tibshirani1996regression} solution:
\begin{equation*} \label{eq:final_M_step_beta}
    \forall k \in \llbrack{1, K}, \quad \widehat{\alpha}_k, \widehat{\beta}_k := \argmin{\alpha_k, \beta_k} \sum_{i=1}^n \widehat{w}_{i, k} (y_i-\alpha_k-\beta_k^t x_i)^2 + 2\lambda_k \norm{\beta_k}_1 \, .
\end{equation*}
Here, the level of sparsity is dictated by the penalty level $\lambda_k$. For the precision matrix $\Omega_k$, we consider the Graphical Lasso \citep{friedman2008sparse} solution:
\begin{equation*} \label{eq:final_M_step_Omega}
    \forall k \in \llbrack{1, K}, \quad \widehat{\Omega}_k := \argmin{\Omega_k}\,  \tr(S_k \Omega_k) - \ln \det{\Omega_k} + \delta_k \norm{\Omega_k}_1  \,, 
\end{equation*}
where $S_k$ is the empirical covariance of group $k$. Depending on the nature of data, or our assumptions thereof, there can be different natural ways to estimate $S_k$. We consider three in this paper.
\begin{itemize}
    \item \textbf{Gaussian data.} When the data is assumed to be Gaussian-like, the most natural $S_k$ is the naive empirical covariance estimator:
    \begin{equation*} \label{eq:final_covariance_estimator}
        S_k := \frac{1}{n_k} \sum_{i=1}^n  \widehat{w}_{i, k} (x_i -\widehat{\mu}_k) (x_i -\widehat{\mu}_k)^t \, , 
    \end{equation*}
    with $n_k :=\sum_{i=1}^n \widehat{w}_{i, k}$ and $\widehat{\mu}_k :=  \frac{1}{n_k} \sum_{i=1}^n \widehat{w}_{i, k} x_i$. In this case, the estimate $\widehat{\Omega}_k$ is exactly the traditional Graphical Lasso estimator of \cite{friedman2008sparse}.
    \item \textbf{Non-Gaussian continuous data.}  When we suspect that the data may come from a more heavy-tailed distribution, we can replace $S_k$ by the non-paranormal estimator of \cite{liu2009nonparanormal} instead.
    \item \textbf{Mix of binary and continuous data.} Likewise, when we observe that the data is either fully binary or a mix of binary and continuous features, then the binary/mixed-data estimator of \cite{fan2017high} is a natural candidate for $S_k$.

\end{itemize}

\subsection{A short word on fully data-adaptive variants.}
The RJM model is a hierarchical model with $K$ subgroups. Because of the unsupervised nature of our applications, the number $K$ is actually unknown. Since RJM is a likelihood-based model, we argue that adaptive selection of $K$ can be carried out using classical criteria such as the Akaike Information Criterion \citep[AIC,][]{akaike1974new} or the Bayesian Information Criterion \citep[BIC,][]{schwarz1978estimating}. Additionally, when the input vector $x\in \R^p$ is embedded in a lower dimension $\R^q$, the hyper-parameter $q$ also needs to be set. Following \citet{taschler2019model} we set the embedding size $q$ with the help of stability metrics \citep{lange2004stability, hennig2007cluster}. In the experiments of Section \ref{sec:exp_mod_sel} we therefore propose a combined AIC-plus-stability scheme for the {\it joint} selection of $(K, q)$. Looking ahead to results, we find that the resulting adaptive RJM is the better performing of all tested methods. This scheme provides an ``out-of-the-box" solution for users and is the one we recommend in practice.

\section{The RJM-EM algorithm}\label{sec:algorithm}
In this section, we describe the EM algorithm implemented to optimised the observed likelihood of our model \eqref{eq:penalised_objective}. Then, we introduce two variants to address the problem of the scalability with ambient dimension $p$: the projected RJM, and the balanced RJM. Although these variants are introduced independently for the sake of clarity, they are fully compatible, and will be combined in the experiments.
\subsection{Basic EM}\label{sec:basic_EM}
We implement a typical EM algorithm, mostly identical to the one found in \cite{perrakis2019latent}. One crucial difference is that we change the regularisation in $\bbeta$ from $\frac{\lambda_k}{\sigma_k} \norm{\beta_k}_1$ to $\frac{\lambda_k}{\sigma_k^2} \norm{\beta_k}_1$. With this modification, we solve the M-step jointly in all parameters instead of doing a block-wise M-step. Hence our algorithm is a regular EM, not a Expectation/Conditional Maximisation (ECM). Another difference, as mentioned in the previous section, is that we do not use any $l_1$ penalty on $\bOmega$ inside the EM (but do so in the final step outside the EM).
The penalisation level $\lambda_k$ on $\beta_k$ starts with the value $\frac{n}{K}$. Then, after the first EM step where the estimated classification does not change, a new value for each $\lambda_k$ is set according to the \textit{fixed-penalty lasso} (FLasso) rule of \citep{perrakis2019latent}. We direct the reader to their work for an analysis of this procedure. Regarding the stopping criterion of the algorithm: the EM stops when the magnitude of the relative shift in both observed likelihood and parameter values go below a pre-determined small threshold, or alternatively after a pre-determined large number of steps.

The EM objective function \eqref{eq:penalised_objective} can be written:
\begin{equation*} \label{eq:objective}
    l(\btheta)
    = -\sum_{i=1}^n \ln\bigg( \sum_{k=1}^K  p(y_i, x_i, z_i=k; \bthetak) \bigg)  +\pen(\btheta) \, .
\end{equation*}
The associated function $Q$ optimised at the M-step:
\begin{equation}\label{eq:Q_function}
     Q(\btheta| \btheta^{(t)}) = -\sum_{i=1}^n \sum_{k=1}^K p_{i, k}^{(t)} \ln p(y_i, x_i, z_i=k; \bthetak) +\pen(\btheta) \, .
\end{equation}
The E-step is as in \cite{perrakis2019latent}:
\begin{equation} \label{eq:E_step}
    p_{i, k}^{(t)} := p(z_i=k|y_i, x_i;\widehat{\btheta}^{(t)}) =  \frac{p\parent{y_i| x_i,z_i=k; \widehat{\btheta}_k^{(y, t)}} p\parent{x_i|z_i=k; \widehat{\btheta}_k^{(x, t)}} \widehat{\tau}_k^{(t)}}{\sum_{l=1}^K p\parent{y_i| x_i,z_i=l; \widehat{\btheta}_l^{(y, t)}} p\parent{x_i|z_i=l;  \widehat{\btheta}_l^{(x, t)}} \widehat{\tau}_l^{(t)}} \, .
\end{equation}
The M-step is $\btheta^{(t+1)}= \argmin{\btheta} Q(\btheta| \btheta^{(t)})$. When the penalty is separable into blocks: $pen(\btheta) = \sum_{k=1}^K \big(\pen(\bthetayk) + \pen(\bthetaxk) + \pen(\bthetazk)\big)$, we have: 
\begin{equation*}
\begin{split}
    Q(\btheta| \btheta^{(t)}) = \sum_{k=1}^K \Bigg(&-\sum_{i=1}^n p_{i, k}^{(t)} \ln p(y_i | x_i, z_i=k; \bthetayk) + \pen(\bthetayk)\\
    &-\sum_{i=1}^n p_{i, k}^{(t)}\ln p(x_i | z_i=k; \bthetaxk)  + \pen(\bthetaxk) \\
    &-\sum_{i=1}^n p_{i, k}^{(t)}\ln \tau_k+ \pen(\bthetazk)\Bigg) \, .
\end{split}
\end{equation*}
Then, the M-step can be decomposed in several independent optimisation problems. In the MoCG case, the M-step can be written:
\begin{gather*}
    \forall k \in \llbrack{1, K}, \quad \min{\alpha_k, \beta_k} \sum_{i=1}^n p_{i, k}^{(t)} (y_i-\alpha_k-\beta_k^t x_i)^2 + 2\lambda_k \norm{\beta_k}_1 \, ,\\
    \forall k \in \llbrack{1, K}, \quad \min{\sigma_k}  \frac{\sum_{i=1}^n p_{i, k}^{(t)}(y_i-\widehat{\alpha}_k-\widehat{\beta}_k^t x_i)^2+2 \lambda_k \norm{\widehat{\beta}_k}_1}{\sigma_k^2} +(n_k^{(t)}+2p+2\gamma) \ln \sigma_k^2  \, ,\\
    \mu_k \text{: Gaussian MLE}\, ,\\
    \Omega_k \text{: Gaussian MLE + \textit{Optimal Approximation Shrinkage}}\, ,\\
    \max{\btau \in \S_K} \sum_{k=1}^K (n_k^{(t)} +\rho) \ln \tau_k \, . 
\end{gather*}
The M-step is a exact maximisation, not a block coordinate update. Indeed, even though the solution $\widehat{\sigma}_k$ depends on $(\widehat{\alpha}_k, \widehat{\beta}_k)$, $\widehat{\alpha}_k$ and $\widehat{\beta}_k$ are not themselves dependent on $\widehat{\sigma}_k$. Hence, the joint maximisation in $(\alpha_k, \beta_k, \alpha_k)$ is possible by first computing $\widehat{\alpha}_k, \widehat{\beta}_k$ and then $\widehat{\sigma}_k$ from them. The solutions to the M-step are:
\begin{gather}
    \widehat{\alpha}_k, \widehat{\beta}_k := \text{Lasso}(\by, \bx, \text{weights} = \{p_{i, k}^{(t)}\}_{i=1}^n, \text{penalty} = 2 \lambda_k) \, , \label{eq:M_step_alpha_beta}\\
    \widehat{\sigma}_k^2 := \frac{\sum_{i=1}^n p_{i, k}^{(t)}(y_i-\widehat{\alpha}_k-\widehat{\beta}_k^t x_i)^2+2 \lambda_k \norm{\widehat{\beta}_k}_1}{n_k^{(t)}+2p+2\gamma} \, , \label{eq:M_step_sigma} \\
    \widehat{\mu}_k := \frac{\sum_{i=1}^n p_{i, k}^{(t)} x_i}{n_k^{(t)}}  \, ,  \label{eq:M_step_mu}\\
    \widehat{\Omega}_k^{-1} = (1-\hat{\delta}) S_k  + \hat{\delta} \frac{\tr(S_k)}{p} I_P\, , \text{ with } S_k := \frac{\sum_{i=1}^n p_{i, k}^{(t)} (x_i-\widehat{\mu}_k) (x_i-\widehat{\mu}_k)^t}{n_k^{(t)}}  \, , \label{eq:M_step_Omega}\\
    \widehat{\tau}_k := \frac{n_k^{(t)}+\rho}{n+K\rho} \, . \label{eq:M_step_tau}
\end{gather}
Where in Eq.~\eqref{eq:M_step_Omega}, the coefficient $\hat{\delta}$ is the \textit{Optimal Approximation Shrinkage} defined from $S_k$, $p$ and $n$ \cite[Eq. 23]{chen2010shrinkage}.

\subsection{Projected-EM} \label{sec:projected_EM}
In this section, we introduce the \textit{projected} variant of RJM-EM, which is meant to make the whole procedure scalable with the dimension $p$.

\paragraph{Concept.} As the number of features $p$ in the \textit{input} variable $x$ grows, so does the computational burden of 
covariance-related estimation in the M-step \eqref{eq:M_step_Omega}. 
Additionally, with a widening gap between the sample size $n$ and the dimension $p$, a stronger regularisation on $\bOmega$ can be necessary to avoid pathological computations. To address these issues, we introduce a \textit{projected} variant of the RJM-EM. In a similar fashion to \cite{taschler2019model}, we start by mapping $x$ from its ambient space $\R^p$ into a lower dimension embedding space $\R^q$, with $q \ll p$. Then, we perform the EM routine within the embedding space, with $q-$dimensional parameters. As a result, 
$q \ll p$ controls the computational cost of the EM. Moreover, the gap between the embedding size $q$ and the sample size $n$ can also be controlled in order to avoid pathological matrix computations. As a side effect, the projection also re-balances the relative contribution of $x$ and $y$ to the E-step (although not as effectively as the \textit{balanced} RJM-EM variant introduced in Section \ref{sec:balanced_EM}). Finally, as studied in \cite{diaconis1984asymptotics}, linear projection steps can make high dimensional variables more ``Gaussian-like", which, 
if using a Gaussian mixture in the low-dimensional space (as we do here), means a better specified model. Naturally, despite these benefits, reducing the data into a lower dimensional space loses information and we discuss the trade-off below.

\paragraph{New objective function.} 
To begin with, we formulate the objective function of the \textit{projected} RJM in the most generic terms, without the Gaussian assumption. We start from the unspecified hierarchical model on $(x, y, z)$ formulated in Eq.~\eqref{eq:hierarchical_conditional_model}. Consider an embedding function $e: \R^p \rightarrow \R^q$. We assume that the embedded random variable $( e(x)|z)$ inherits from $x$ a parametric pdf in the embedding space. We define $\bthetaex$, a reduced version of the ambient space parameter $\bthetax$, that is sufficient to describe the distribution of $(e(x)|z)$, and call $p(e(x) | z; \bthetaex)$ the associated pdf. To define the new objective function, we simply replace in the \textit{mixture density} \eqref{eq:mixture_conditional_model} the pdf $p(x | z; \bthetax)$ on the ambient $x$ by the pdf $p(e(x) | z; \bthetaex)$ on the embedded $e(x)$. The resulting \textit{projected mixture density} is written:
\begin{equation}\label{eq:projected_mixture_conditional_model}
    p(y, x; \btheta) := \sum_{k=1}^K p(y | x, z=k; \bthetayk) p(e(x) | z=k; \bthetaexk) p(z=k; \bthetazk) \, .
\end{equation}
Assume for instance that the embedded data follows a Hierarchical Gaussian model: $(e(x) | z=k) \sim \N(\mu_k, \Omega_k^{-1})$, where now the mean vectors $\bmu := (\mu_k)_{k=1}^K \in \R^{K\times q}$ and the precision matrices $\bOmega := (\Omega_k)_{k=1}^K \in  S_q^{++}(\R)^K$ live in the embedding space. Then, the most natural parametrisation choice is $\bthetaex := (\bmu, \bOmega)$, such that: 
\begin{equation*}
    p(e(x) | z, \bthetaex) := \sum_{k=1}^K \indic_{z = k}\,  \varphi_q(e(x); \mu_k, \Omega_k^{-1}) \, ,
\end{equation*}
and
\begin{equation}\label{eq:projected_mixture_conditional_model2}
    p(y, x; \btheta) := \sum_{k=1}^K \varphi_1(y ; \alpha_k + \beta_k^t x, \sigma_k^2) \varphi_q(e(x); \mu_k, \Omega_k^{-1}) \tau_k  \, .
\end{equation}
Consider for example the special case of a linear embedding $e(x) = W^t x$, for a certain $W \in \R^{p \times q}$. If the ambient space data $x$ follows a Hierarchical Gaussian model in $\R^p$, then, the embedded $W^t x$ also follows a Hierarchical Gaussian model in $\R^q$, inherited from the ambient one.

\paragraph{EM computations.} Consider $n$ samples $(x_i, y_i)_{i=1}^n$. The penalised negative log-likelihood corresponding to the mixture density \eqref{eq:projected_mixture_conditional_model} is:
\begin{equation}  \label{eq:objective_projected}
    l(\btheta) = -\sum_{i=1}^n \ln\bigg( \sum_{k=1}^K p(y_i | x_i, z_i=k; \bthetayk) p(e(x_i) | z_i=k; \bthetaexk) \tau_k\bigg)  +\pen(\btheta) \, , 
\end{equation}
The E-step is changed accordingly:
\begin{equation} \label{eq:E_step_projected}
    p_{i, k}^{(t)} :=  \frac{p\parent{y_i| x_i,z_i=k; \widehat{\btheta}_k^{(y, t)}} p\parent{e(x_i)|z_i=k; \widehat{\btheta}_k^{(e(x), t)}} \widehat{\tau}_k^{(t)}}{\sum_{l=1}^K p\parent{y_i| x_i,z_i=l; \widehat{\btheta}_l^{(y, t)}} p\parent{e(x_i)|z_i=l;  \widehat{\btheta}_l^{(e(x), t)}} \widehat{\tau}_l^{(t)}} \, .
\end{equation}
For the M-step, only the optimisation in $\bthetaxk$ is different from the ambient space case. In the Gaussian case, with no regularisation on either $\bmu$ or $\bOmega$, the M-step solutions are simply the Maximum Likelihood Estimates:
\begin{gather*}
    \widehat{\mu}_k := \frac{\sum_{i=1}^n p_{i, k}^{(t)} e(x_i)}{n_k^{(t)}}  \, ,  \\
    \widehat{\Omega}_k^{-1} := \frac{\sum_{i=1}^n p_{i, k}^{(t)} (e(x_i)-\widehat{\mu}_k) (e(x_i)-\widehat{\mu}_k)^t}{n_k^{(t)}}   \, .
\end{gather*}

\paragraph{Discussion.}
When defining the \textit{projected} pdf \eqref{eq:projected_mixture_conditional_model2}, 
we choose to keep a generative relation for $y$ of the form $x \mapsto y$ instead of $e(x) \mapsto y$. This is because in practice a large part of the computational cost for large $p$ comes from the optimisation of $p(x | z, \bthetax)$, which, in the Gaussian case, means an optimisation over $S_p^{++}(\R)$, the space of $p \times p$ covariance matrices. Whereas the model $p(y | x, z, \bthetay)$ only involves a $p-$dimensional vector parameter. In our applications, we consider that preserving the $p-$dimensional information present within the model $x \in \R^p \mapsto \beta_k^t x \in \R$ is worth paying the cost of solving a $p-$dimensional regression problem at each M-step.
In other applications, e.g. with very large $p$, or where computational efficiency is still more important, then the generative model $(y | x, z = k) \sim \N(\beta_k^t e(x), \sigma_k^2)$ could be used instead. 
Also note that, since the function $\cdot \mapsto \varphi_q(\cdot; \mu_k, \Omega_k^{-1})$ is a pdf over $\R^q$, then the integral of $\varphi_q(e(x); \mu_k, \Omega_k^{-1})$ over $x\in \R^p$ is not 1 (it is infinite). This means that the observed pdf $p(y, x; \btheta)$ \eqref{eq:projected_mixture_conditional_model} is not a proper density function over $(x, y)$. This has no bearing on the EM computations or theory. As discussed in \cite{delyon1999convergence}, the observed variables $(x, y)$ are treated as fixed constants. The property that actually matters is that the complete pdf $p(y, x, z; \btheta)$ is integrable with $z$.

\paragraph{Embedding trade-off.} With a preliminary projection step, we naturally incur a loss of information. The success of an embedded EM relies on the assumption that enough of the signal in $x$ is preserved after data reduction. It is well known that random projection 
can preserve mean signals
\citep[see][and many subsequent works]{johnson1984extensions, dasgupta1999learning}, as also PCA projection \citep{fradkin2003experiments, deegalla2006reducing}. More recent works \citep{taschler2019model,lartigue2022unsupervised} have studied projections in the absence of any mean difference (i.e. latent groups co-located in the ambient space), 
arguing that 
differential {\it covariance signal} can still be preserved, especially under PCA projection. 
All these works support the notion that signals relevant to detection of the latent subgroup structure in $x$ can be preserved by the projection step. 
Given all these considerations, the chosen embedding size $q$ should strike a good balance between computational efficiency, preservation of information and regularisation \citep[see][for a more detailed analysis of this trade-off]{lartigue2022unsupervised}. In this work, we propose an empirical procedure to set $q$ adaptively in Section \ref{sec:exp_mod_sel}. 

\subsection{Balanced-EM} \label{sec:balanced_EM}
In this section, we introduce the \textit{balanced} variant of RJM-EM which allows the user to control the relative influence of the signals in $y$ and $x$ via the E-step. The \textit{balanced} and \textit{projected} RJM-EM are fully compatible with each other.

\paragraph{The balance problem.} The complete log-likelihood of the model \eqref{eq:hierarchical_conditional_model} is decomposed as a sum of three conditional log-likelihoods: $\ln p(y_i | x_i, z_i=k; \bthetayk)$, $\ln p(x_i | z_i=k; \bthetaxk)$ and $\ln p(z_i=k; \bthetazk)$. The magnitude of each of these terms depends on the dimension of the corresponding variable. In particular, when $p$ is large, the term in $x$ completely dominates the sum and the contribution of $y$ is lost.
This in turn means that the term in $x$ dominates the E-step \eqref{eq:E_step}, hence the subgroup assignment within the EM. In applications, this behaviour may be undesirable, because, as noted previously, the response $y$ usually captures a particularly important aspect of an applied problem and in that sense is not ``just another variable". Hence, we are interested in solutions that allow us to control the balance between the influence of $x$ and $y$ on subgroup attribution. 

\paragraph{Relation with projection.} In Section \ref{sec:projected_EM}, we introduced a projection on $x$ in order to control computational burden. As a side effect, reducing the ``block'' of $x$ from a $p-$ to $q-$dimensions also reduces its prominence in the E-step \eqref{eq:E_step}. Hence, the cluster identification is less dominated by $x$. Despite this effect, we argue that the role of controlling the balance between $y$ and $x$ in the E-step should not fall to the embedding size $q$, since there are already several other competing factors that play into choosing $q$. 
For instance, to obtain an even balance between the influence of $y$ and $x$ for cluster assignment in the E-step would imply $q=1$, but this would imply a massive loss of information (reducing the feature signal to univariate). Rather, we propose to preserve information on $x$ by choice of $q$ (discussed below) whilst simultaneously controlling the E-step contribution.

\paragraph{Balanced RJM.} To address the balance problem without loss of information on $x$, we introduce the \textit{balanced} variant of the RJM EM algorithm. In a similar fashion to the \textit{projected} RJM, we introduce a new, improper, density function, the \textit{balanced mixture density}:
\begin{equation}\label{eq:balanced_mixture_conditional_model}
    p(y, x; \btheta) := \sum_{k=1}^K p(y | x, z=k; \bthetayk) p^{\frac{1}{T}}(x | z=k; \bthetaxk) p(z=k; \bthetazk) \, ,
\end{equation}
where $T \in \R_+^*$ is a \textit{balancing parameter}. By adjusting $T$ we can re-balance the contribution of $x$ to the likelihood (relative to those of $y$ and $z$). Typically, by taking $T>1$, we reduce the impact of $x$ on the group-labels estimation of the E-step. In addition, if $pen(\bthetaxk)\neq0$, we also replace $\pen(\bthetayk) + \pen(\bthetaxk) + \pen(\bthetazk)$ by $\pen(\bthetayk) + \frac{1}{T} \pen(\bthetaxk) + \pen(\bthetazk)$. This way, since the optimisation in $\bthetax, \bthetay$ and $\bthetaz$ are separable, the balancing parameter $T$ affects only the E-step and not the M-step.\\ 
\\
The new objective function of the EM can be written:
\begin{equation*} \label{eq:objective_tmp}
    l(\btheta)= -\sum_{i=1}^n \ln\bigg( \sum_{k=1}^K p(y_i | x_i, z_i=k; \bthetayk) p^{\frac{1}{T}}(x_i | z_i=k; \bthetaxk) \tau_k\bigg)  +\pen(\btheta) \, , 
\end{equation*}
The E-step becomes:
\begin{equation} \label{eq:E_step_tmp}
    p_{i, k}^{(t)} :=  \frac{p\parent{y_i| x_i,z_i=k; \widehat{\btheta}_k^{(y, t)}} p^{\frac{1}{T}}\parent{x_i|z_i=k; \widehat{\btheta}_k^{(x, t)}} \widehat{\tau}_k^{(t)}}{\sum_{l=1}^K p\parent{y_i| x_i,z_i=l; \widehat{\btheta}_l^{(y, t)}} p^{\frac{1}{T}}\parent{x_i|z_i=l;  \widehat{\btheta}_l^{(x, t)}} \widehat{\tau}_l^{(t)}} \, .
\end{equation}
And the function $Q$:
\begin{equation*}
\begin{split}
    Q(\btheta| \btheta^{(t)}) = \sum_{k=1}^K \Bigg(&-\sum_{i=1}^n p_{i, k}^{(t)} \ln p(y_i | x_i, z_i=k; \bthetayk) + \pen(\bthetayk)\\
    &-\frac{1}{T}\sum_{i=1}^n p_{i, k}^{(t)}\ln p(x_i | z_i=k; \bthetaxk)  + \frac{1}{T} \pen(\bthetaxk) \\
    &-\sum_{i=1}^n p_{i, k}^{(t)}\ln \tau_k+ \pen(\bthetazk)\Bigg) \, .
\end{split}
\end{equation*}
Hence, the M-step is unchanged. The \textit{balanced} RJM and the \textit{projected} RJM can be freely combined into the \textit{balanced} \textit{projected} RJM by replacing $p(x | z=k; \bthetaxk)$ with $p(e(x) | z=k; \bthetaexk)$ in the equations of this section.

\paragraph{Remarks.} As was the case for the \textit{projected} RJM, with the addition of the term $\frac{1}{T}$, the new objective function does not integrate to 1 anymore, hence cannot formally be interpreted as a likelihood. Adding the corresponding normalisation term would resolve this, but would also counteract the effect of the re-balancing and modify the dynamic of the optimisation in unwanted ways. From an optimisation point of view, losing the probabilistic interpretation is not problematic; as already mentioned, the EM optimisation theory does not actually require the optimised function to have an integral over $(y, x)$ of 1 \citep[see for instance][]{delyon1999convergence}. Moreover, the addition of $T$ does not affect the interpretation of the model parameters $\bthetax, \bthetay$ and $\bthetaz$ (and the M-step is unaffected). 

\section{Convergence of the RJM-EM algorithm}\label{sec:theory}
In this section, we state a convergence Proposition for the \textit{balanced projected} RJM-EM. With certain conditions on the regularisation, we can guarantee that the EM sequence will converge towards the set of stationary points of the objective function. This is an application of a convergence Theorem from \citet{delyon1999convergence}, which itself builds on the seminal works of \citet{wu1983convergence} and \citet{lange1995gradient}, and tailored for models within the exponential family of distributions.

Within the \textit{balanced projected} RJM model, the unpenalised complete likelihood is:
\begin{equation*}
    p(\by, \bx, \bz; \btheta) = \prod_{i=1}^n \sum_{k=1}^K  \varphi_1(y_i ; \alpha_k + \beta_k^t x_i, \sigma_k^2) \varphi_q^{\frac{1}{T}}(e(x_i); \mu_k, \Omega_k^{-1}) \tau_k \indic_{z_i=k} \, ,
\end{equation*}
and the corresponding observed likelihood is:
\begin{equation*}
     p(\by, \bx; \btheta) = \prod_{i=1}^n \sum_{k=1}^K  \varphi_1(y_i ; \alpha_k + \beta_k^t x_i, \sigma_k^2) \varphi_q^{\frac{1}{T}}(e(x_i); \mu_k, \Omega_k^{-1}) \tau_k   \, .
\end{equation*}
Borrowing notations from \cite{delyon1999convergence}, we can express the penalised complete likelihood as:
\begin{equation*}
\begin{split}
    f(\bz; \btheta) &= p(\by, \bx, \bz; \btheta) exp\parent{-\pen(\btheta)} \\
    &= exp\parent{-\pen(\btheta) + \sum_{i=1}^n \sum_{k=1}^K \indic_{z_i=k} \ln p(y_i, x_i, z_i=k ; \btheta)} \\
    &= exp\parent{-\psi(\btheta) + \dotprod{\tilde{S}(z), \phi(\btheta)}}\\
    &= exp (L(s, \btheta))\, .
\end{split}
\end{equation*}
Where $\tilde{S}(z) \! := \! \parent{\indic_{z_i=k}}_{i, k \in \llbrack{1, n}\times\llbrack{1, K}} \in \R^{n\times K}$, $\phi(\btheta) \! :=\!  \parent{\ln p(y_i, x_i, z_i=k ; \btheta)}_{i, k \in \llbrack{1, n}\times\llbrack{1, K}} \in \R^{n\times K}$, $\psi(\btheta) \! := \! \pen(\btheta)$, and $L(s, \btheta) \! := \!  -\psi(\btheta) + \dotprod{\tilde{S}(z), \phi(\btheta)}$. The sufficient statistic function $\tilde{S}(z)$ takes its values inside the open and convex set $\S := \R_+^{n \times K}$. With $\nu$ a combination of Dirac measures: $\nu(d\bz) = \prod_{i=1}^n \sum_{k=1}^K d \delta_{k}(z_i)$, the E-step of the EM algorithm can be formally expressed as:
\begin{equation*}
    \bar{s}(\btheta) := \int_{\bz} \tilde{S}(\bz) \frac{p(\by, \bx, \bz; \btheta)}{p(\by, \bx; \btheta)} \nu(d\bz) \in \S \, .
\end{equation*}
Such that $(\bar{s}(\btheta))_{i, k} =p(z_i=k|y_i, x_i; \btheta)$. For a given $s\in \S$, the M-step can be formally expressed as:
\begin{equation*}
    \widehat{\btheta}(s) := \argmax{\btheta \in \Theta} L(s, \btheta)  \, .
\end{equation*}
Then, a full EM-step is:
\begin{equation*}
    \btheta^{(t+1)} := \widehat{\btheta}(\bar{s}(\btheta^{(t)})) \, .
\end{equation*}
We call $\L := \brace{\btheta \in \Theta \mid \btheta := \widehat{\btheta} \circ \bar{s}(\btheta)}$ the set of fixed points of the EM, and consider $d(\cdot, \cdot)$ a point to set distance between elements of $\Theta$ and closed sets of $\Theta$. With $g(\btheta) := p(\by, \bx; \btheta) exp\parent{-\pen(\btheta)}$ the penalised observed likelihood, the objective function of the EM is $l(\btheta):=-\ln g(\btheta)$.
The following Proposition is an application of Theorem 1 from \cite{delyon1999convergence}.

\begin{proposition}[Convergence of the \textit{balanced projected} RJM-EM]\label{thm:convergence_EM}
    Assume that 
    \begin{itemize}
        \item[\textbf{(C1)}] The function $\pen(\btheta)$ is twice continuously differentiable.
        \item[\textbf{(C2)}] The M-step function $\widehat{\btheta}(s)$ is well defined and is continuously differentiable.
        \item[\textbf{(C3)}] The EM sequence $(\btheta^{(t)})_{t \in \bN}$ remains within a compact of the parameter space $\Theta$.
    \end{itemize} 
    Then, the sequence $\parent{l(\btheta^{(t)})}_{t \in \bN}$ is decreasing and $\underset{t \to \infty}{lim} d(\btheta^{(t)}, \L) \longrightarrow 0$. Moreover, the set of fixed points $\L$ is equal to the set of stationary points of $l(\btheta)$: $\L = \brace{\btheta \in \Theta \mid \nabla_{\btheta} l(\btheta) = 0}$.
\end{proposition}

\begin{proof}
    This is a direct application of Theorem 1 of \cite{delyon1999convergence}, which we recall in Appendix \ref{app:thm_dlm}. By property of the Gaussian pdf, the likelihood $p(\by, \bx, \bz; \btheta)$ of our model is $C^{2}$ ($C^{\infty}$ even) in the parameter $\btheta$, and thanks to assumption \textbf{C1}, so is $\pen(\btheta)$. Moreover, since this is a finite mixture model, the integrals over $\bz\in \R^n$ are actually finite sums over $\llbrack{1, K}^n$. 
    Hence, the regularity conditions M1-4 of \cite{delyon1999convergence} are all easily verified. Then, our assumption \textbf{C2} implies their condition M5. Finally, with assumption \textbf{C3} verified, their Theorem 1 applies.
\end{proof}

\begin{corollary}
    Proposition \ref{thm:convergence_EM} is immediately extended to the \textit{balanced} RJM-EM, \textit{projected} RJM-EM and the regular RJM-EM by taking $W=I_p$ and/or $T=1$.
\end{corollary}

\paragraph{Verifying the conditions.} Condition \textbf{C2} can be verified on a case by case basis with specific penalty functions. When the M-step is available in closed form, the function $\widehat{\btheta}(s)$ can be studied explicitly. For instance, the block M-steps \eqref{eq:M_step_sigma}, \eqref{eq:M_step_mu} and \eqref{eq:M_step_tau}, are all clearly continuously differentiable ($C^{\infty}$ even) in the sufficient statistic $(p_{i, k}^{(t)})_{i, k \in \llbrack{1, n}\times\llbrack{1, K}}$. When the M-step is not explicit, then, as discussed in \cite{delyon1999convergence}, the regularity of $\widehat{\btheta}$ can usually be established through the implicit function theorem.

A sufficient condition that guarantees that \textbf{C3} is verified for all EM sequences is that the level sets $\brace{\btheta \mid l(\btheta) \leq c}$ are compacts of $\Theta$ for any $c \in \R$, see \cite{delyon1999convergence}. Accordingly, we provide in the following Lemma conditions on $pen(\btheta)$ that are sufficient to verify \textbf{C3}. 

\begin{lemma}[Sufficient conditions for \textbf{(C3)}] 
    Assume that the M-step is well defined, that $\pen(\btheta)$ is continuous, and that there exists a positive constant $\epsilon > 0$ such that:
    \begin{equation*}
       \mathrm{pen}(\btheta)\geq \epsilon \sum_{k=1}^{K} \parent{\log \tau_k^{-1} + \norm{\mu_k} + \norm{\Omega_k} + \log \det{\Omega_k^{-1}}  + \sigma_k^{-1} + \ln \sigma_k + \det{\alpha_k} + \norm{\beta_k}}\, .
    \end{equation*}
    Where the norms $\norm{\cdot}$ do not need to be specified since the dimension is finite.\\
    \\
    Then the level sets $\brace{\btheta \mid l(\btheta) \leq c}$ are compacts of $\Theta$ for all $c \in \R$.
\end{lemma}
Up to a minor parametrisation changes, the proof of the Lemma is identical to the proof of Theorem 2 in \cite{perrakis2019latent}.

{We note that the penalties discussed in Section \ref{sec:regularisation} and used in the experiments of this paper do not verify the conditions of Proposition \ref{thm:convergence_EM}. In particular, the $l_1$ penalty on $\beta_k$ is not twice continuously differentiable. This is neither surprising nor undesired: with the inclusion of an $l_1$ penalty, the algorithm is expected to converge towards singular points of the likelihood, not stationary points. For our applications, especially when $p$ is large, we judge that having a sparse linear regression between $x$ and $y$ within the EM is worth forgoing the formal convergence guarantee offered by Proposition \ref{thm:convergence_EM}.}

\section{Experiments}\label{sec:experiments}

\subsection{Protocol and compared methods.}
We carried out a series of experiments to study the behaviour of the different variants of RJM proposed.
The empirical results span 
a range of scenarios 
involving both simulated and real data.

To put the performances of the RJM methods in perspective, we compare them to a collection of well-known mixture and clustering methods. In particular, these include: 
\begin{itemize}
    \item \textbf{KMeans(X).} KMeans on the \textit{input} $x \in \R^p$ alone.
    \item \textbf{KMeans(X, Y).} KMeans on both the \textit{input} $x$ and the \textit{response} $y$ concatenated into one vector $(x, y)\in \R^{p+1}$.
    \item \textbf{GaussianMixture(X).} EM with Mixture of Gaussian (MoG) model (elliptical Gaussians) on $x \in \R^p$ alone. 
    \item \textbf{GaussianMixture(X, Y).} EM with MoG model on the concatenated $(x, y)\in \R^{p+1}$.
    \item \textbf{MoE.} Mixture of Experts, a method which estimates a hierarchical model $x \longrightarrow y$. The \textit{FlexMix} package \citep{Leisch2004FlexMix, Leisch2007Fitting, Leisch2008FlexMix} was used (the \textit{mixtools} package \citep{Benaglia2009mixtools} was also tried and yielded similar results).
\end{itemize}

\medskip

\noindent
{\it RJM variants.} 
For RJM, we consider the variants mentioned in the previous sections: $x$ can either remain in the ambient space $\R^p$ or be projected in an embedding space of dimension $q$, and the likelihood of $x$ can either be balanced to alter its weight in the E-step or not. Combining these possibilities leads to the definition of four methods:
\begin{itemize}
    \item \textbf{RJM.} $x$ in ambient space ($W=I_p$), no re-balancing ($T=1$).
    \item \textbf{bal.-RJM.} $x$ in ambient space, balancing in play ($T\neq 1$).
    \item \textbf{proj.-RJM.} $x$ projected unto $\R^q$ ($W\in \R^{p\times q}$), no re-balancing.
    \item \textbf{bal. proj.-RJM.} $x$ projected unto $\R^q$, balancing in play.
\end{itemize}
For the two projected RJM, we consider the linear embedding $e(x) = W^t x$, where $W$ is the $q-$dimensional PCA projection on the full data $x$. PCA projection was shown in \cite{lartigue2022unsupervised} to adequately preserve subgroup information. In most of the following experiments, in the interest of simplicity, we will fix the values of $q$ and $T$. For the balancing coefficient $T$, the values we use are $T=p$ if $x$ remains in the ambient space $\R^p$ and $T=q$ if $x$ is projected unto $\R^q$ (this means that the contribution of $x$ and $y$ to the E-step are equally weighted). In some experiments of Appendix \ref{app:exp_mod_sel}, we explore a grid of values for $q$ and $T$ to understand their effect on results. In Section \ref{sec:exp_mod_sel}, we study a fully adaptive scheme to set $q$ (and $K$).

\medskip

\noindent
{\it Data types considered and evaluation.} 
Throughout this section, we compare these methods on different data types:
\begin{itemize}
    \item Synthetic Gaussian data in Section \ref{sec:exp_gaussian}, Appendix \ref{app:exp_gaussian} and small parts of Appendix \ref{app:exp_mod_sel}.
    \item Synthetic non-Gaussian continuous data and synthetic mixed binary/Gaussian data in Appendix \ref{app:exp_non_gaussian}
    \item Real RNA sequencing data (count data, hence non-continuous) in Section \ref{sec:exp_tcga} and \ref{sec:exp_mod_sel}, and in Appendix \ref{app:exp_tcga} and \ref{app:exp_mod_sel}.
\end{itemize}
To evaluate the performance of each method with respect to learning the latent groups, we use the Adjusted Rand Index (that we will often simply refer to as ``Rand Index") between the learned and true labels (the latter are of course available for evaluation purposes in the synthetic data context). 
For the high-dimensional parameters, we consider sparsity pattern recovery metrics. We are interested in case where the true $\bbeta$ and $\bOmega$ are sparse, and want to understand whether 
sparsity patterns specific to latent groups can be effectively recovered. To do so, we compare the estimated $\widehat{\bbeta}$ and $\widehat{\Omega}$ to the true, data-generating parameters, with the former being the outputs of the ``final M-step in the ambient space" (i.e. the output of sparse estimators of $\bbeta$ and $\bOmega$ using the labels estimated by the respective procedure).

\subsection{Gaussian simulations} \label{sec:exp_gaussian}
This experiment is a preliminary study of the compared methods under familiar settings (moderate dimensionality and Gaussian data). We study the evolution of the performances with the sample size. Across a wide range of scenarios, we find that the RJM methods, combining the various signals, 
 are often the most data efficient.

We consider a synthetic dataset where the input $x$ follows a hierarchical Gaussian distribution $(x | z = k) \sim \N(\mu_k, \Omega_k^{-1}) $, 
with $K=2$ balanced classes and 
the response $y$ verifies: $(y | x, z = k) \sim \N(\beta_k^t x, \sigma_k^2)$ (no intercept), with 
$\sigma_1 {=} \sigma_2{=}0.5$.
 Since the parameters of this generative model are fully known to us, we can define a \textit{parameter-Oracle} method (simply called \textit{Oracle}) to serve as reference. This Oracle estimates labels from data 
with a maximum likelihood classifier using the knowledge of the real model parameters $\tau, \mu, \Omega, \alpha, \beta, \sigma$.

\paragraph{Definition and control of the group-specific signals.} 
With the above settings, there are three sources that constitute the latent group-specific signal:
\begin{itemize}
    \item mean signal in $x$, captured by the difference between $\mu_1$ and $\mu_2$,
    \item covariance signal in $x$, captured by the difference between $\Omega_1$ and $\Omega_2$,
    \item and signal in the regression for $y$, captured by the difference between $\beta_1$ and $\beta_2$.
\end{itemize}
In the simulation, we can of course control each of these sources. At first, we impose that the precision matrices of the two latent subgroups are the same: $\Omega_1 = \Omega_2$. For each simulation, a common precision matrix is generated randomly. With no difference between second order parameters, the group-signal is only expressed in the difference between the $\beta_k$'s (signal that we commonly call ``in $\beta$'' or ``in $y$'') and/or in the difference between the $\mu_k$'s (signal that we commonly call ``in $\mu$'' or ``in $x$''). 

For simplicity, in this Section we set parameters according to the following scheme: we generate each $\beta_k \in \R^p$ with a limited number of non-zero coefficients (here $p/10$), all of which have the same absolute value (denoted ``$|\beta|$") and a random sign. We refer to $|\beta|$ as the \textit{amplitude} of $\beta_k$. By default, there is no overlap between the non-zero coefficient sets of the respective $\beta_k$'s. Hence, larger amplitude $|\beta|$ implies more distinct subgroups and the quantity $|\beta|$ quantifies the amount of signal in $\beta$. For the signal in $x$, we independently generate $K$ vectors $(\mu_k)_{k=1}^K$, then, we control their spacing with an affine transformation $\tilde{\mu}_k := \mu_k + \lambda (\mu_k - \overline{\mu})$, where $\overline{\mu} := \frac{1}{K} \sum_k \mu_k$ is the barycenter. This transformation does not change the barycenter, $\overline{\mu} = \frac{1}{K} \sum_k \tilde{\mu}_k$, and is such that:
\begin{equation*}
    \frac{1}{K} \sum_k \norm{\tilde{\mu}_k - \overline{\mu}}_2^2 = (1+\lambda)^2 \frac{1}{K} \sum_k \norm{\mu_k - \overline{\mu}}_2^2 \, .
\end{equation*}
Let $\delta_{\mu}$ be the desired magnitude of the spacing between the $\mu_k$. We set 
$\lambda = \frac{\delta_{\mu}}{\sqrt{\frac{1}{K} \sum_k \norm{\mu_k - \overline{\mu}}_2^2}} - 1$, 
such that:
\begin{equation*}
    \sqrt{\frac{1}{K} \sum_k \norm{\tilde{\mu}_k - \overline{\mu}}_2^2} = \delta_{\mu} \, .
\end{equation*}
In the particular case where $K=2$, this results in $\norm{\frac{\tilde{\mu}_1-\tilde{\mu}_2}{2}}_2 = \delta_{\mu}$. Clearly, as $\delta_{\mu}$ increases, the signal in $x$ increases as well. Hence, we adopt $\delta_{\mu}$ as our $x$-signal metric.

\paragraph{Latent class assignment results.} In the first experiment, we study the various methods by exploring cases with different forms of signal, drawing $n$ iid samples ($n$ varying from 100 to 2000) from the generating model above with $p=100$. The data $x$ and $y$ are available to the various algorithms but the latent indicator $z$ is not.
For proj.-RJM, the embedding size is $q=5$. The balanced RJM (bal.-RJM) applies to the likelihood of $x$ in the E-step the exponent $1/p$ for the ambient RJM and $1/q$ for the projected RJM (proj.-RJM). 

Results appear in \figurename~\ref{fig:rand_index_vs_n}. The colour convention adopted here is consistent throughout the paper. 
The first configuration (leftmost panel) has $|\beta|=5$ but full overlap ($\beta_1=\beta_2$), hence no group-specific signal in the regression model. Instead the signal is purely in $x$, with $\delta_{\mu} = 0.2$. As expected, with a mean 
difference and no regression signal, KMeans($x$) and MoG($x$) 
perform relatively well (matched by RJM).
The second configuration (middle panel) has no signal in $x$ ($\delta_{\mu} = 0$), but a strong signal in $\beta$: $|\beta| = 5$ with no overlap on the non-zero coefficient sets between the two latent subgroups. When the data is scarce, the balanced and projected RJM are well behaved, with bal. proj.-RJM being the better of the three variants. When the data is abundant, MoE is the best method, and close to the optimal. On the other hand, the MoE computations fail when the sample size is too small, in particular when $n<p$. Among the remaining methods, KMeans($x, y$) and MoG($x, y$) are acceptable. For the last configuration (rightmost panel), we have signal in both $x$, with $\delta_{\mu} = 0.2$, and in $y$, with $|\beta| = 5$, no overlap between subgroups. In this case, RJM, proj.-RJM and bal. proj.-RJM 
can exploit both signals and 
reach good performances at all the sample sizes tested. RJM is the best of the three and is matched by MoE when the data is abundant. KMeans($x$) and MoG($x$) are also quite effective overall.

\figurename~\ref{fig:rand_index_vs_n_flipped_signs} in Appendix \ref{app:exp_gaussian} shows additional results for the case
where the sparsity patterns of $\beta_1$ and $\beta_2$ are equal instead of being disjoint. The corresponding coefficients all have the same amplitude $|\beta|$ but opposed signs between latent subgroups. We call this the \textit{flipped signs} convention, as opposed to the \textit{no overlap} convention used in \figurename~\ref{fig:rand_index_vs_n}. In \figurename~\ref{fig:grid_signal}, Appendix \ref{app:exp_gaussian} we study a grid of signal intensities, for a fixed sample size $n=500$. 

\begin{figure}[tbhp]
    \centering
    \includegraphics[width=\linewidth]{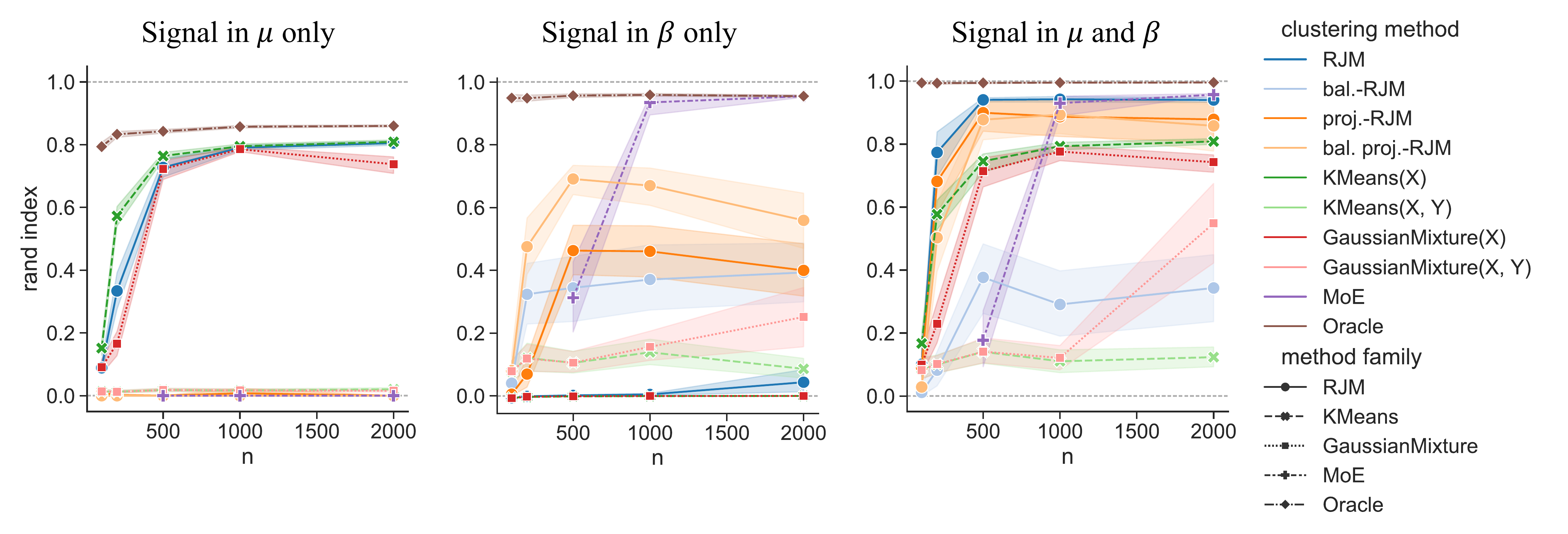}
    \caption{Simulated data, Rand Index vs sample size $n$. The problem dimension is fixed to $p=100$ (see text for details). (Left) Signal in $\mu$ only, $\delta_{\mu} = 0.2$. $\beta_1=\beta_2$. (Middle) Signal in $\beta$ only, $|\beta| = 5$ (with disjoint sparsity patterns between the two latent subgroups). (Right) Signal in both $\mu$ and $\beta$ with $\delta_{\mu} = 0.2, |\beta| = 5$.}
    \label{fig:rand_index_vs_n}
\end{figure}

\paragraph{Recovery of latent group-specific sparsity patterns.} 
Sparsity patterns 
can be useful for scientifically interpretable high-dimensional modelling. S-RJM returns sparse regression coefficients $\bbeta$ and sparse precision matrices $\bOmega$ that are specific to the latent groups. 
In addition to the Rand Index, we therefore also investigate recovery of these parameters.
For the initial experiments in \figurename~\ref{fig:rand_index_vs_n}, we kept $\Omega_1{=}\Omega_2$ to isolate the effect of the mean signals. For this next experiment, we 
move to the more general setting of $\Omega_1\neq\Omega_2$.
In the following we have $p=100$, exactly 10 non-zero coefficients 
in the the (true) $\beta_k$'s and 
$\sim$200 non-zero off-diagonal coefficients in the $\Omega_k$'s.

While S-RJM is a joint learning method, a common workflow in applications is to cluster the data followed by parameter estimation at the group level. To understand performance in this setting and allow for a fair comparison, we use the same sparse estimation methodology for all the methods, i.e. applying the sparse learner to the data, using the cluster labels returned by the respective method. 
This is done as follows. 
First, 
to deal with invariance to label permutation, we reorder the labels with the Hungarian algorithm \citep{kuhn1955hungarian} to match them as best as possible with the true ones according to the 0-1 loss. Then, as discussed in Section \ref{sec:regularisation}, we apply a sparse estimator
in the ambient space to obtain $\widehat{\beta}_k$ and $\widehat{\Omega}_k$. 
The specific estimators used are the Lasso \citep{tibshirani1996regression} for $\bbeta$ and Graphical-Lasso \citep{friedman2008sparse} for $\bOmega$.
We compute Precision-Recall (PR) curves (with respect the true, data-generating sparsity pattern in respectively $\beta_k$ or $\Omega_k$) and compute its Area Under the Curve (AUC). 
We display the evolution of the PR-AUC with the sample size on \figurename~\ref{fig:PR_AUC_beta_Omega_vs_n_different_Omegas_more_edges} (for the ROC-AUC version of this Figure, see \figurename~\ref{fig:AUC_beta_Omega_vs_n_different_Omegas_more_edges} in Appendix \ref{app:exp_gaussian}). The recovery of $\bOmega$ is mediocre for all methods. The recovery of $\bbeta$ is very good for bal. proj.-RJM in all scenarios, and for proj.-RJM when there is signal in $\beta$. As with the Rand Index, MoE becomes dominant when the data is abundant.

Appendix \ref{app:exp_gaussian} includes a number of variants of this experiment, including: \figurename~\ref{fig:AUC_beta_Omega_vs_n_different_Omegas}, with much sparser $\Omega_1$ and $\Omega_2$ (only $\sim 6$ non-zero off-diagonal coefficients); \figurename~\ref{fig:AUC_beta_Omega_vs_n}, with $\Omega_1=\Omega_2$ (same experiment as \figurename~\ref{fig:rand_index_vs_n}); \figurename~\ref{fig:AUC_beta_Omega_vs_n_flipped_signs}, with $\Omega_1=\Omega_2$ and \textit{flipped signs} convention on the $(\beta_k)_{k=1, 2}$ (same experiment as \figurename~\ref{fig:rand_index_vs_n_flipped_signs}).

\begin{figure}[tbhp]
    \centering
    \includegraphics[width=\linewidth]{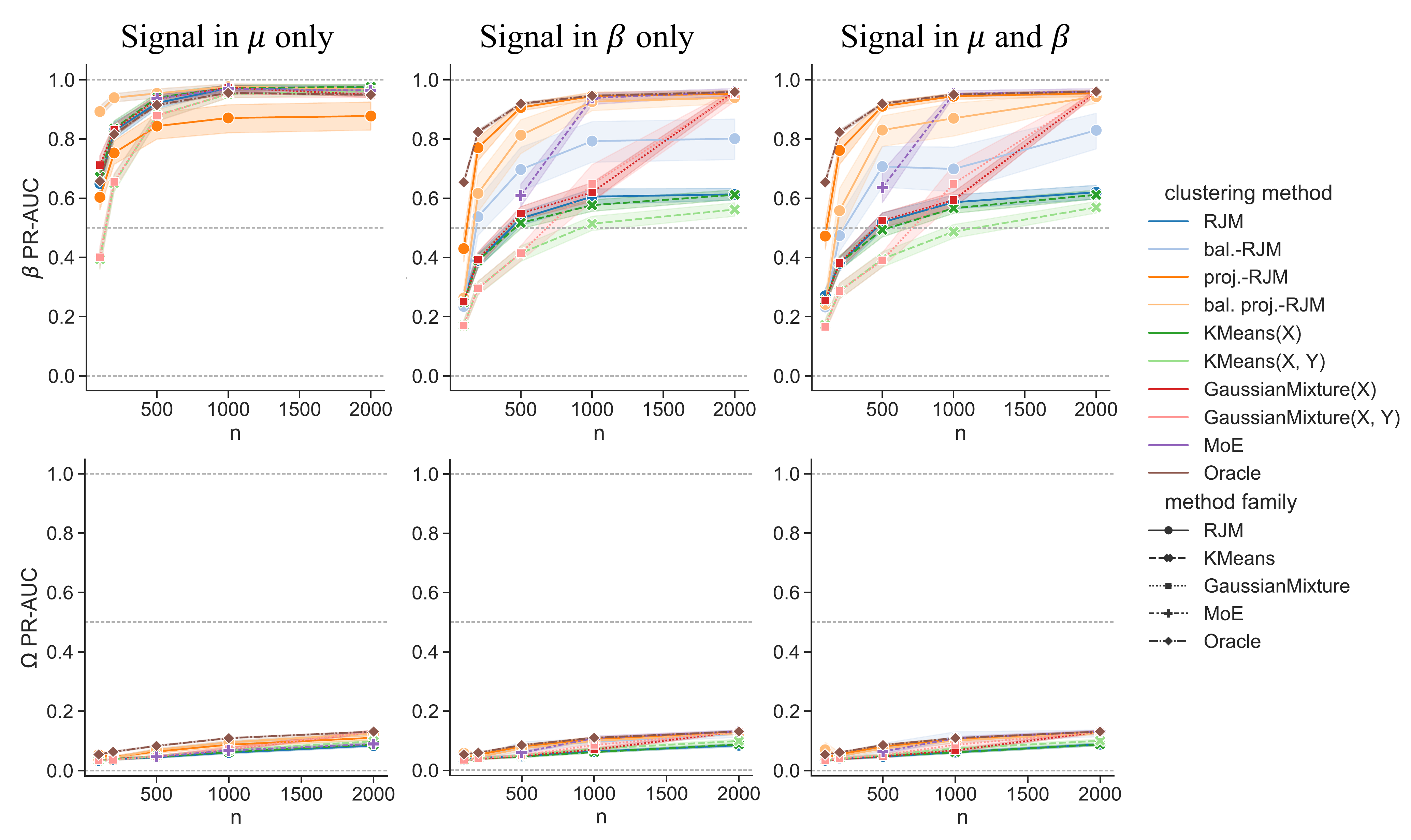}
    \caption{Area under the Precision-Recall curve (PR-AUC)  with respect to recovery of non zeros in $\bbeta$ (top) or $\bOmega$ (bottom) vs the sample size $n$. Problem dimension (i.e. size of ambient space) fixed to $p=100$. (Left) Signal in $\mu$: $\delta_{\mu} = 0.2$. No signal in $\beta$: $\beta_1=\beta_2$. (Middle) Signal in $\beta$: $|\beta| = 5$, with disjoint sparsity patterns between the latent subgroups. No signal in $\mu$: $\delta_{\mu}=0$ (Right) Signal in both $\mu$ and $\beta$. $\delta_{\mu} = 0.2, |\beta| = 5$. Unlike \figurename~\ref{fig:rand_index_vs_n}, for all of these experiments, there is a covariance signal in $x$, such that $\Omega_1 \neq \Omega_2$.}
    \label{fig:PR_AUC_beta_Omega_vs_n_different_Omegas_more_edges}
\end{figure}

\subsection{Non-Gaussian simulated data} \label{sec:exp_non_gaussian}
In this section and the following one, we explore the performances and behaviours of the methods in the non-Gaussian case. We start with two simple non-Gaussian simulations, the results of which can be found in Appendix \ref{app:exp_non_gaussian} (figures \ref{fig:student_rand_index} to \ref{fig:mixed_binary_pr_auc}). In the first simulation, the data is sampled from a heavy tailed student-t distribution. In the second simulation, the data is a mix of binary and continuous features, following the model of \citet{fan2017high}. The Rand Index analysis of the student data (\figurename~\ref{fig:student_rand_index}) displays some differences from the Gaussian case (\figurename~\ref{fig:grid_signal}), with for instance relatively better performances of KMeans and MoG. Although, as in the Gaussian case, the two balanced RJM perform the best in the high $y-$signal, low $x-$signal settings. Results with mixed (binary and continuous) data (\figurename~\ref{fig:mixed_binary_rand_index}) are broadly similar to the Gaussian case. The sparsity pattern recovery analyses (\figurename~\ref{fig:student_pr_auc} for the student data, \figurename~\ref{fig:mixed_binary_pr_auc} for the mixed data) highlight the relevance of using non-Gaussian estimators for $\bOmega$.

\subsection{Non-Gaussian biomedical data} \label{sec:exp_tcga}
Next, we turn to real data experiments. We use an RNA-sequencing dataset from The Cancer Genome Atlas (TCGA). Here, the features are gene expression levels, i.e. the rows in the feature matrix correspond to patient samples and the columns $p_{\text{total}} = 60488$ specific transcripts (i.e. counts of different RNAs; the data are non-Gaussian in nature). 
Each sample belongs to one of four classes: bronchus-lungs ($n_1=1146$ observations), breast ($n_2 = 1223$), kidney ($n_3 = 1023$), brain ($n_4 = 705$) and it is these that play the role of subgroups in the analysis.

As the data is real, we do not know the true population distributions and cannot control all aspects of the signals. However, we can control and vary some aspects of the problem by sub-sampling $p \leq p_{\text{total}}$ genes (to vary dimensionality), $K \leq 4$ tissues (to vary the number of latent groups) 
and $n \leq \sum_{k=1}^K n_k$ patient samples (to vary sample size) for each simulation. The resulting datasets of size $n \times p$ play the role of $x$.
To study responses whilst retaining the ability to provide unambiguous, ground-truth results, we create a synthetic response variable $y$. First, we generate a sparse vector $\bbeta\in \R^{K \times p}$ (with disjoint class-specific sparsity patterns), as in the previous section. Then we sample: $(y | z=k)\sim  \N(\beta_k^t x, \sigma_k^2)$. Additionally, we control the (mean) signal present in $x$ by shifting 
the group-specific data (as we did for the previous experiment). In the below, we either keep the original mean differences (no change, i.e. mean levels as in the original dataset) or remove all mean differences (i.e shift to ensure equal group specific empirical means). The latter results in a situation where the {\it only} signal left in $x$ is in the (real, unknown) covariance (and higher moments). 
On the resulting dataset $(x, y)$
we run the same methods as in the previous experiments. For the \textit{parameter-Oracle}, we keep essentially the same definition as in synthetic experiments but replace the true parameters (not known here) with estimates $\widehat{\bmu}$, $\widehat{\bOmega}$, $\widehat{\bbeta}$ and $\widehat{\bsigma}$ learned with knowledge of the true labels. In particular, $\widehat{\bOmega}$ is estimated with a $l_2$ regularisation \citep[using the \textit{Optimal Approximation Shrinkage} of][]{chen2010shrinkage}. The hyper-parameter settings of the RJM methods are the same as in Section \ref{sec:exp_gaussian}, namely: embedding size $q=5$ for proj.-RJM and bal. proj.-RJM, balancing parameter $T=\frac{1}{p}$ for bal.-RJM, and $T=\frac{1}{q}$ for bal. proj.-RJM.

\medskip 
\noindent
{\it Varying sample size}.
We start by varying sample size $n$, for a fixed small $p=100$. \figurename~\ref{fig:TCGA_vs_n_K4} displays the results (for $K=4$, i.e. all tissue types above). See Appendix \ref{app:exp_tcga} (figures \ref{fig:TCGA_vs_n_K2} to \ref{fig:TCGA_vs_n_K3_no_brain}) for similar results with different $K$'s and tissue combinations. We see that, when there is signal in $y$, bal. proj.-RJM, bal.-RJM and proj.-RJM arer the best performers at all sample sizes. MoE matches these approaches when the sample size is very large: $n>10p=1000$ for $K=2$ and $n>20p=2000$ for $K= 3$ or 4, while MoG($(x, y)$) and KMeans($x, y$) are competitive (relative to the other methods) when the sample size is small ($n=p=100$). When there is no signal in $y$, the problem is harder, and even the Oracle performance is poor. However, we see that proj.-RJM is similar to the Oracle. The MoG on $(x, y)$ and $(x)$ can also match the Oracle for larger $n$.

\begin{figure}[tbhp]
    \centering
    \includegraphics[width=\linewidth]{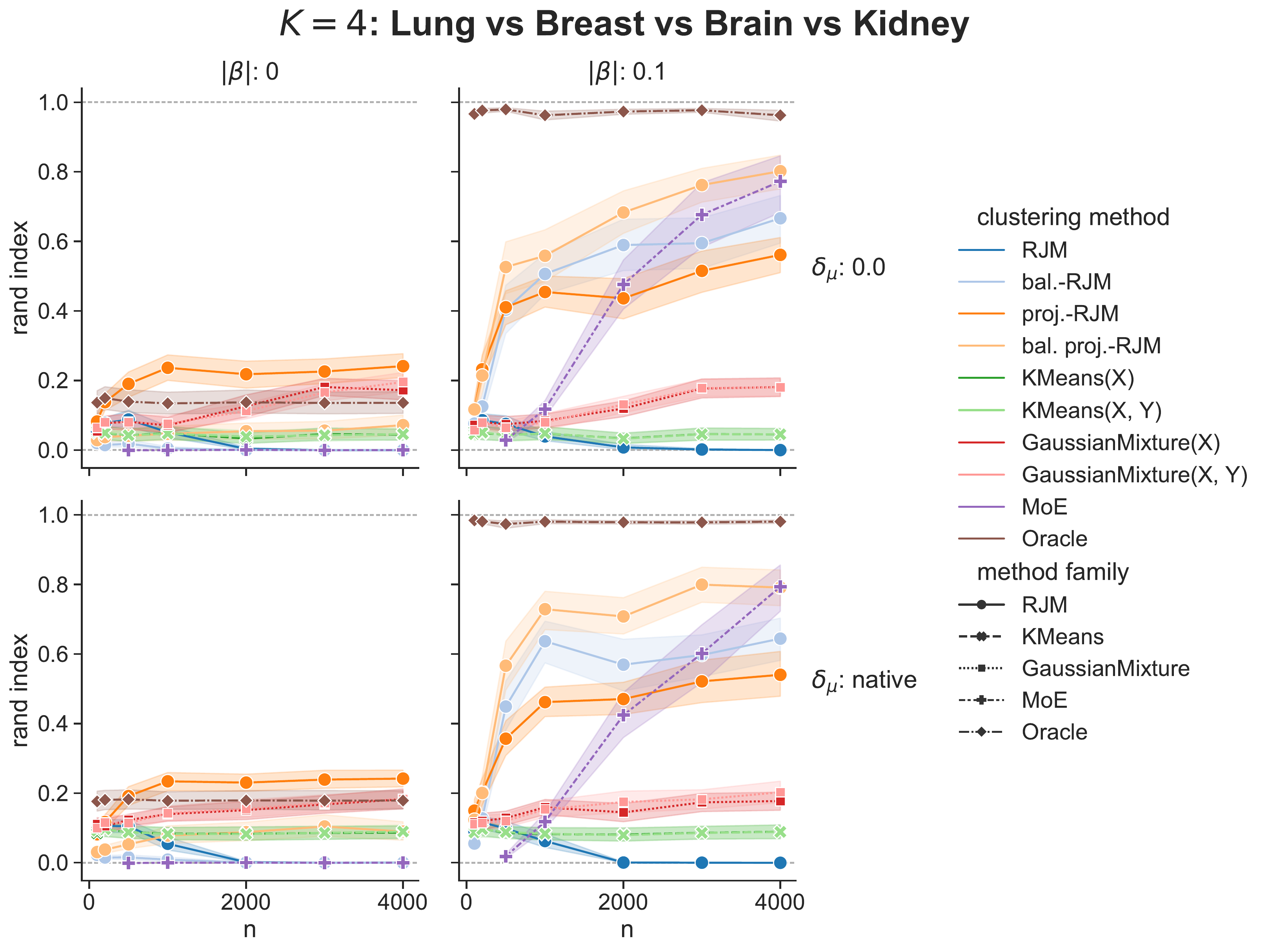}
    \caption{Cancer data, Rand Index vs. sample size $n$. RNA sequencing data from The Cancer Genome Atlas (TCGA) with dimension (i.e. size of ambient space) fixed to $p=100$ (see text for details). There are $K=4$ latent subgroups (corresponding to bronchus-lung, breast, brain and kidney samples).} 
    \label{fig:TCGA_vs_n_K4}
\end{figure}

\medskip 
\noindent
{\it Higher-dimensional setting}.
Next, for a fixed sample size $n = 500$, we increase the problem dimension $p$ from $100$ to $10000$. Regarding the signal potentially added by these new dimensions, we explore several scenarios. \figurename~\ref{fig:TCGA_vs_p_K4_one_coeff} displays the results in a ``worst case scenario" in which there is only a single non-zero coefficient in each $\beta_k$ (with none added as $p$ increases). 
The signal in $x$ is controlled as described above, i.e. either left as in the original data or removed entirely. We observe that, with signal in $y$ and no mean difference (top right), the two balanced RJMs perform the best, followed by proj.-RJM, MoG($x, y$) and KMeans($x, y$). As expected, all methods get worse as $p$ increases, since by design there is no new signal as the dimension increases, only more noise (with the exception of some potential second order or higher signal added in $x$). When there is no signal in the means and no signal in $y$ (only covariance signal, top left), even the Oracle does not perform well (however, there is still some detectable signal). Here, projected RJM is the only method that maintains a relatively acceptable level of performance, and outperforms the other methods for all values of $p$. 

With the original mean difference (bottom row; i.e. mean difference as in the original data), the Oracle 
does not degrade so strongly with $p$ (likely due to mean signal in additional components of $x$). With no signal in $\beta$ (bottom left figure), even the Oracle does not perform well. Nevertheless, some methods still recover signal. Projected RJM maintains performance when $p$ is not too large. For larger $p$, the two MoG models, which maintain consistent levels of performances, take over. When there is signal in $y$ (bottom right), we observe that the three regularised RJMs perform well for low $p=100$ (worsening greatly as $p$ increases). When $p$ is very large, the two projected RJMs match the MoGs. (Note that in all these cases, MoE was only computed for the smallest $p=100$, due to computational issues when $p$ is too large, in particular when $p>n$.)

\figurename~\ref{fig:execution_times} shows the execution times of the different methods. This figure demonstrates the computational benefits of the projection step in the RJMs. As the dimension grows, the computation times for the ambient space RJMs 
become long whereas those for the projected RJMs grow much more slowly.

In Appendix \ref{app:exp_tcga}, Figures \ref{fig:TCGA_vs_p_K4} and \ref{fig:TCGA_vs_p_K4_fixed_fraction} explore two alternative scenarios, where there are respectively 10 and $1\% p$ non-zero coefficients in each $\beta_k$. For these two experiments in the Appendix, we also re-scale $\bbeta$ by the number of non-zero coefficients, hence the amplitudes are actually $|\beta|/10$ and $|\beta|/(1\% p)$ respectively. This re-scaling reduces the signal in $y$. The observations and conclusions are similar across all three experiments.

\begin{figure}[tbhp]
    \centering
    \includegraphics[width=\linewidth]{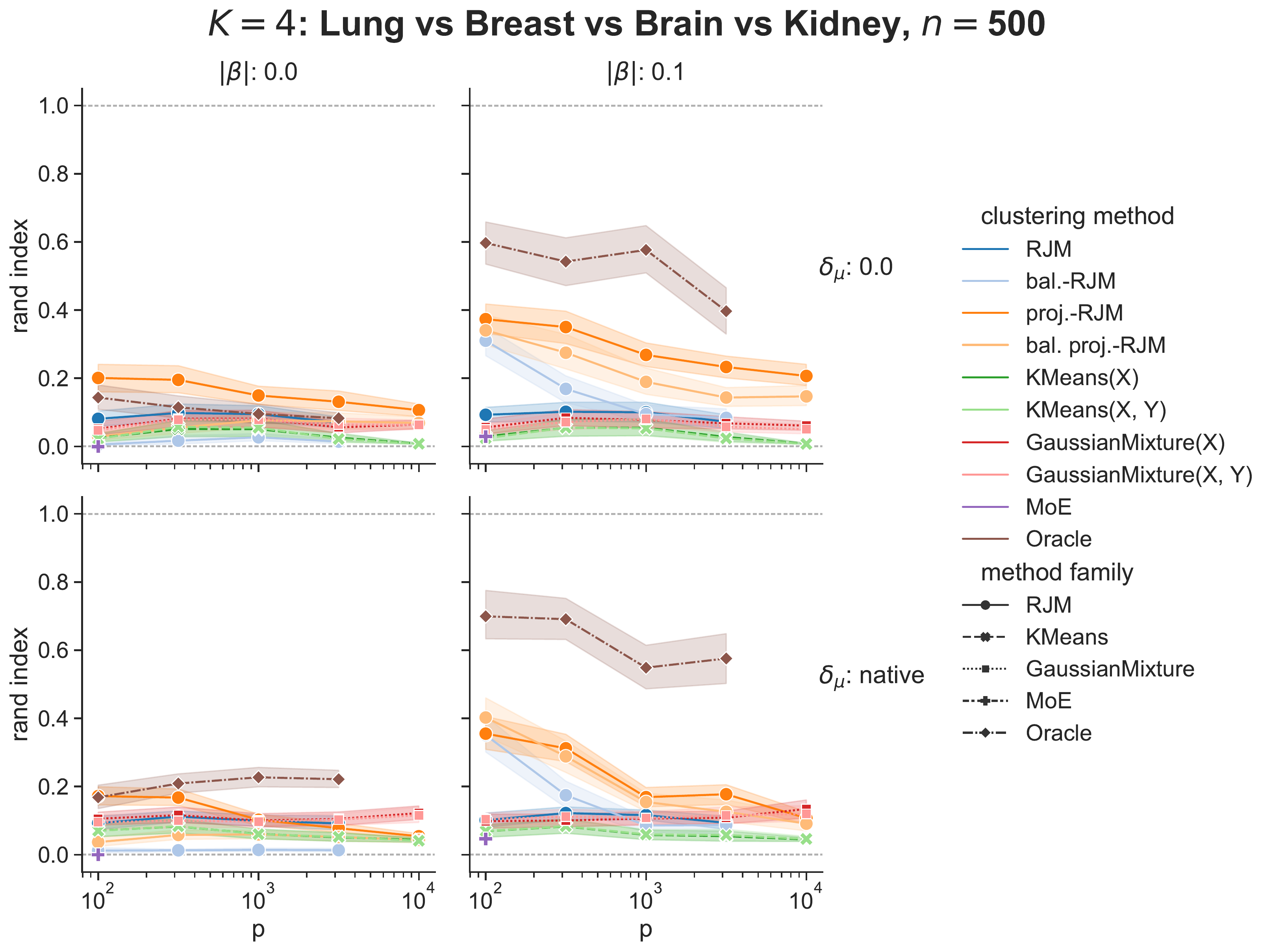}
    \caption{Cancer data, Rand Index vs dimension $p$ (log scale). RNA sequencing data from TCGA, with sample size fixed to $n=500$ and $K=4$ throughout.
    For all $p$, there is exactly one non-zero coefficient in $\beta_k$. }
    \label{fig:TCGA_vs_p_K4_one_coeff}
\end{figure}

\begin{figure}[tbhp]
    \centering
    \includegraphics[width=0.75\linewidth]{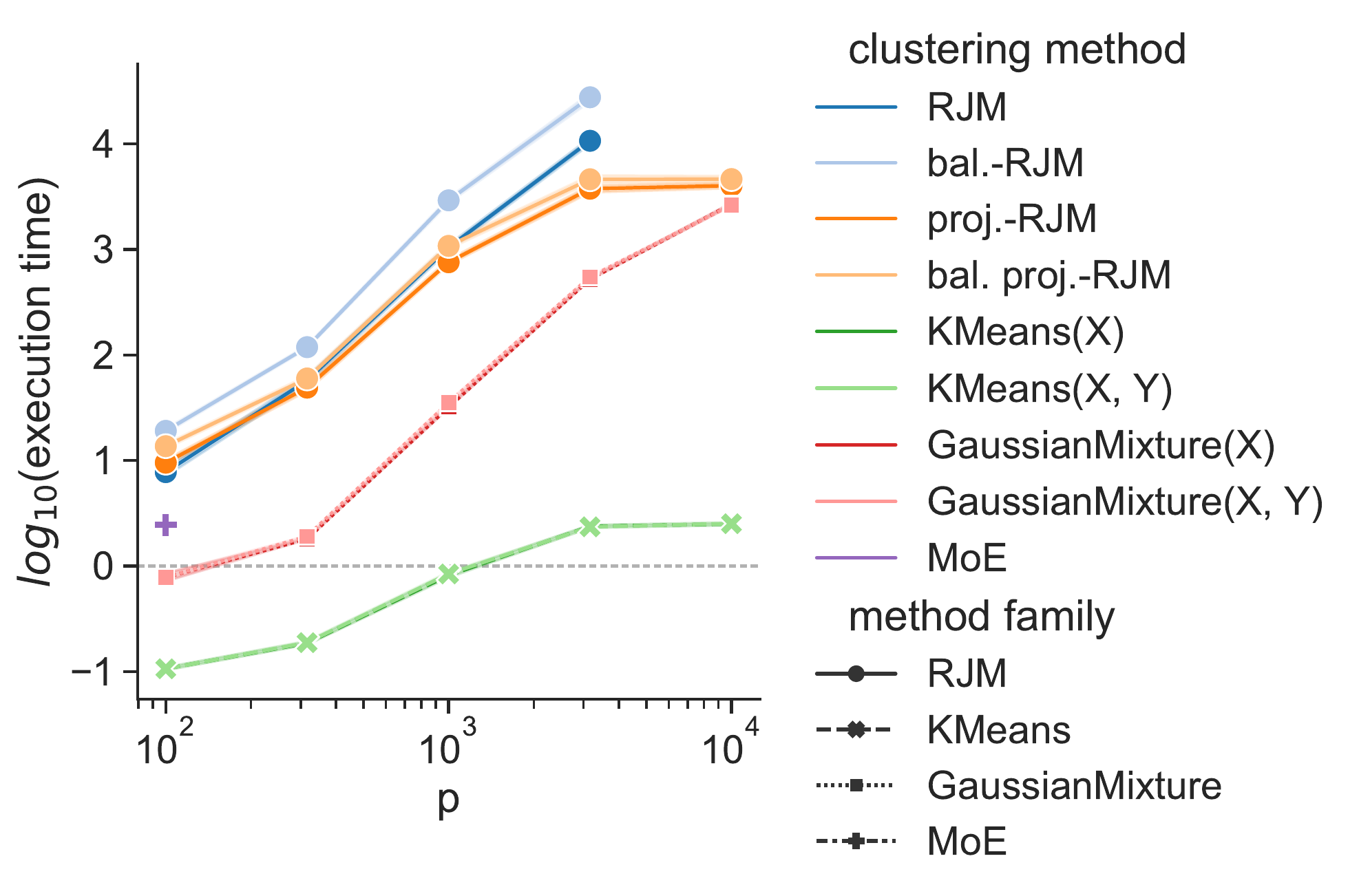}
    \caption{Execution (wall clock) time vs problem dimension $p$ with TCGA data (both axes on logarithmic scale). $K=4$, experimental settings as in \figurename~\ref{fig:TCGA_vs_p_K4_one_coeff}. Due to the projection step, the computational cost of the projected RJMs levels off relative to the ambient space approaches.}
    \label{fig:execution_times}
\end{figure}

\medskip 
\noindent
{\it Autoencoder-RJM}.
Following \citet{taschler2019model} and \cite{lartigue2022unsupervised},
the projected RJMs above use PCA as the data reduction step. 
Due to the modularity of the overall scheme, other embeddings could be used, including nonlinear embeddings tailored to specific applications.
For example, given a neural embedding in a given applied setting that transforms high-dimensional data into low-dimensional, more Gaussian-like form, RJM could be run on the embedded data as 
an interpretable and computationally lightweight step within a more complex overall workflow. 
As a basic illustration of this modularity, we run proj.-RJM and bal. proj.-RJM on the TCGA dataset ($K=4$) using a nonlinear autoencoder. We consider the highest dimensional setting of \figurename~\ref{fig:TCGA_vs_p_K4_one_coeff} ($p{=}10000$) and use a basic architecture with one hidden layer of size 512 (i.e. in dimensional terms: $10000 \rightarrow 512 \rightarrow q \rightarrow 512 \rightarrow 10000$). 
The experimental settings are as in the lower right quadrant of \figurename~\ref{fig:TCGA_vs_p_K4_one_coeff}: mean difference $\delta_{\mu}$ is as in the original data and $|\beta|=0.1$, with only one non-zero coefficient among the $p=10000$. We also explore different embedding sizes $q$, from 4 to 32. For bal. proj.-RJM, the balancing parameter is $T=q^{0.5}$. \figurename~\ref{fig:autoencoder1} shows Rand Index vs. sample size $n$. Autoencoder-RJM performs well, especially for embedding sizes higher than $q=4$. Additionally, we observe that the best embedding size
can vary, echoing the point previously made about embedding size trade-offs, as discussed in more detail in \cite{lartigue2022unsupervised}.

\begin{figure}[tbhp]
    \centering
    \includegraphics[width=\linewidth]{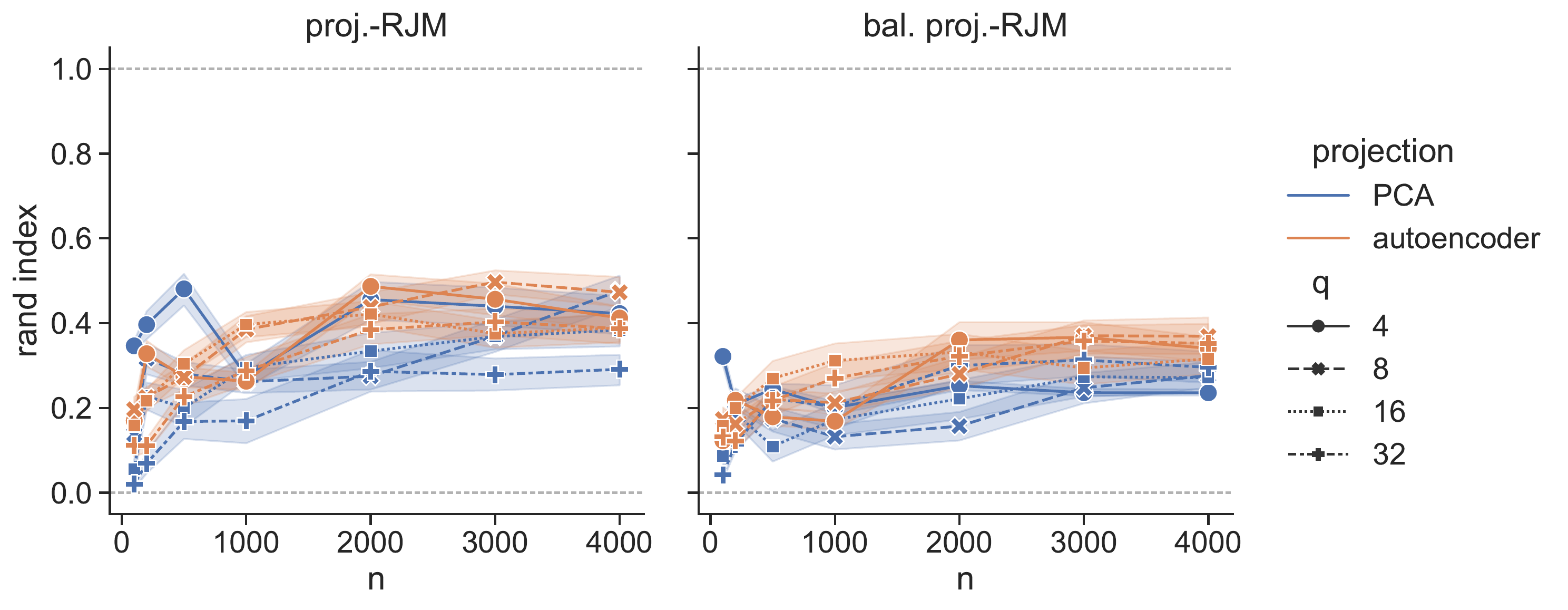}
    \caption{Cancer data, Rand index vs. samples size $n$, comparison between PCA and autoencoder embeddings. TCGA RNA sequencing data as above, with problem dimension (number of genes) fixed to $p=10000$. The original mean difference in $x$ is retained; signal in $y$ is of magnitude $|\beta|=0.1$ (see text for details). For each of the $K=4$ latent classes, only one coefficient (among the 10000) in $\beta_k$ is non-zero. Two projected RJM variants are: proj.-RJM (left) and bal. proj.-RJM (right).}
    \label{fig:autoencoder1}
\end{figure}

\subsection{Fully adaptive S-RJM} \label{sec:exp_mod_sel}
In the experiments above, we used fixed values for two of the RJM hyper-parameters: the embedding size $q$ of the projected RJM and the balancing exponent $\frac{1}{T}$ of the balanced RJM. As shown in Appendix \ref{app:exp_mod_sel}, different values of these parameters can be important. Likewise, the true number of latent subgroups $K$ is usually unknown and must be estimated by the clustering method and in general these aspects are coupled. 
In this section, we propose an adaptive scheme to set $q$ and $K$ (but not $T$) for RJM. The parameter $T$ controls the relative importance of $x$ and $y$ and we think this should be defined by the nature of the applied problem.
Default values such as $T=1$ (no re-balancing), $T=q$ in the projected space, or $T = p$ in the ambient space are all reasonable candidates that could fulfil different needs. 

\medskip
\noindent 
{\it Data-adaptive procedure}.
We set $K$ and $q$ jointly using a combination of classical Information Criteria (IC) and a stability measure. 
For any specific combination $(q,K)$ we define these quantities as described below. 
First, the IC, which can be either the Akaike Information Criterion \citep[AIC,][]{akaike1974new} 
   $ \text{AIC} = - 2\ln p(\by, \bx ; \widehat{\btheta}) + 2  df$, 
or the Bayesian Information Criterion \citep[BIC,][]{schwarz1978estimating}
$\text{BIC} =- 2 \ln p(\by, \bx ; \widehat{\btheta}) + df \ln(n)$,
where $\widehat{\btheta}$ is the model parameter estimated by the algorithm, $p(\by, \bx ; \widehat{\btheta})$ is the corresponding observed likelihood without regularisation \eqref{eq:hierarchical_conditional_gaussian_density}, and $df=K(3 + p +q(q+3)/2)$ is the number of degrees of freedom in RJM. Second, a stability criterion taken from \cite{taschler2019model}. 
This is described in detail in the reference and summarized in the following. The general idea is to quantify stability 
of the label assignments under data perturbation whilst controlling overall computational costs. To do so, we sample the data to produce 5 (overlapping) folds each containing 75\% of the total observations. Then, we carry out latent class learning on each fold. For each of the 10 different pairs of folds, we compute the Rand index between the two estimated label assignments (for only those samples shared between the respective folds) and return the average as the stability score (higher scores indicate more stability).

With these elements, the procedure is as follows. First, we define a grid of values of $(K, q)$ to assess. For each $q$, we get the 
IC-optimal 
number of latent subgroups $\widehat{K}(q)$. Then, we select among the candidate pairs $(q, \widehat{K}(q))$ to maximise the stability score. The result is the pair $(\widehat{q}, \widehat{K}(\widehat{q}))$. This selection methodology is designed for projected RJM. To adapt it to the ambient space RJM, the $q-$selection step is omitted. See Appendix \ref{app:exp_mod_sel}, for an illustrated analysis of the different steps of the selection procedure.

\medskip
\noindent 
{\it Experimental configuration}.
We use the TGCA dataset with all $K^*=4$ subgroups. Data and experimental settings are as in Section \ref{sec:exp_tcga}: we sub-sample $p=100$ genes and $n=500$ observations, keep the original $\mu$, and generate responses $y$ with a $\bbeta$ of magnitude $|\beta|=0.1$. Each $\beta_k$ has 10 non zero coefficients (with mutually disjoint sparsity patterns). We include the adaptive version (using AIC and stability, as outlined above) of all four RJM variants 
(for the two ambient RJMs, $q{=}p$ is fixed and only $\widehat{K}$ is estimated). We take $T=p$ for the balanced version. For the two projected RJMs, both $\widehat{K}$ and $\widehat{q}$ are estimated. We take $T=\widehat{q}$ for the balanced version. We also run the non-adaptive version of each of the four RJMs. To see whether adaptive methods can compete with knowledge of the correct $K$ for the non-adaptive methods we set $K=K^*=4$. Additionally, we run the non-adaptive projected RJMs with all values of $q$ from the grid of candidate embedding sizes of adaptive projected RJM. We also run all the other clustering methods from the state of the art, all of them with a fixed number of estimated subgroups equal true value: $K=K^*=4$. 
We run 50 independent simulations.


\medskip
\noindent 
{\it Results}.
\figurename~\ref{fig:model_selection_full_rand_index} compares adaptive RJM variants to the other methods. 
The adaptive proj.-RJM (pink) reaches the same Rand Index as the non-adaptive proj.-RJM with the better value of $q$ ($q=5$ here). Interestingly, adaptive balanced projected RJM (light pink) is actually better than its non-adaptive version regardless of the value of $q$. This is striking, given that the non-adaptive variants have access to 
the true number of clusters and a label oracle-informed procedure for setting $q$, while the adaptive methods have to proceed in an entirely data-driven manner. This suggests that 
setting 
$\widehat{q}$ on a case-by-case basis can actually improve on a embedding size $q$ that is {\it fixed} for all realisations.
However, both versions of the adaptive ambient space RJM (turquoise and light turquoise) under-perform. 


\begin{figure}[tbhp]
    \centering
    \includegraphics[width=\linewidth]{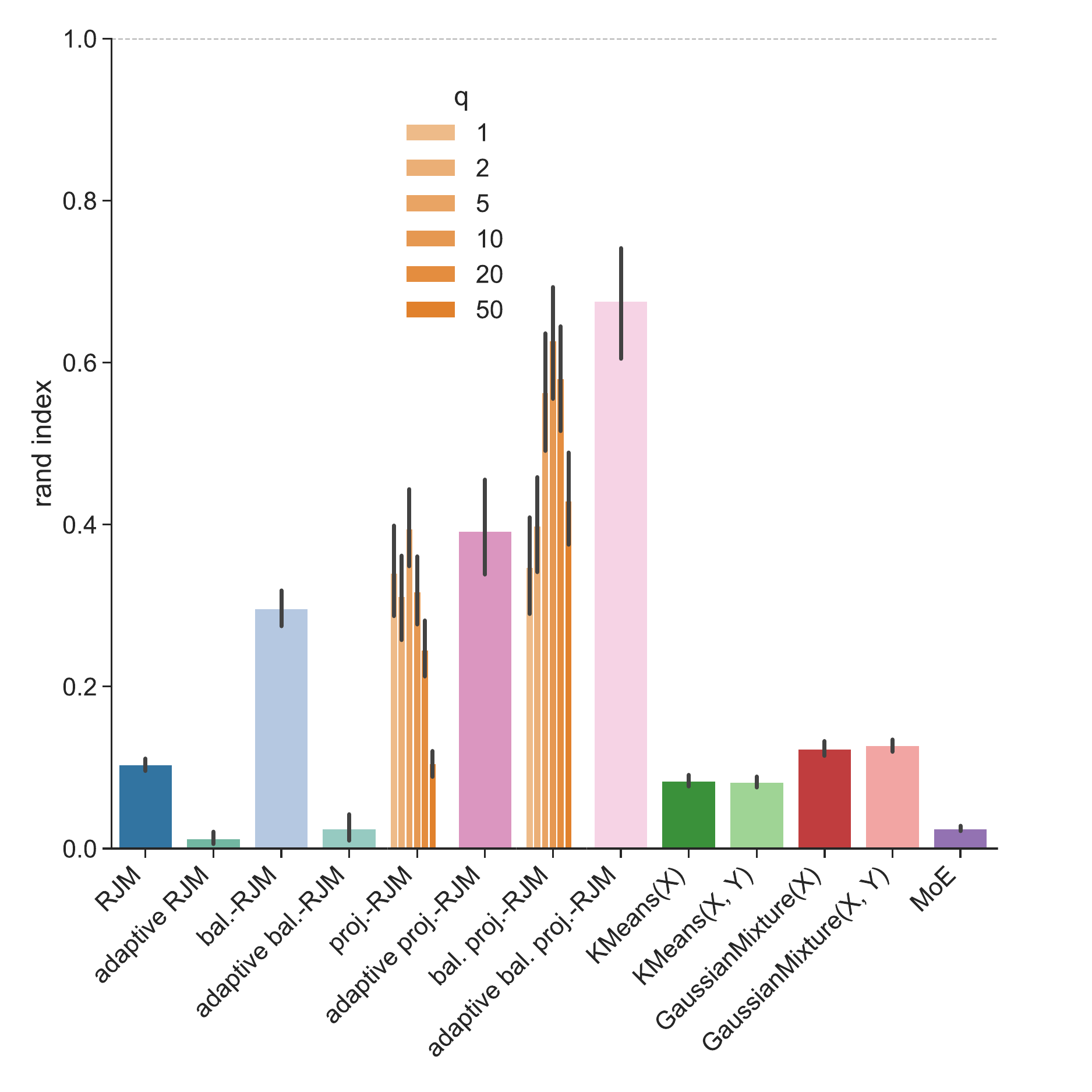}
    \caption{Rand Index of the adaptive RJM compared to the other methods (which know and use the true value $K^*=4$). Adaptive RJM selects $\widehat{K}$ with AIC, and proj.-RJM selects $(\widehat{K}, \widehat{q})$ with AIC combined with stability (see text for details). The non-adaptive projected RJM is run with several different values for the embedding size $q$, represented by an orange gradient (the set of values for $q$ is the same as for the adaptive projected RJM).}

    \label{fig:model_selection_full_rand_index}
\end{figure}

\section{Discussion and conclusion}
We introduced a scalable framework for the joint mixture modelling of feature distributions and regression models specific to latent subgroups. We built on the Regularised Joint Mixtures (RJM) model and algorithm of \cite{perrakis2019latent},
addressing in particular issues that arise in high dimensions via a combination of data reduction and re-weighting of the regression and feature model components of the joint model. 
The schemes we introduced are modular, in the sense of allowing different data reductions to be combined with high-dimensional learners and we showed examples including PCA and an autoencoder as the data reduction and a range of sparse learners, including non-Gaussian models. Taken together, our results provide a way to learn interpretable high-dimensional regression models and sparse feature distributions under conditions of latent heterogeneity that could otherwise severely confound the models.

In the interests of making S-RJM as user-friendly as possible, we also put forward a fully adaptive scheme (that sets automatically the embedding dimension $q$ and number of clusters $K$). We found that this fully adaptive scheme 
was highly effective on the real RNA-sequencing data, matching the performance of variants with oracle-like information used to set the hyper-parameters. 
We proposed a number of variants of the general S-RJM scheme, but would recommend in particular the balanced projected RJM for most situations. This combines data reduction with a default re-weighting and we think provides a good default.

Our work was motivated by, and illustrated using real data from, the biomedical domain. In many areas of biomedicine, latent heterogeneity is of crucial interest in defining disease subtypes. Such subtypes may differ in subtle ways, for example at the level of covariance or graphical model structure \citep{stadler2017} rather than the kind of large mean difference that classical K-means and related models are geared towards. At the same time, due to the fact that disease subtypes involve differences in the underlying molecular biology, predictive models themselves may differ between subtypes, hence the need to also account for the possibility of non-identical sparsity patterns \citep{dondelinger2020}.
Results on several real and simulated datasets, for a wide range of class signal types and intensities, support the notion that S-RJM is capable of learning latent structure and latent group-specific parameters under challenging circumstances, often outperforming standard models. 
Hence, we think S-RJM will be a relevant tool in dissecting heterogeneity in stratified and personalized medicine, as it allows efficient high-dimensional modelling with interpretable outputs. 

Another related, but more general, setting in which we think S-RJM can contribute is as a lightweight first stage tool for out-of-distribution (OOD) learning. OOD learning has been studied in the context of causality \citep[see e.g][]{peters2016,arjovsky2019} and also in biomedical applications \citep[e.g.][]{warnat2020}, where concerns arise due to the fact that in practical use cases distributions at  deployment time may differ from the training distribution, such that non-robust high-dimensional learners may perform unreliably in practice. 
Since S-RJM, and in particular the balanced variants, place an emphasis on learning latent structure associated with differences in regression models and explicit differences in feature distributions, it could be used as a pre-processing step to identify distributional regimes (corresponding to the latent groups) using which OOD methods could be trained and tested (i.e. in a cross-regime sense). 

\bigskip
\noindent
{\bf Code.} Code to run S-RJM and reproduce the simulation results is publicly available at \url{https://github.com/tlartigue/Scalable-Regularised-Joint-Mixture-Models}.

\bigskip
\noindent
{\bf Acknowledgements.} {The results of this paper are in part based upon data generated by the TCGA Research Network: \url{https://www.cancer.gov/tcga}.
Supported by the Bundesministerium f\"{u}r Bildung und Forschung (BMBF) project ``MechML”.}

\newpage

\appendix
\section{Shrinkage on the precision matrix within the EM}\label{app:oas}
In this appendix, we give more details on the shrinkage on $\bOmega$ used in the ambient space RJM-EM.

When the EM is run on the projected $W^t x\in \R^q$, and the embedding size $q$ is small compared to the sample size $n$, then, it is computationally feasible to have no penalty at all on $\bOmega$ in the M-step of the EM. However, when the EM is performed in the ambient space, this regularisation effect is absent. A natural way to deal with singular covariance matrices in high dimension is with a penalty on the nuclear norm: $\pen(\Omega_k|\delta) = \frac{\delta}{2} \tr(\Omega_k)$, with a $\delta \in \R^*_+$. This is a proper penalty, with which the RJM-EM runs smoothly. For our experiments however, we have made the choice to adopt an adaptive shrinkage intensity $\delta$ in order to avoid unnecessary loss of information through excessive shrinkage when the empirical covariance estimate is already close to regular. To that end, instead of a simple penalty term, we apply at each M-step a state of the art adaptive shrinkage on $\Omega_k$: the \textit{Optimal Approximation Shrinkage} of \cite{chen2010shrinkage}. This shrinkage takes the form:
\begin{equation*}
     \widehat{\Omega}_k^{-1} = (1-\hat{\delta}) S_k  + \hat{\delta} \frac{\tr(S_k)}{p} I_p \, .
\end{equation*}
Where $S_k$ is the current empirical covariance estimate, and $\hat{\delta}$ is set adaptively from $S_k$ at each M-step \cite[see][Eq. 23]{chen2010shrinkage}. To fit the penalised EM formalism, we can write this shrunk $\widehat{\Omega}_k$ as the solution of the penalised M-step:
\begin{equation*}
    \widehat{\Omega}_k^{-1} := \argmin{\Omega_k}\brace{ \frac{1}{2} \tr(\Omega_k S_k) - \frac{1}{2} \ln \det{\Omega_k} + \pen(\Omega_k |S_k)} \, ,
\end{equation*}
with the penalty:
\begin{equation*}
    \pen(\Omega_k| S_k) = \frac{\hat{\delta}}{2} \tr\parent{\parent{\frac{\tr(S_k)}{p}I_p-S_k}\Omega_k} \, .
\end{equation*}
However, this penalty is technically improper in the context of an EM algorithm. Indeed, because of its dependency of $S_k$, which is re-estimated at each new step, this penalty causes the overall objective function $l(\btheta)$ to change slightly at each step. Despite this, we observe in practice that the ambient space EM is well behaved with this adaptive shrinkage. For any user that prefers a more theoretically sound objective function, we recommend the classical non-adaptive nuclear norm regularisation.

\section{EM Convergence Theorem} \label{app:thm_dlm}
In this appendix, we recall Theorem 1 of \cite{delyon1999convergence}, which is the basis for our own Proposition \ref{thm:convergence_EM} in Section \ref{sec:theory}.

For $n\in \bN^*$, \cite{delyon1999convergence} consider a $\sigma$-finite positive Borel measure $\nu$ on $\R^n$ and a family of positive integrable Borel functions on $\R^n$: $\brace{\bz \in \R^n \mapsto f(\bz ; \btheta) \mid \btheta \in \Theta}$. Where the parameter set $\Theta$ is a subset of $\R^l$ for some $l \in \bN^*$. They define the functions: $g(\btheta) = \int_{\bz\in \R^n} f(\bz, \btheta) \nu(d\bz)$, $l(\btheta) := -\ln g(\btheta)$ and 
\begin{equation*}
    p(\bz; \theta) :=
  \begin{cases}
    f(\bz, \btheta)/g(\btheta)       & \quad \text{if } g(\btheta) \neq 0\\
    0  & \quad \text{if } g(\btheta) = 0 \, .
  \end{cases}
\end{equation*}
They define the following regularity conditions:
\begin{itemize}
    \item \textbf{M1.} $\Theta$ is an open set of $\R^l$, and the function $f$ can be written: 
    \begin{equation*}
        f(\bz; \btheta) = exp\parent{-\psi(\btheta) + \dotprod{\tilde{S}(\bz), \phi(\btheta)}} \, .
    \end{equation*}
    Where $\dotprod{.,.}$ denotes the scalar product in $\R^m$ for some $m\in \bN^*$. $\tilde{S}$ is a Borel function from $\R^n$ to an open set $\S \subseteq \R^m$ such that the convex hull of $S(\R^n) \subseteq \S$. Additionally, $\forall \btheta \in \Theta:$
    \begin{equation*}
        \int_{\bz \in \R^n} \det{\tilde{S}(\bz)} p(\bz; \btheta) \nu(d\bz) < \infty \,.
    \end{equation*}
    
    \item \textbf{M2.} $\phi$ and $\psi$ are twice continuously differentiable on $\Theta$.
    
    \item \textbf{M3.} $\overline{s} : \Theta \longrightarrow \S$ defined as:
    \begin{equation*}
       \overline{s}(\btheta) := \int_{\bz \in \R^n} \tilde{S}(\bz) p(\bz; \btheta) \nu(d\bz)\,,
    \end{equation*}
    is continuously differentiable in $\btheta$.
    
    \item \textbf{M4.} The function $l$ is continuously differentiable on $\Theta$ and
    \begin{equation*}
        \nabla_{\btheta} \int_{\bz \in \R^n} f(\bz; \btheta) \nu(d\bz) = \int_{\bz \in \R^n} \nabla_{\btheta} f(\bz; \btheta) \nu(d\bz)
    \end{equation*}
    
    \item \textbf{M5.} Let $L(s, \btheta) := -\psi(\btheta) + \dotprod{s, \phi(\btheta)}$ There exists a function $\widehat{\btheta} : \S \longrightarrow \Theta$ continuously differentiable such that:
    \begin{equation*}
        \forall \btheta \in \Theta, \quad \forall s\in \S, L(s, \widehat{\btheta}(s)) \geq L(s, \btheta) \, .
    \end{equation*}
\end{itemize}

With this formalism, the EM sequence is defined from an initial point $\btheta^{(0)}$ with the recurrence formula: $\btheta^{(t+1)} := \widehat{\btheta} \circ \bar{s}(\btheta^{(t)})$. \cite{delyon1999convergence} define $\L(\btheta)$ the set of limit points of the sequence $\btheta^{(t)}$ initialised with $\btheta^{(0)}=\btheta$. Let clos($A$) denote the closure of a set $A$ and $d(x, A)$ the distance between a point $x$ and a closed set $A$. The following statement is an amalgamation of Theorem 1 and related results from \cite{delyon1999convergence}:
\begin{theorem}[Theorem 1 of \cite{delyon1999convergence}]
    Assume that (M1)–(M5) hold. Then, $l(\widehat{\btheta} \circ \bar{s}(\btheta))\leq l(\btheta)$ with equality if and only if $\widehat{\btheta} \circ \bar{s}(\btheta) =\btheta$. Moreover, the set of fixed points of the EM is equal to the set of stationary points of $l(\btheta)$: $\L:=\brace{\btheta \in \Theta \mid \btheta = \widehat{\btheta} \circ \bar{s}(\btheta)} = \brace{\btheta \in \Theta \mid \nabla_{\btheta} l(\btheta) = 0}$. If we assume in addition that for any $\btheta \in \Theta$, clos($\L(\btheta)$) is a compact subset of $\Theta$, then, for any initial point $\btheta^{(0)}=\btheta$, $\L(\btheta)\subseteq\L$ 
    and $\underset{t \to \infty}{lim} d(\btheta^{(t)}, \L(\btheta)) \longrightarrow 0$. 
\end{theorem}

\newpage
\section{Gaussian simulations} \label{app:exp_gaussian}
In this appendix, we provide additional figures (\ref{fig:rand_index_vs_n_flipped_signs} to \ref{fig:AUC_beta_Omega_vs_n_flipped_signs}) to expand on the analysis of the synthetic Gaussian data experiments (Section \ref{sec:exp_gaussian}) of the main paper.

\begin{figure}[tbhp]
    \centering
    \includegraphics[width=\linewidth]{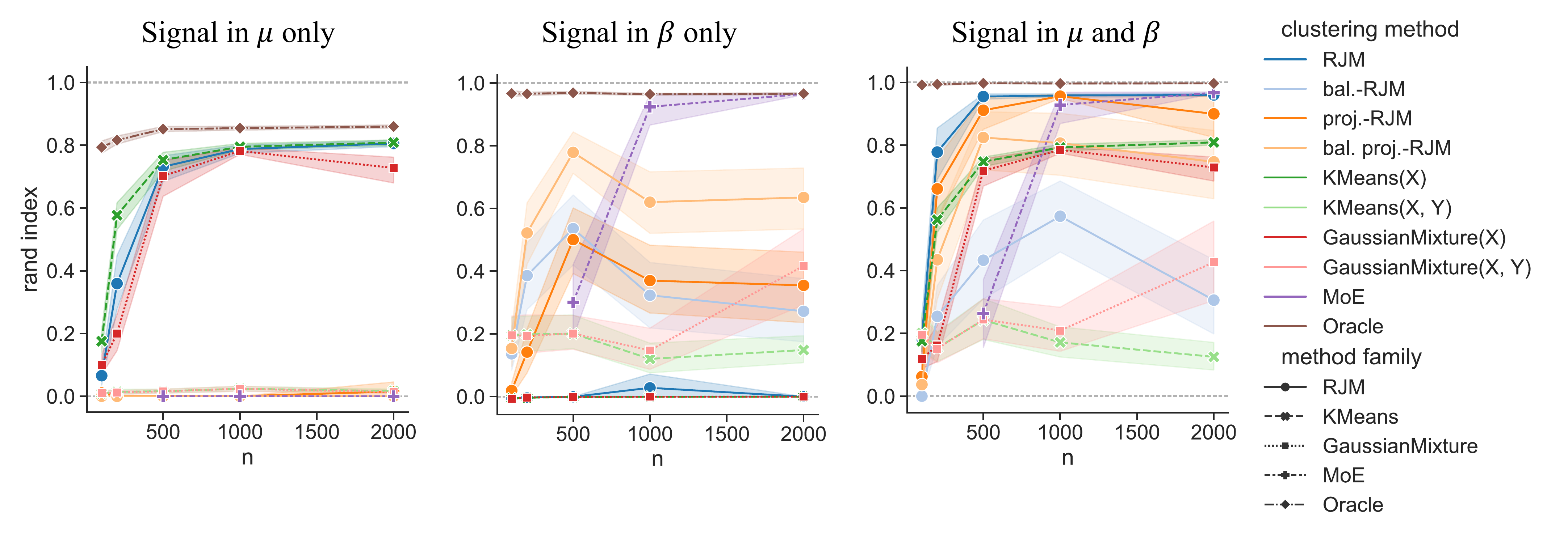}
    \caption{\textbf{Flipped signs}. Rand Index vs the sample size. Ambient space size $p=100$. (Left) Signal in $\mu$ only. $\delta_{\mu} = 0.2$. $\beta_1=\beta_2$. (Middle) Signal in $\beta$ only. $|\beta| = 5$. No overlap between the positions of non zeros coefficients between the two latent subgroups. (Right) Signal in both $\mu$ and $\beta$. $\delta_{\mu} = 0.2, |\beta| = 5$.}
    \label{fig:rand_index_vs_n_flipped_signs}
\end{figure}

\begin{figure}[tbhp]
    \centering
    \includegraphics[width=\linewidth]{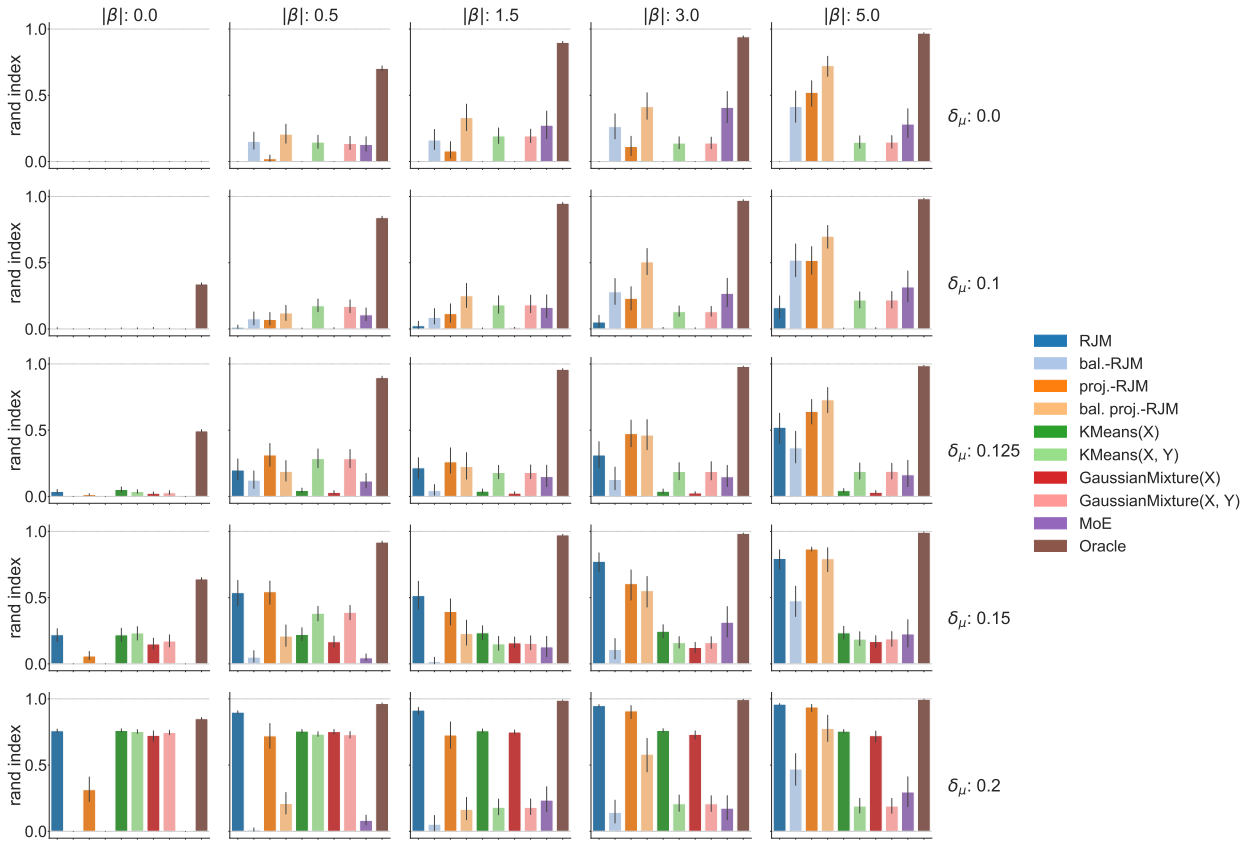}
    \caption{Rand Index of several methods over a grid of different forms and magnitudes. For each regime, we can identify the preferable methods. Ambient space size $p=100$. Sample size $n=500$. When low/no signal in $\mu$ and high signal in $\beta$, we observe that bal. proj.-RJM is better than proj.-RJM which is better than bal.-RJM which is better than MoE. ``Gradation" effect of the regularisation in RJM: each layer of regularisation of $x$ increases the performances. When the signal in $x$ increases, the regular RJM gradually becomes the preeminent method. With KMeans($x$) and MoG($x$) getting better as well. For some reason, when there is a strong signal in $x$, KMeans($x, y$) and MoG($x, y$) get worse as the signal in $y$ increases. Maybe the geometry gets confusing for these two model based classifiers when the signal is strong in both $x$ and $y$.}
    \label{fig:grid_signal}
\end{figure}

\begin{figure}[tbhp]
    \centering
    \includegraphics[width=0.8\linewidth]{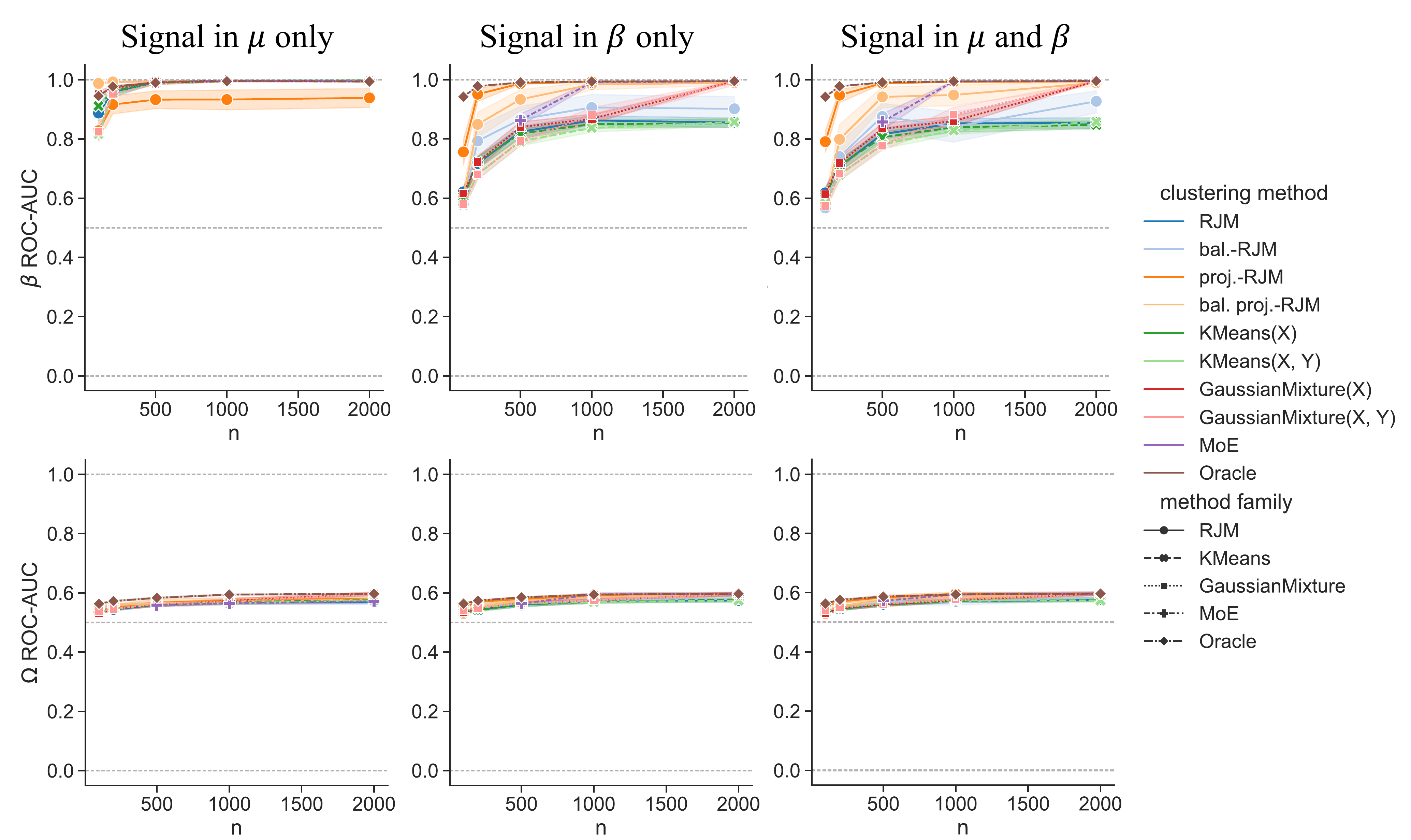}
    \caption{\textbf{ROC-AUC} on the recovery of the non zeros in $\bbeta$ (top) or $\bOmega$ (bottom) vs the sample size. Ambient space size $p=100$. (Left) Signal in $\mu$: $\delta_{\mu} = 0.2$. No signal in $\beta$: $\beta_1=\beta_2$. (Middle) Signal in $\beta$: $|\beta| = 5$, with no overlap between the non-zeros of the two latent subgroups. No signal in $\mu$: $\delta_{\mu}=0$ (Right) Signal in both $\mu$ and $\beta$. $\delta_{\mu} = 0.2, |\beta| = 5$. For all of these experiments, there is a covariance signal in $x$: $\Omega_1 \neq \Omega_2$. Where these two sparse precision matrices are generated independently. ROC-AUC version of \figurename~\ref{fig:PR_AUC_beta_Omega_vs_n_different_Omegas_more_edges}.}
    \label{fig:AUC_beta_Omega_vs_n_different_Omegas_more_edges}
\end{figure}

\begin{figure}[tbhp]
    \centering
    \includegraphics[width=0.8\linewidth]{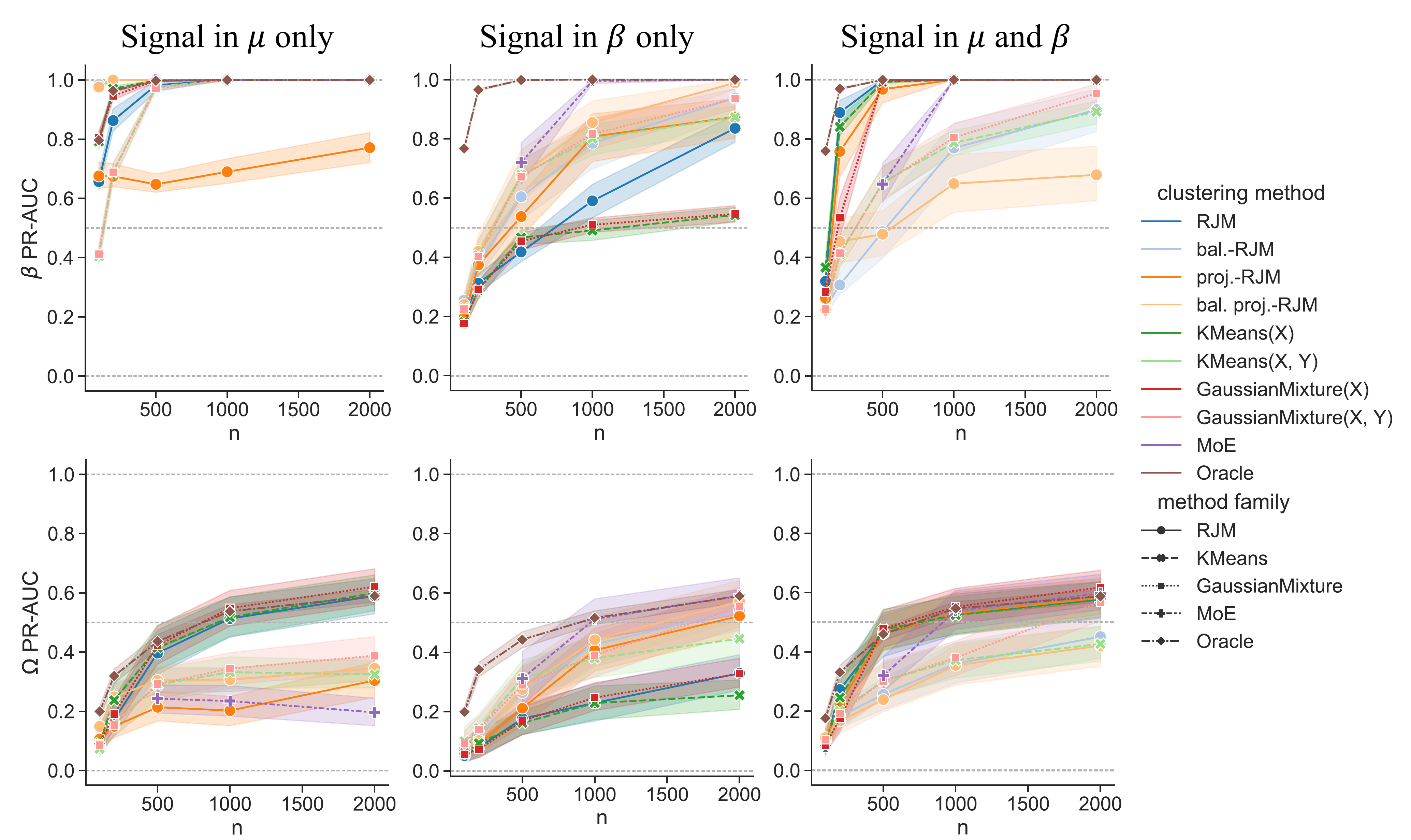}
    \caption{\textbf{Very sparse $\bOmega$}. Precision-Recall (PR) AUC on the recovery of the non zeros in $\bbeta$ (top) or $\bOmega$ (bottom) vs the sample size. Ambient space size $p=100$. (Left) Signal in $\mu$: $\delta_{\mu} = 0.2$. No signal in $\beta$: $\beta_1=\beta_2$. (Middle) Signal in $\beta$: $|\beta| = 5$, with no overlap between the non-zeros of the two latent subgroups. No signal in $\mu$: $\delta_{\mu}=0$ (Right) Signal in both $\mu$ and $\beta$. $\delta_{\mu} = 0.2, |\beta| = 5$. Same experiment as \figurename~\ref{fig:PR_AUC_beta_Omega_vs_n_different_Omegas_more_edges}, with much sparser $\bOmega$. The recovery of $\bOmega$ is much better.}
    \label{fig:AUC_beta_Omega_vs_n_different_Omegas}
\end{figure}

\begin{figure}[tbhp]
    \centering
    \includegraphics[width=0.8\linewidth]{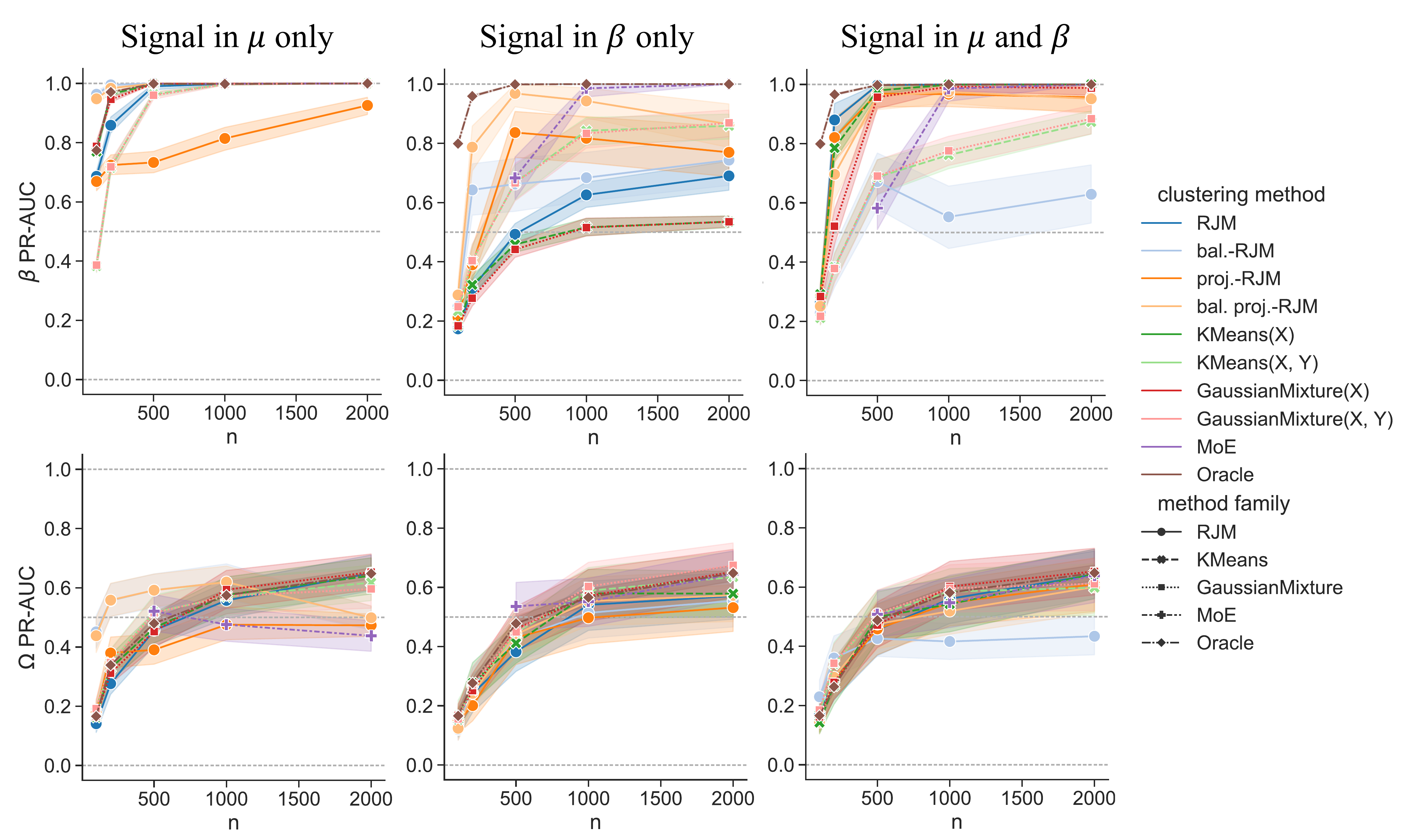}
    \caption{$\boldsymbol{\Omega_1=\Omega_2}$. Precision-Recall (PR) AUC on the recovery of the non zeros in $\bbeta$ (top) or $\bOmega$ (bottom) vs the sample size. Ambient space size $p=100$. (Left) Signal in $\mu$ only. $\delta_{\mu} = 0.2$. $\beta_1=\beta_2$. (Middle) Signal in $\beta$ only. $|\beta| = 5$. No overlap between the positions of non zeros coefficients between the two latent subgroups. (Right) Signal in both $\mu$ and $\beta$. $\delta_{\mu} = 0.2, |\beta| = 5$. Sparsity pattern recovery metrics of the experiment analysed in \figurename~\ref{fig:rand_index_vs_n}.}
    \label{fig:AUC_beta_Omega_vs_n}
\end{figure}

\begin{figure}[tbhp]
    \centering
    \includegraphics[width=0.8\linewidth]{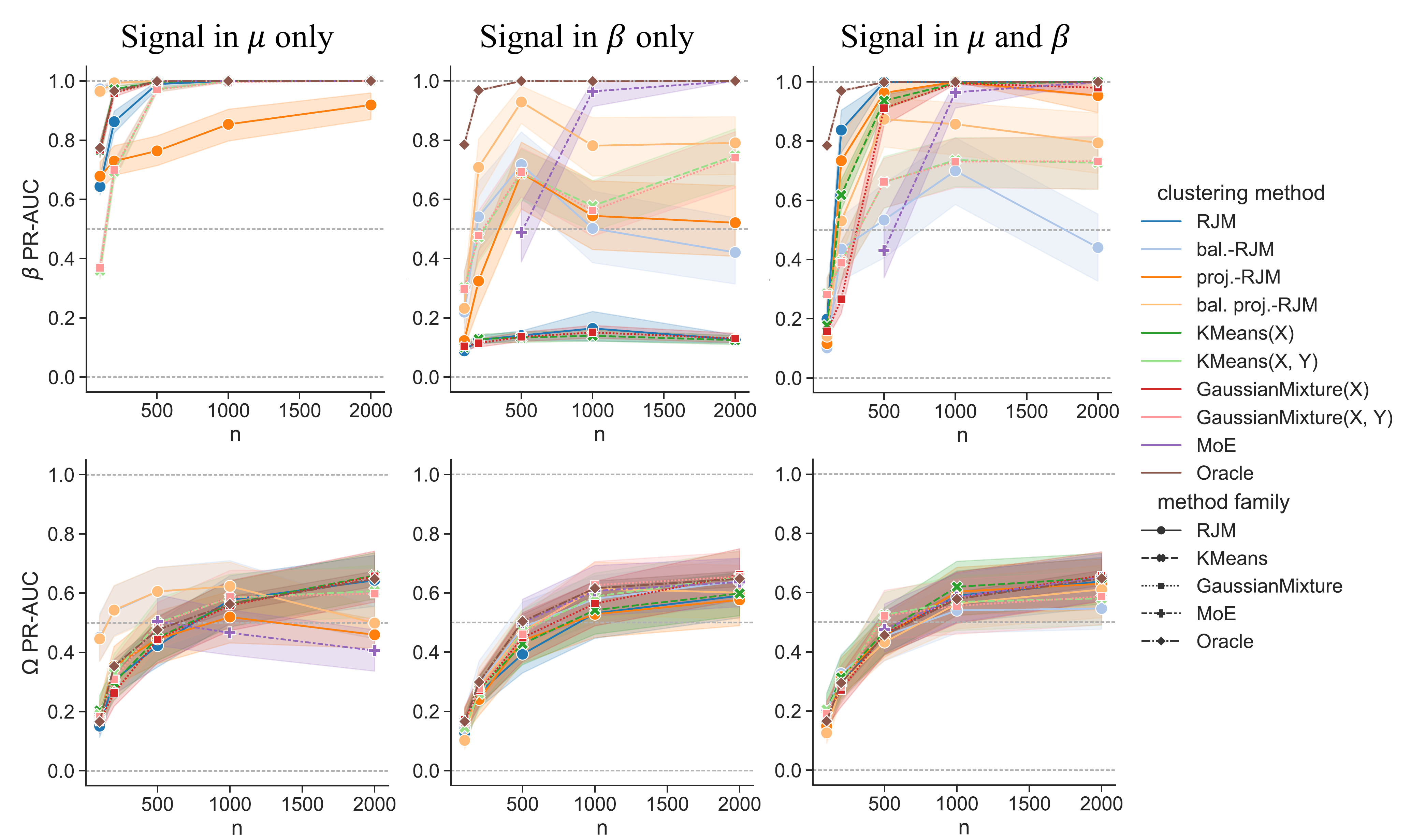}
    \caption{\textbf{$\boldsymbol{\Omega_1=\Omega_2}$, flipped signs in $\boldsymbol{\beta}$}. Precision-Recall (PR) AUC on the recovery of the non zeros in $\beta$ (top) or $\bOmega$ (bottom) vs the sample size. Ambient space size $p=100$. (Left) Signal in $\mu$ only. $\delta_{\mu} = 0.2$. $\beta_1=\beta_2$. (Middle) Signal in $\beta$ only. $|\beta| = 5$. (Right) Signal in both $\mu$ and $\beta$. $\delta_{\mu} = 0.2, |\beta| = 5$. Unlike \figurename~\ref{fig:AUC_beta_Omega_vs_n}, the positions of non zeros coefficients are the same for each of the two latent subgroups, but the coefficients have opposite values between the two latent subgroups: $\beta_1=-\beta_2$. (Right) Signal in both $\mu$ and $\beta$, $\beta_1=-\beta_2$. $\delta_{\mu} = 0.2, |\beta| = 5$. Sparsity pattern recovery metrics of the experiment analysed in \figurename~\ref{fig:rand_index_vs_n_flipped_signs}.}
    \label{fig:AUC_beta_Omega_vs_n_flipped_signs}
\end{figure}

\newpage

\section{Non-Gaussian simulated data.}\label{app:exp_non_gaussian}
In this appendix, we provide additional figures (\ref{fig:student_rand_index} to \ref{fig:mixed_binary_pr_auc}) to expand on the analysis of the Non-Gaussian data experiments (Section \ref{sec:exp_non_gaussian}) of the main paper.

\begin{figure}[tbhp]
    \centering
    \includegraphics[width=\linewidth]{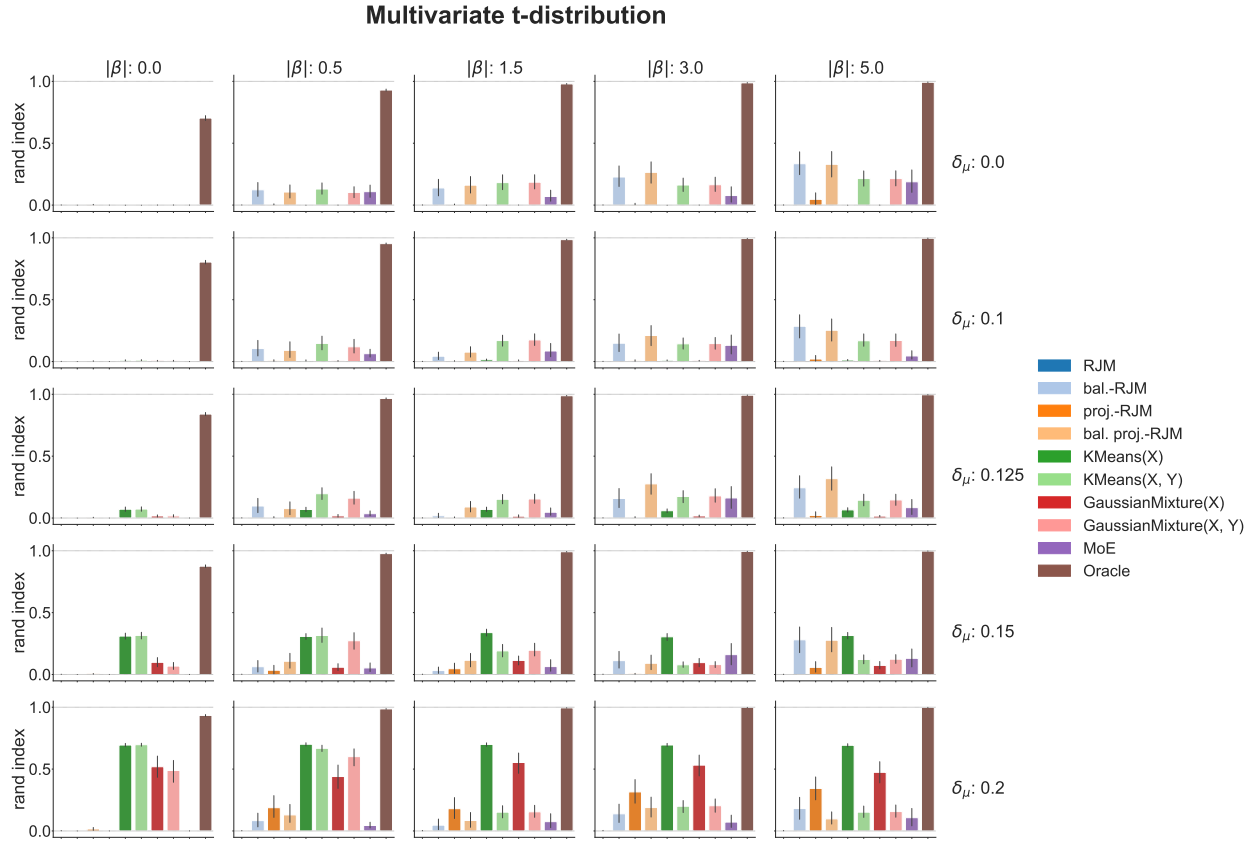}
    \caption{\textbf{t-student Rand Index.} Rand Index for different forms and intensities of signals. The data is simulated as a multivariate t-student with 10 degrees of freedom. The precision matrix are different: $\Omega_1 \neq \Omega_2$. The balanced versions (with or without projection) of RJM are the better choice when the prevalent signal in $\beta$. Whereas KMeans and MoG are preferable when there is a signal in $\mu$.}
    \label{fig:student_rand_index}
\end{figure}

\begin{figure}[tbhp]
    \centering
    \includegraphics[width=\linewidth]{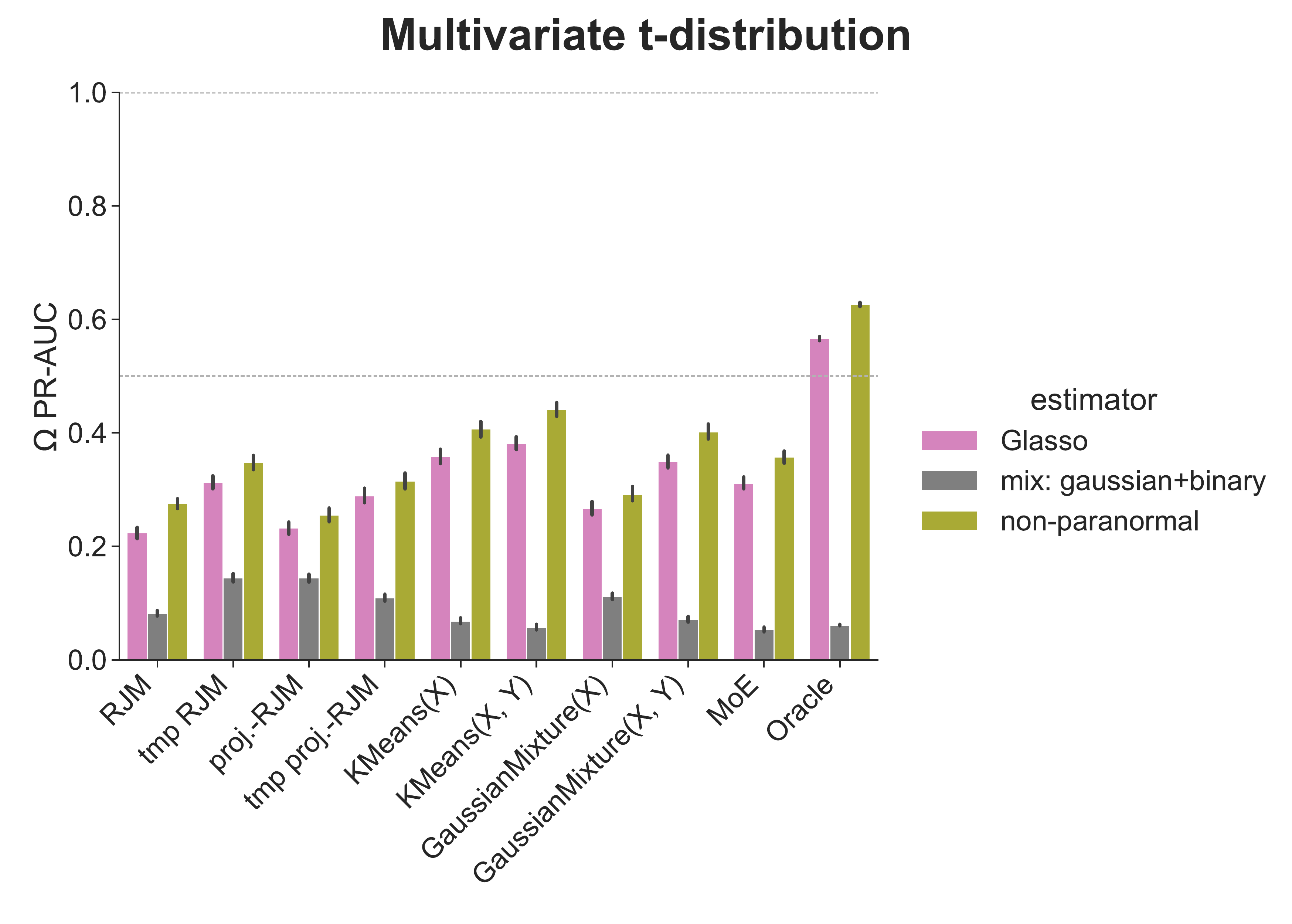}
    \caption{\textbf{t-Student sparsity pattern recovery}. PR AUC of the sparsity pattern recovery in the estimation of precision matrix $\bOmega$. The data is generated as a t-Student from the real $\bOmega$. We compare three different graph estimators always computed from the same (estimated or oracle) group-labels. The Glasso estimator comes from \cite{friedman2008sparse}, the non-paranormal from \cite{liu2009nonparanormal} and the ``mixed data" estimator from \cite{fan2017high}. With this data, the non-paranormal and Glasso estimators perform similarly, and dominate the specialised ``mixed data" estimator.}
    \label{fig:student_pr_auc}
\end{figure}

\begin{figure}[tbhp]
    \centering
    \includegraphics[width=\linewidth]{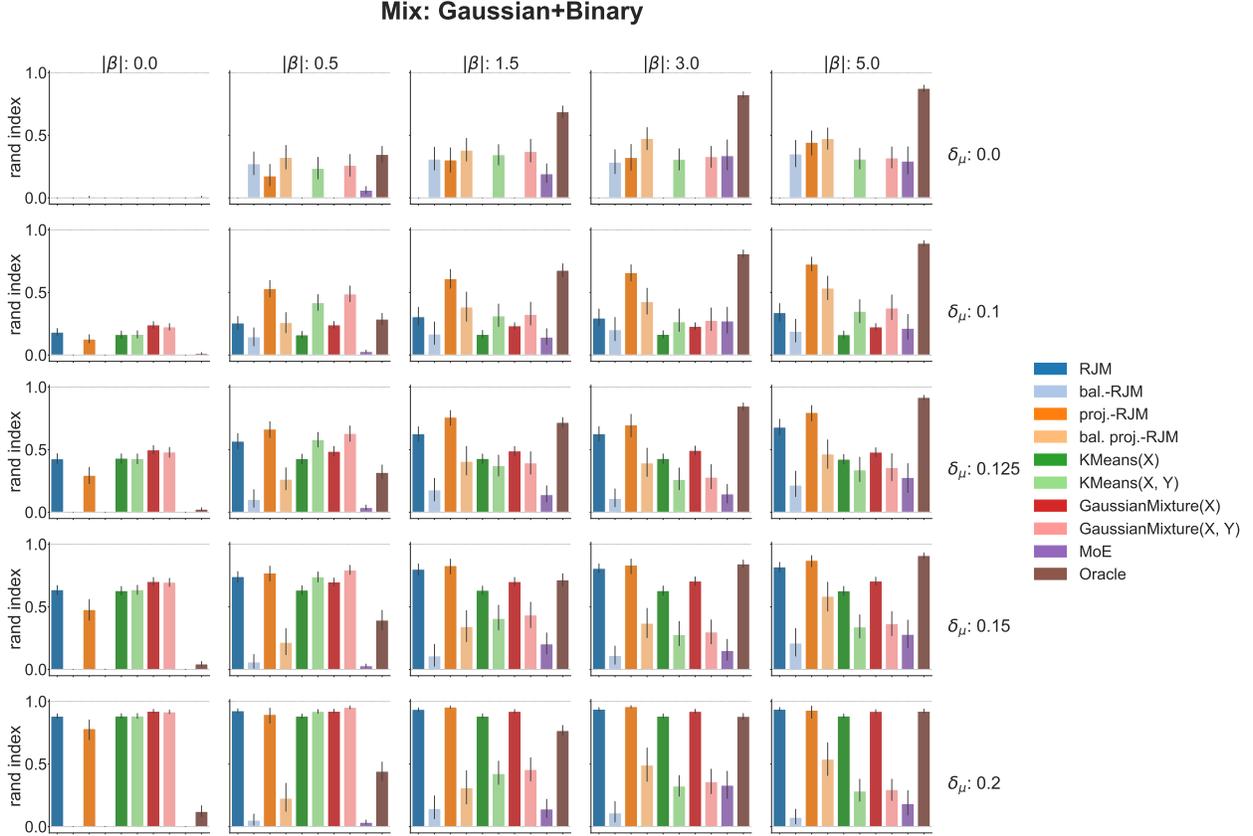}
    \caption{\textbf{Mixed data Rand Index.} Rand Index for different forms and intensities of signals when the data is made of a mix of Gaussian ($20\%$) and binary ($80\%$) variables. Despite their mixed nature, the $p$ features are described by one underlying precision matrix $\Omega_k$. The specific distribution is described in \cite{fan2017high}. The precision matrix are different: $\Omega_1 \neq \Omega_2$. The data $x$ is as a consequence very non Gaussian. This may explain why the Oracle method is under-performing a lot when there is no signal in $y$. Since the Oracle uses the latent $\bOmega$ to build a Gaussian model based classifier. When the proportion of binary features in the mix increases, the actual data becomes more and more different from $\sum_k \pi_k \N(\mu_k, \Omega_k^{-1})$, hence the prior information of the Oracle is less and less relevant. By contrast, the other methods have more flexibility since they can use an actual empirical estimation of the covariance matrices. When the signal in $y$ increases, the prior knowledge of the true $\bbeta$ gets the Oracle a strong advantage, and it becomes dominant again. Regarding the respective performances of the actual clustering methods, they are quite similar to the ones on Gaussian data as seen on \figurename~\ref{fig:grid_signal}. The regularised RJM (bal., proj. and bal. + proj.) are dominant when the signal is mostly in $y$. Regular ambient RJM and the classical KMeans and MoG are best when the signal is in $x$ only. Projected RJM is always good.}
    \label{fig:mixed_binary_rand_index}
\end{figure}

\begin{figure}[tbhp]
    \centering
    \includegraphics[width=\linewidth]{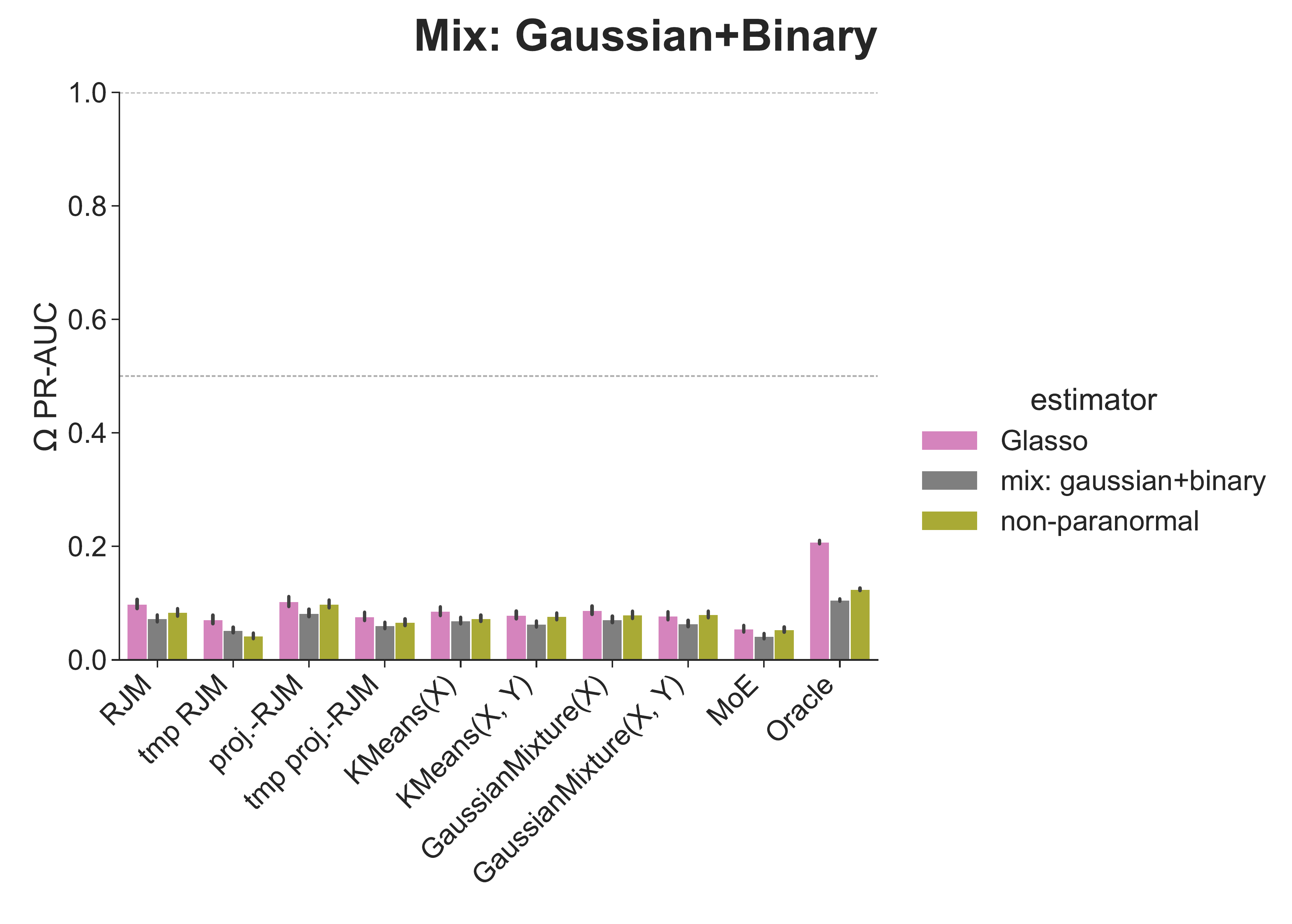}
    \caption{\textbf{Mixed data PR AUC.} PR AUC of the sparsitency in the precision matrix $\bOmega$. For each latent subgroup, the data is generated as a mix of Gaussian and binary data with an underlying matrix $\Omega_k$. We compare three different graph estimators always computed from the same (estimated or oracle) group-labels. The AUc are much worse than on \figurename~\ref{fig:student_pr_auc}. This is because the estimators feature too many False Positives in a unbalanced environment. Indeed, the True $\bOmega$ is very sparse ($\sim 4$ edges in a graph with $p=100$ nodes).}
    \label{fig:mixed_binary_pr_auc}
\end{figure}

\newpage
\section{Non-Gaussian biomedical data}\label{app:exp_tcga}
In this appendix, we present alternative versions of \figurename~\ref{fig:TCGA_vs_n_K4} (Rand Index vs sample size) with different combinations of cell types (latent subgroups) in the dataset (figures \ref{fig:TCGA_vs_n_K2} to \ref{fig:TCGA_vs_n_K3_no_brain}). Then, we explore two alternative settings for the experiment of \figurename~\ref{fig:TCGA_vs_p_K4_one_coeff} (Rand Index vs ambient dimension, figures \ref{fig:TCGA_vs_p_K4} and \ref{fig:TCGA_vs_p_K4_fixed_fraction}).

\begin{figure}[tbhp]
    \centering
    \includegraphics[width=\linewidth]{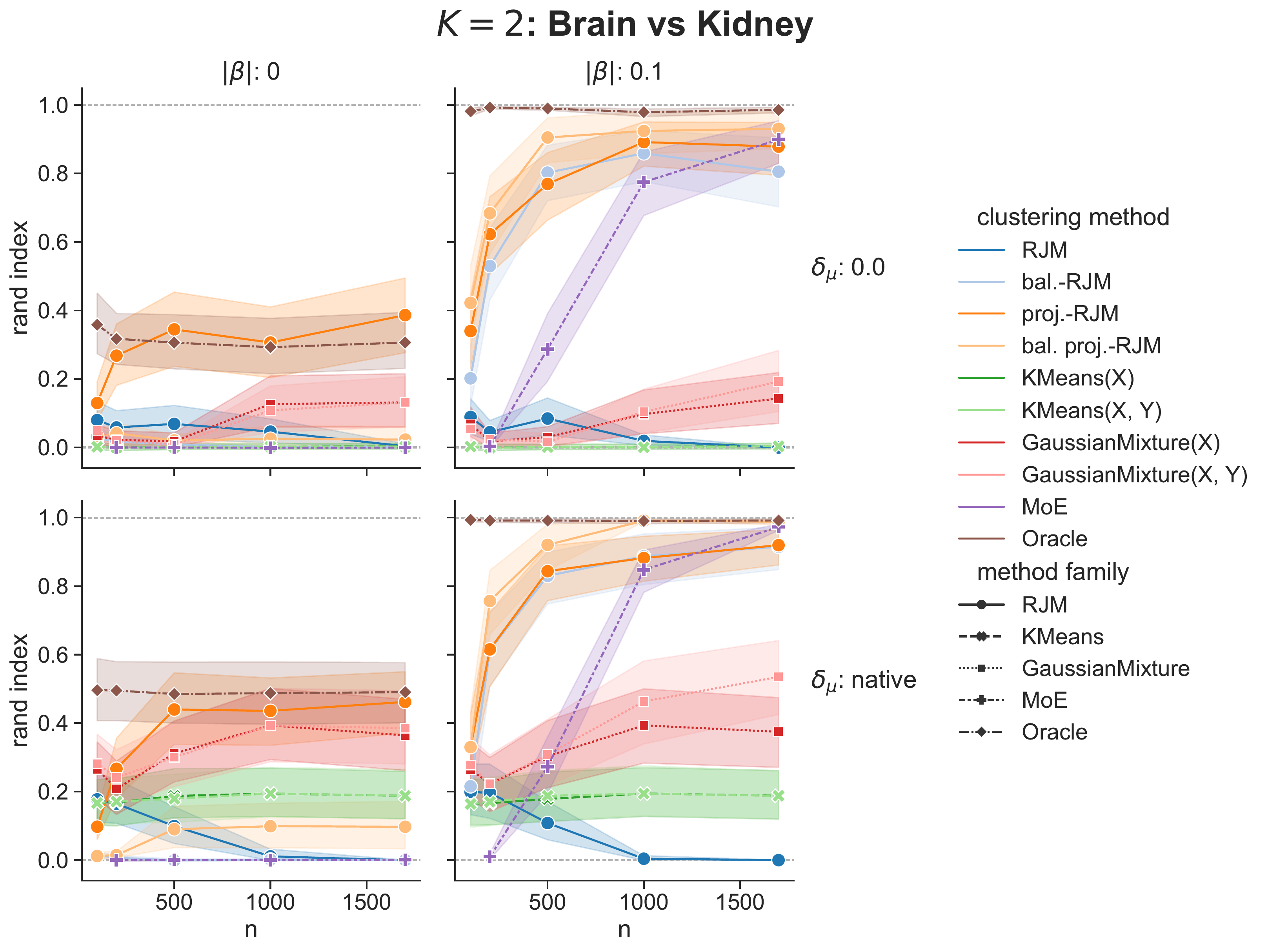}
    \caption{Rand Index against sample size $n$ with real RNA sequencing data from TCGA. Each cell comes from either \textbf{brain or kidney} tissue. The RNA count data plays the part of $x$. In order to assess the strength of covariance signal in $x$, we can remove the mean signal (top row, $\delta_{\mu} = 0$). We also add a synthetic variable $y$ defined from the real $x$ and an artificial $\bbeta$. When $|\beta|=0$, $y$ is pure noise and independent of the latent subgroup. Projected RJM is the best method on the raw data, close to Oracle when $n$ increases. Projected, bal. and bal. proj.-RJM are good when there is group-relevant signal in $\beta$.}
    \label{fig:TCGA_vs_n_K2}
\end{figure}

\begin{figure}[tbhp]
    \centering
    \includegraphics[width=\linewidth]{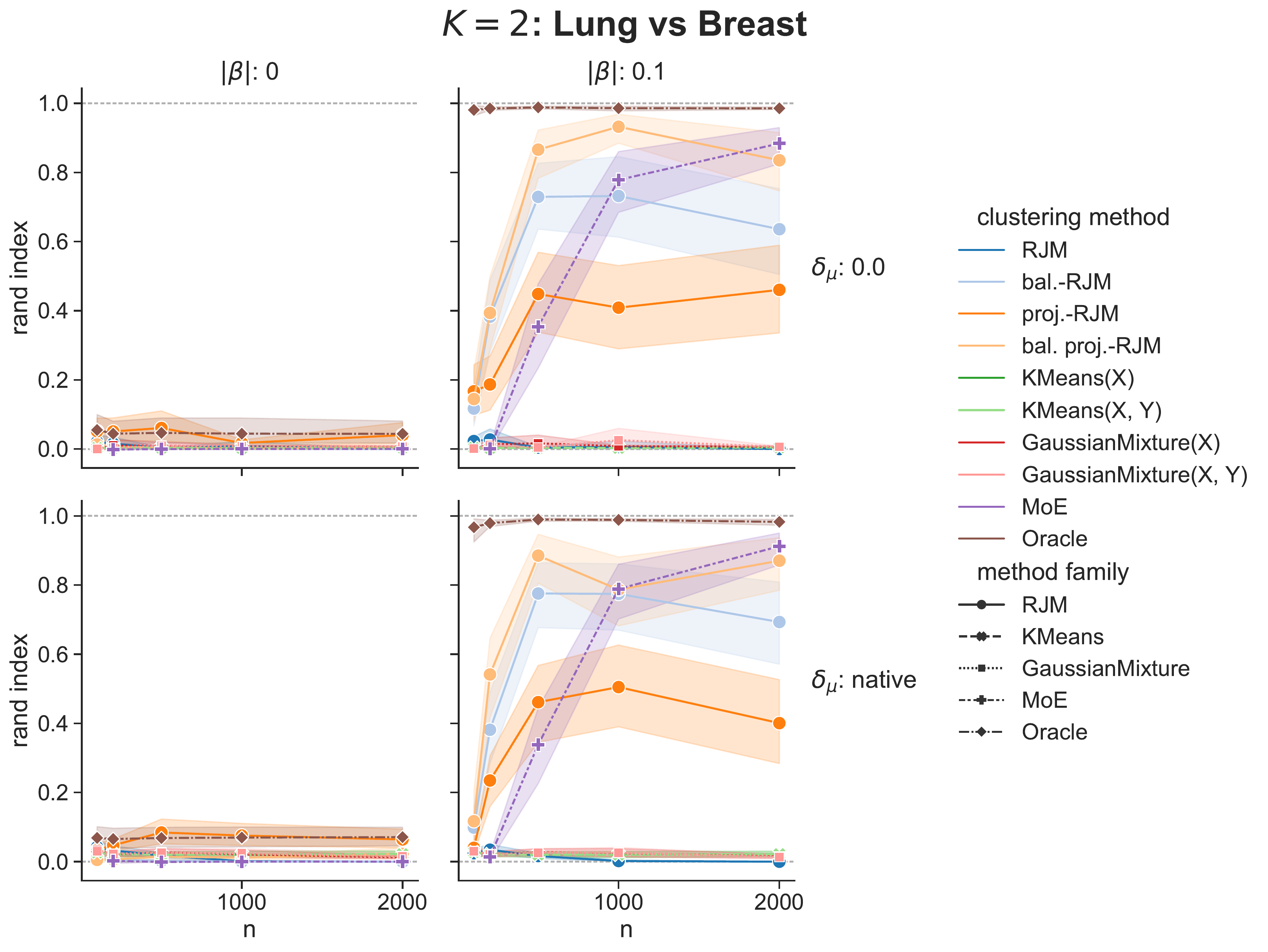}
    \caption{Rand Index against sample size $n$ with real RNA sequencing data from TCGA. Each cell comes from either \textbf{bronchus-lung or breast} tissue. The performances are worse than in \figurename~\ref{fig:TCGA_vs_n_K2} because the Brain tissue cells, the most distinct latent subgroup of the four, are absent in this experiment.}
    \label{fig:TCGA_vs_n_K2_no_brain}
\end{figure}

\begin{figure}[tbhp]
    \centering
    \includegraphics[width=\linewidth]{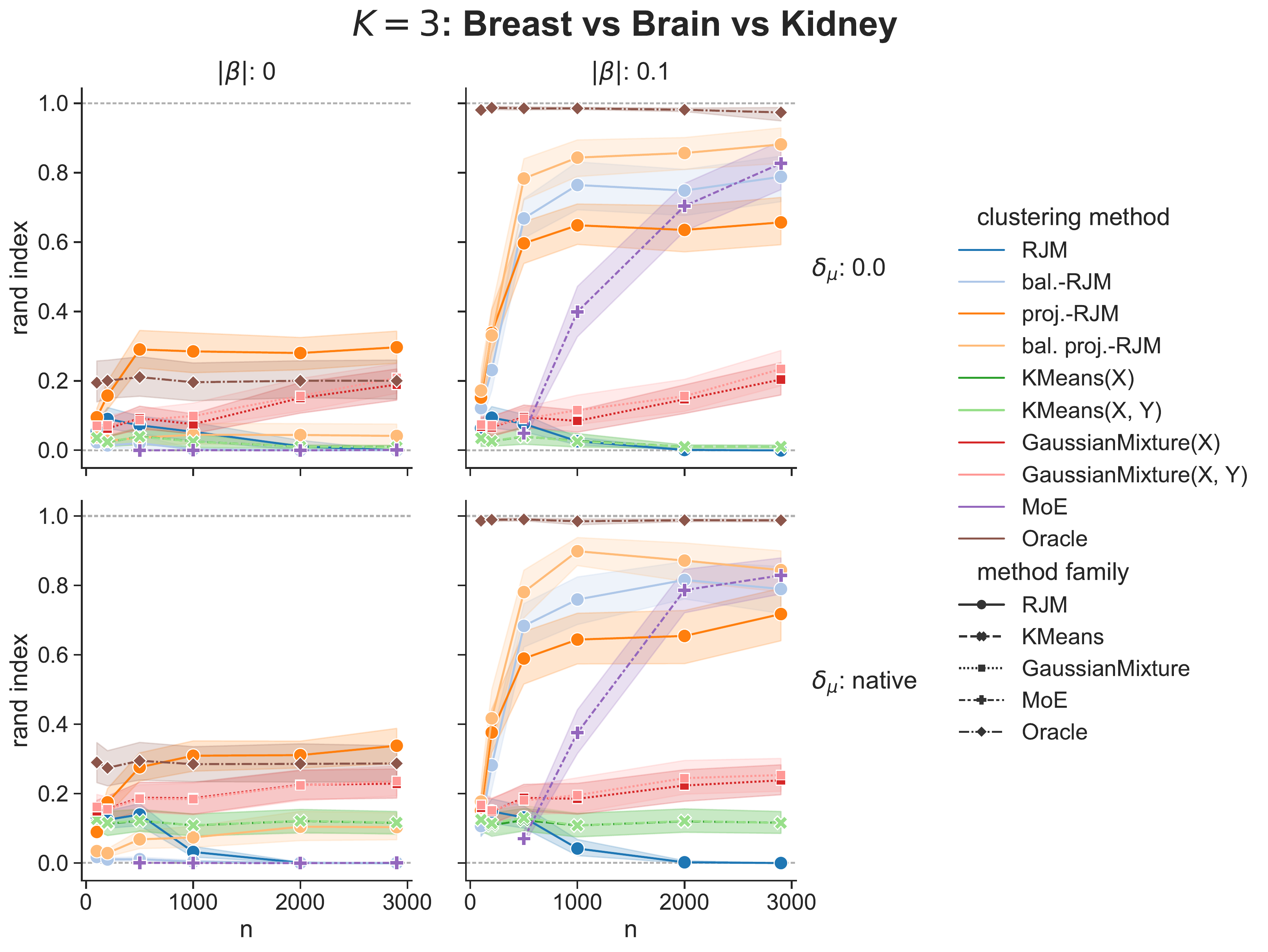}
    \caption{Rand Index against sample size $n$ with real RNA sequencing data from TCGA. Each cell comes from either \textbf{breast, brain or kidney} tissue. The results are similar to \figurename~\ref{fig:TCGA_vs_n_K2}, where $K=2$ and the brain cells are part of the dataset, with overall lower Rand Index, especially when there is no signal in $y$.}
    \label{fig:TCGA_vs_n_K3}
\end{figure}

\begin{figure}[tbhp]
    \centering
    \includegraphics[width=\linewidth]{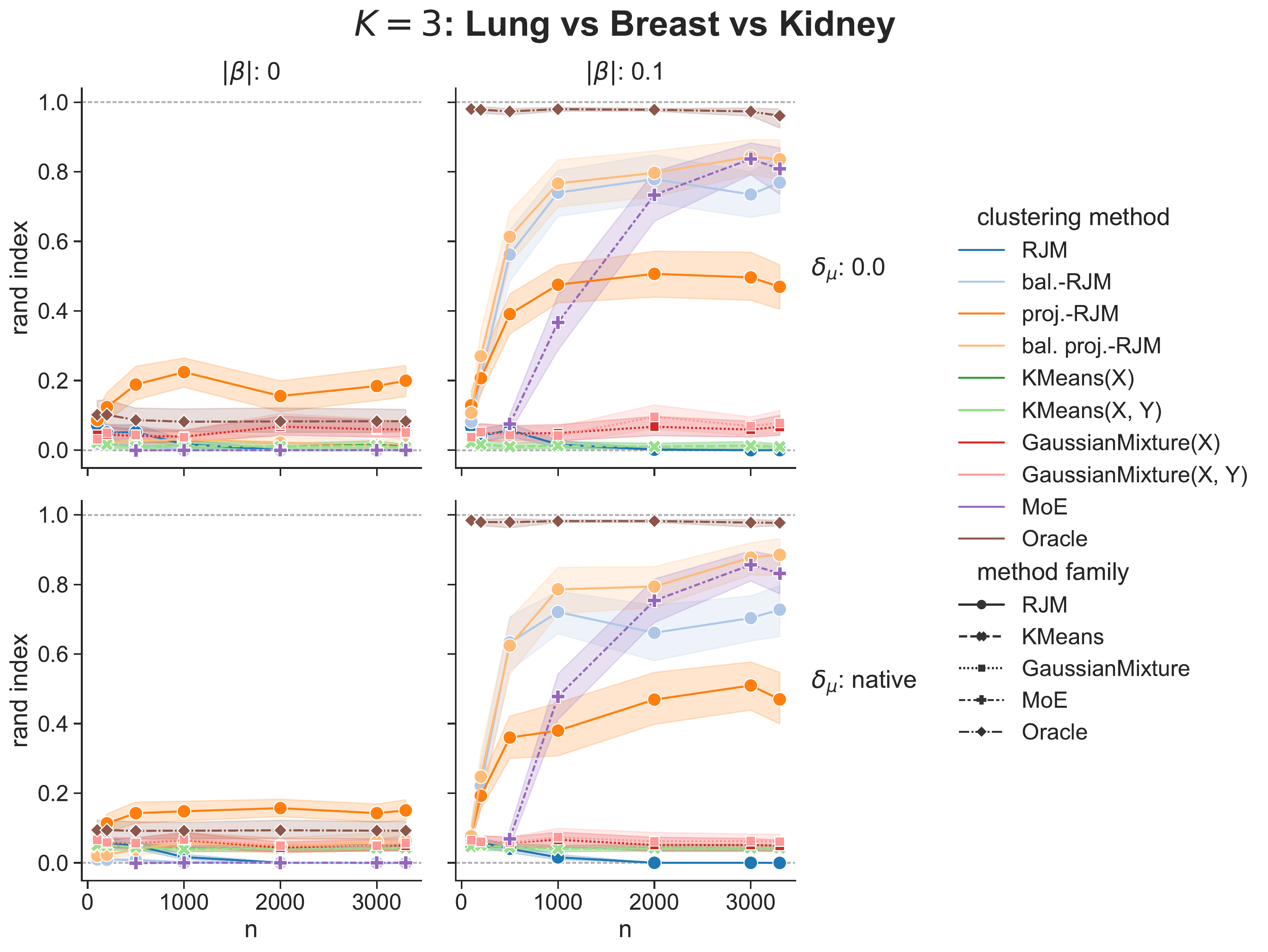}
    \caption{Rand Index against sample size $n$ with real RNA sequencing data from TCGA. Each cell comes from either \textbf{bronchus-lung, breast or kidney} tissue. In this case, proj.-RJM is clearly better than the Oracle when there is no signal in $|\beta|$. When there is signal in $|\beta|$, the two balanced RJM dominate, with proj.-RJM below, and MoE catching up when the sample size gets very large.}
    \label{fig:TCGA_vs_n_K3_no_brain}
\end{figure}

\begin{figure}[tbhp]
    \centering
    \includegraphics[width=\linewidth]{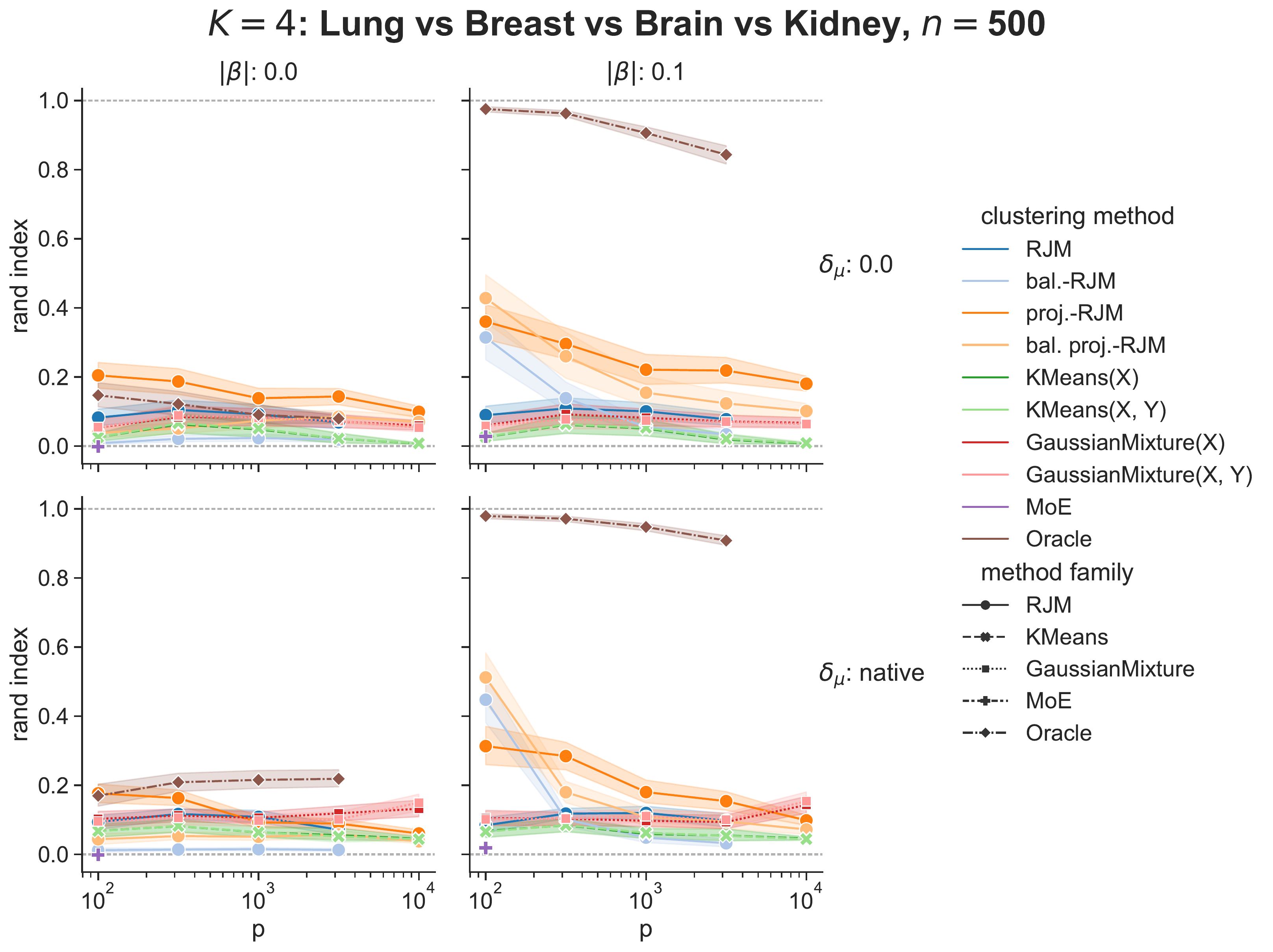}
    \caption{Rand Index against ambient space size $p$ with real RNA sequencing data from TCGA. The x-axis has a logarithmic scale. All four tissues are used. \textbf{For all $p$, only 10 non zero coefficients in $\beta_k$}. The absolute value of the non zero coefficient in $\bbeta$ is $|\beta|/10$. As in \figurename~\ref{fig:TCGA_vs_p_K4_one_coeff}, the regularised RJM methods mostly dominate.}
    \label{fig:TCGA_vs_p_K4}
\end{figure}

\begin{figure}[tbhp]
    \centering
    \includegraphics[width=\linewidth]{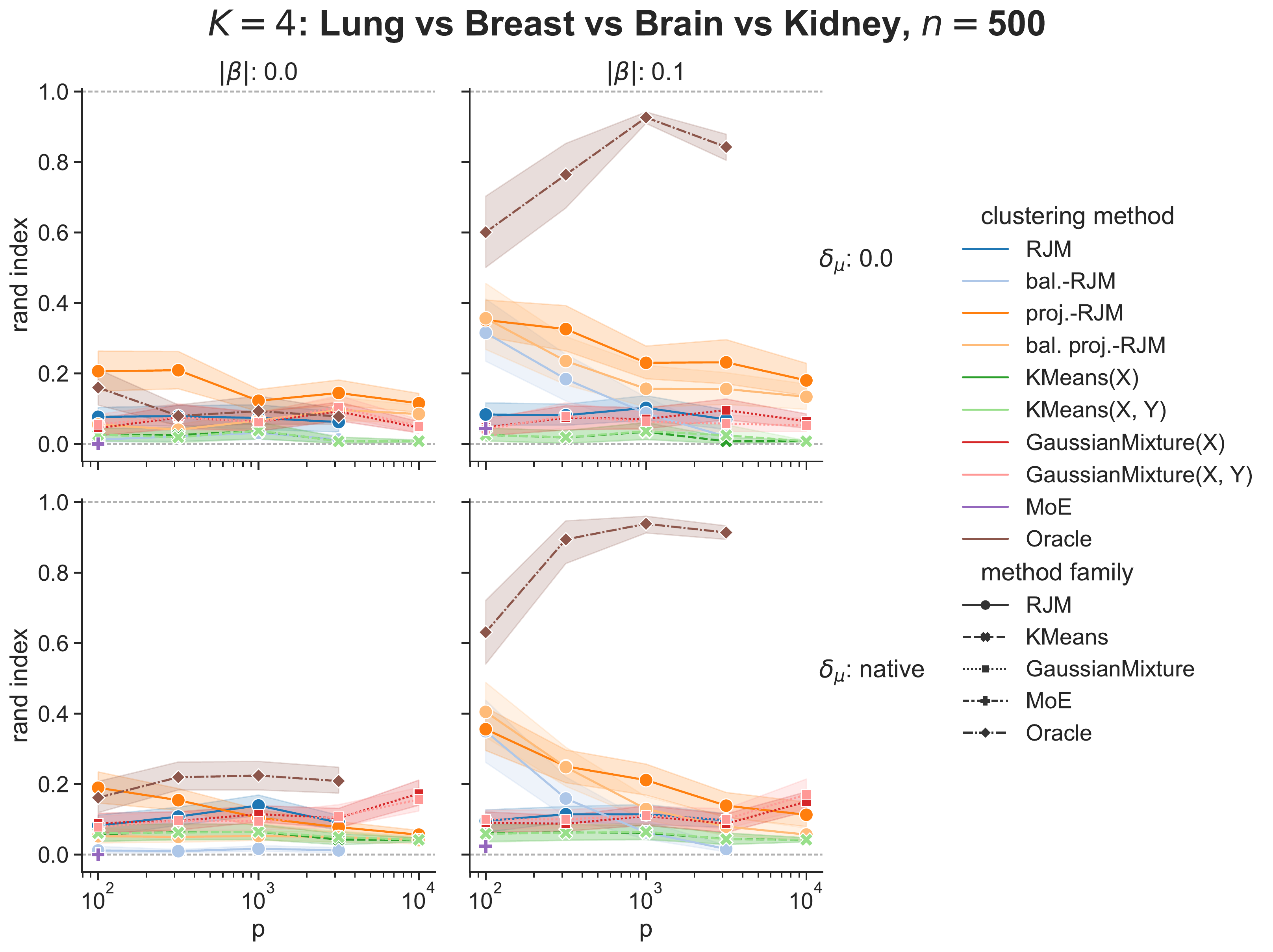}
    \caption{Rand Index against ambient space size $p$ with real RNA sequencing data from TCGA. The x-axis has a logarithmic scale. All four tissues are used. \textbf{For all $p$, $1\% p$ non zero coefficients in $\beta_k$}. The absolute value of the non zero coefficient in $\bbeta$ is $|\beta|/(1\% p)$. As in \figurename~\ref{fig:TCGA_vs_p_K4_one_coeff}, the regularised RJM methods mostly dominate.}
    \label{fig:TCGA_vs_p_K4_fixed_fraction}
\end{figure}

\newpage

\section{Fully adaptive S-RJM}\label{app:exp_mod_sel}
In this appendix, we provide additional figures (\ref{fig:rand_index_vs_q} to \ref{fig:K_q_choice_vs_pen_IC}) to expand on the analysis of the model selection experiments (Section \ref{sec:exp_mod_sel}) of the main paper.

\paragraph{Understanding the effect of $q$ and $T$.}
Before showcasing additional figures from the experiments of Section \ref{sec:exp_mod_sel}, we propose a brief analysis of the effect of the projection parameter $q$ and the balancing parameter $T$ on RJM. We explore grids of values for both these parameters and observe the resulting Rand Indices after running RJM. These experiments are made with synthetic Gaussian data, $K=2$ latent subgroups, where we control the amount of signal in $x$ and $y$ through the mean difference $\delta_{\mu}$ and the intensity $|\beta|$ of the non zero coefficients in $\beta$. Additionally, the precision matrices $\Omega_k$ are different between latent subgroups. On \figurename~\ref{fig:rand_index_vs_q}, we observe the Rand Index as a function of the embedding size $q$. For this experiment, we take a somewhat large ambient space size $p=500$, such that larger embedding sizes could reasonably be considered. The Oracle is defined from the knowledge of the true model parameter hence, like the ambient space RJM, it is independent of $q$. proj.-RJM is not expected to converge towards ambient RJM as $q$ grows since ambient RJM has an $l_2$ regularisation on $\widehat{\Omega}$ at each M-step whereas ambient RJM does not. We can see that there usually is an optimal embedding size $q^*$ different from $p$ or 1, which suggests the existence of a trade-off between regularisation and loss of information in the projection, see \cite{lartigue2022unsupervised} for an exploration of this question. On \figurename~\ref{fig:rand_index_vs_T}, we observe the Rand Index as a function of the balancing exponent $T$. In this example, we take $T = q^{\upsilon}$ and vary the value of $\upsilon$, hence $\upsilon=\frac{\ln T}{\ln q}$ is the natural x-axis. Negative $\upsilon$ means that $x$ weights more than usual in the E-step, positive $T$ means the opposite. Since $x$ is of dimension $p=100$ for RJM and $q=5$ for proj.-RJM, $x$ ``weights" natively more than $y$ when $T=1$ (no re-balancing). The balancing can have a huge impact on the performances. The obvious trend apparent on \figurename~\ref{fig:rand_index_vs_T} is that reducing the weight of $x$ ($T>0$) improves the performances when the prevalent signal is in $y$ and vice versa. Other observations: even though the orange and blue curves share a similar shape overall, they exhibit crucial difference. On the two rightmost figures of the first row (high signal in $y$), with no re-balancing ($T=1$), the ambient RJM has an approximately 0 Rand Index, whereas, even without any re-balancing, the projected RJM is at 0.45 (left figure) or 0.65 (right figure). Then, as $T$ increases, the performance of both RJM get better, with the ambient RJM ``catching up" with its projected version. This suggest that, with prevalent signal in $y$, some form of regularisation is needed for good performances, and that one or the other (projection or balancing) is enough to see a performance improvement. With both yielding the best results as evidenced by the higher Rand Index achieved by the projected RJM for high values of $T$.
\begin{figure}[tbhp]
    \centering
    \includegraphics[width=\linewidth]{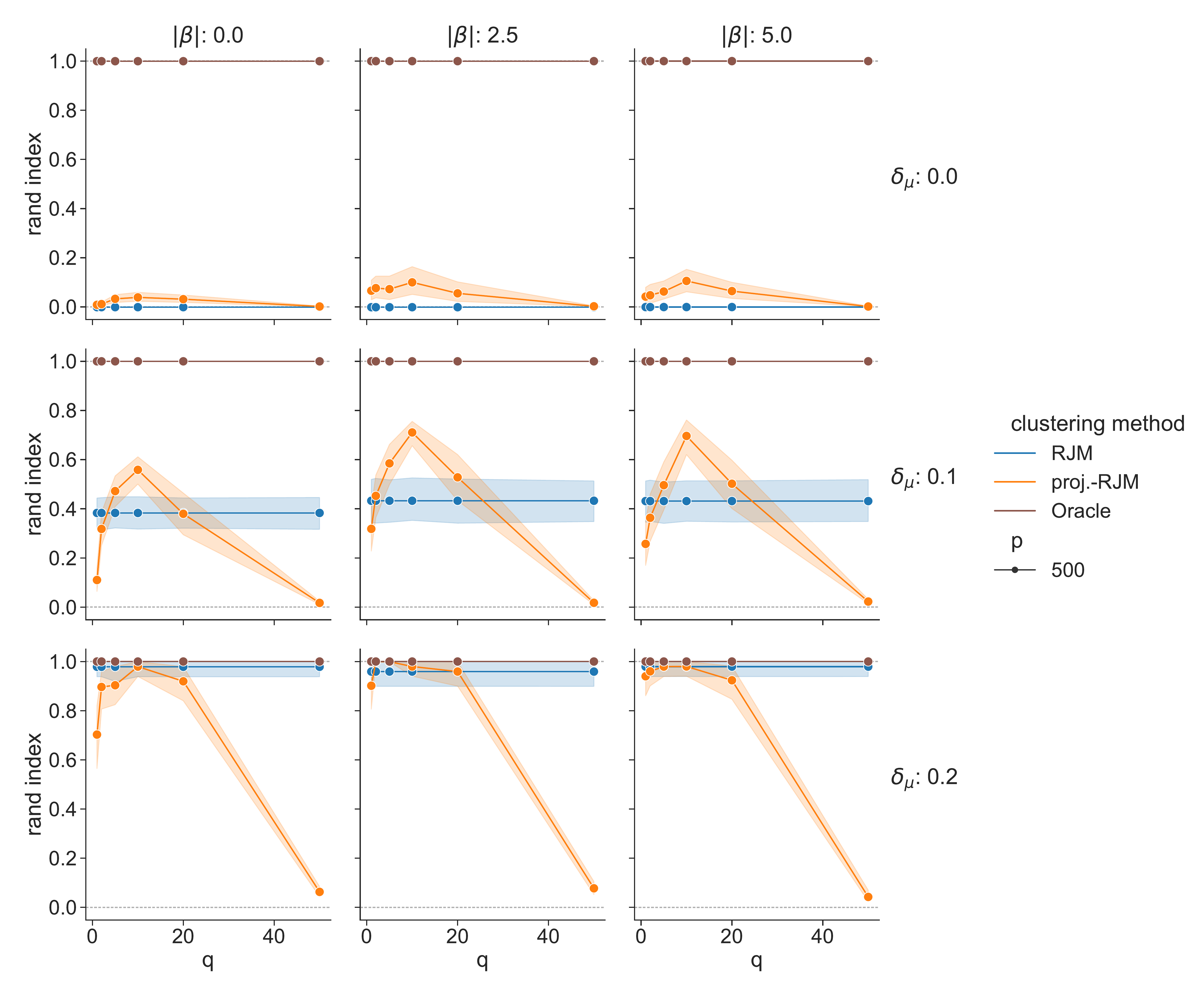}
    \caption{Rand Index vs the embedding space size, over a grid of different signal types and intensities. The true covariance matrices are different in each subgroup. The optimal embedding size $q$ is most often somewhere in between the extremes $q=1$ and $q=p$.}
    \label{fig:rand_index_vs_q}
\end{figure}

\begin{figure}[tbhp]
    \centering
    \includegraphics[width=\linewidth]{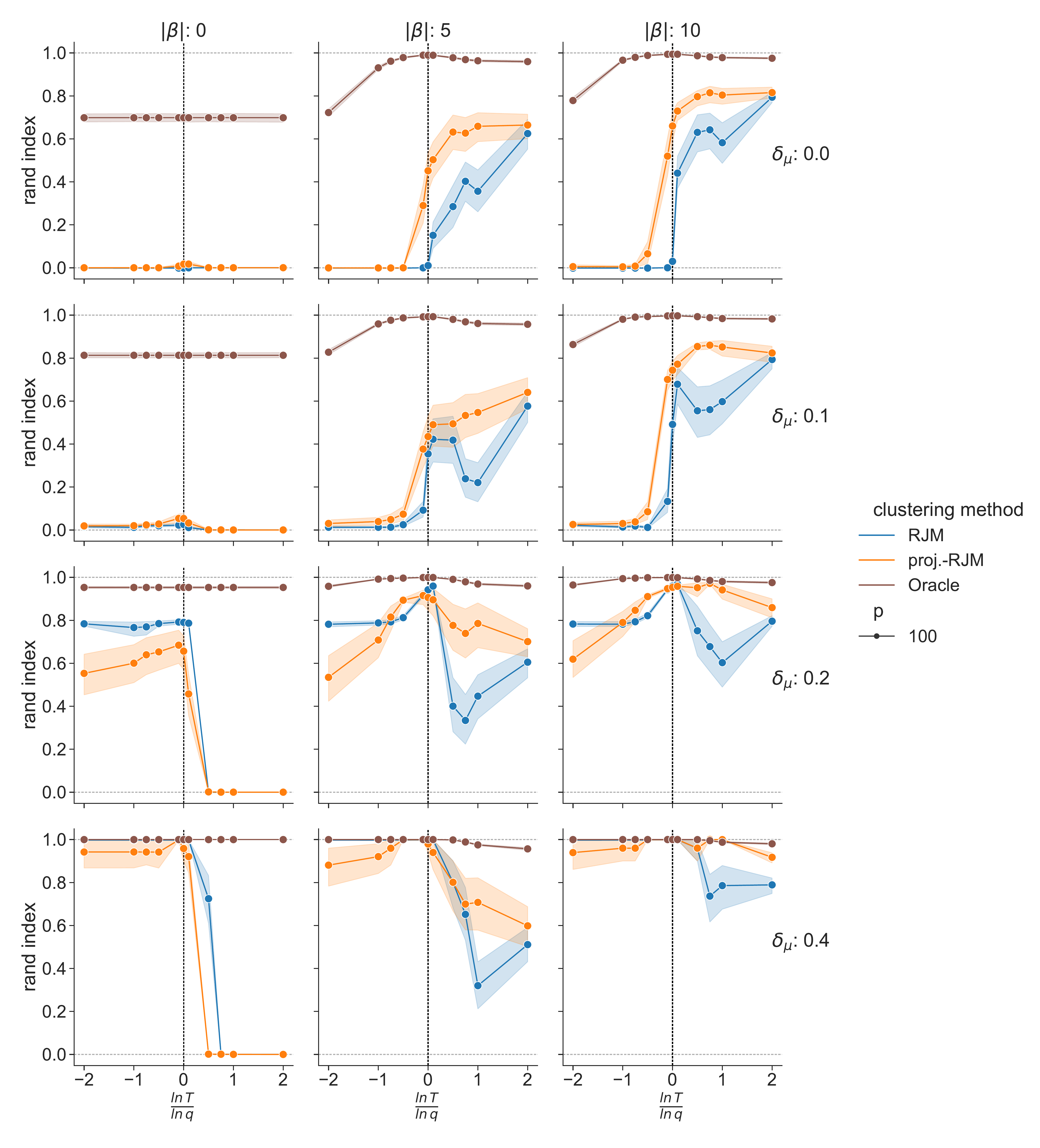}
    \caption{Rand Index vs the balancing of $x$. We take $T=q^{\upsilon}$, and set the exponent $\upsilon$ to be the x-axis. Depending on the nature of the signal, the correct re-balancing of blocs $x$ and $y$ can drastically improve the performances.}
    \label{fig:rand_index_vs_T}
\end{figure}

\newpage
\paragraph{Model selection procedure.} In the following, we provide additional figures from the experiment of Section \ref{sec:exp_mod_sel} to illustrate and analyse our model selection procedure, which sets $(q, K)$ in an adaptive way.

On \figurename~\ref{fig:rand_index_vs_K_q}, we observe the Rand Index of proj.-RJM for each combination of $(q, K)$ over our grid of values. These Rand Indices are averaged over the 50 simulations. This means that, despite the better average Rand Index being obtained with $q=5, K=5$, this combination $q, K$ is not necessarily the best for every simulation. An adaptive method, able to fine tune $q, K$ for each simulation could potentially achieve even higher performances. Of course, an adaptive method does not have access to these oracle Rand Indices, and instead bases its decision on empirical heuristics.

\begin{figure}[tbhp]
    \centering
    \includegraphics[width=\linewidth]{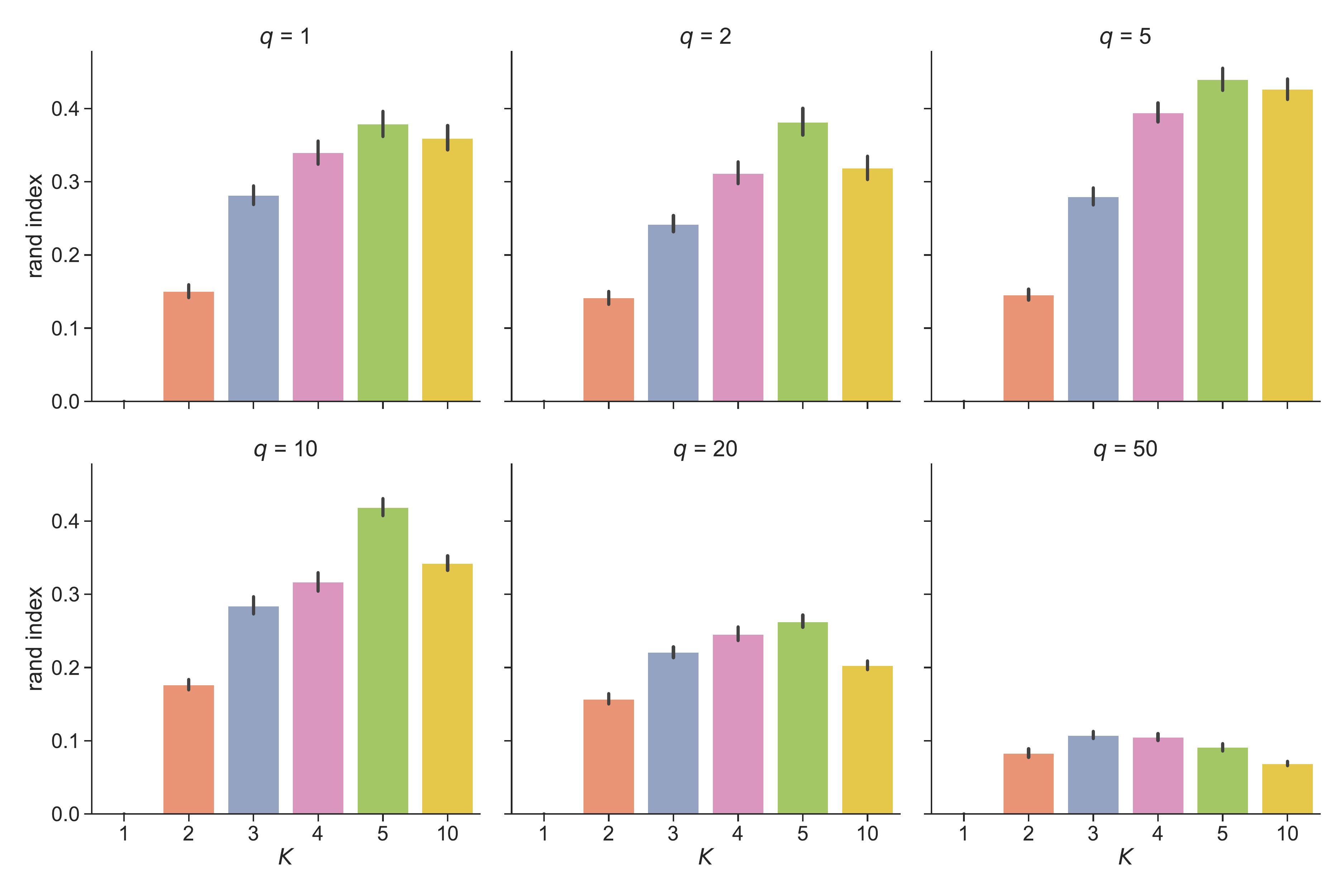}
    \caption{Average Rand Index of proj.-RJM for each combination of $K$ and $q$. We can see that the better average Rand Index is obtained with $q=5, K=5$. Other combinations around this one also have decent average Rand Indices.}
    \label{fig:rand_index_vs_K_q}
\end{figure}

Our procedure first selects a $\widehat{K}(q)$ for each $q$ according to an Information Criterion (IC), here either AIC or BIC. Then choose a pair $(q, \widehat{K}(q)$ according to a stability criterion. On \figurename~\ref{fig:K_choice}, we visualise the first step: the choice of $\widehat{K}(q)$ by the IC for each value of $q$. We display the histogram of these choices over the 50 simulation. Then on \figurename~\ref{fig:K_q_choice}, we visualise the final choice of the pair $(\widehat{q}, \widehat{K}(\widehat{q})$, with the stability criterion taken into account. On the \figurename, we associate to each pair $(q, k)$ the number of times it has been chosen by the model selection procedure, over the 50 simulations. As expected, BIC is more conservative and often picks smaller values of $K$. However, the most common choice with AIC is $(\widehat{q}, \widehat{K}(\widehat{q}) = (5, 5)$, the combination which we know from \figurename~\ref{fig:rand_index_vs_K_q} to have the highest average Rand Index. The other pairs often chosen by AIC also have good average Rand Index. This is good sign for the model selection procedure with AIC.

\begin{figure}[tbhp]
    \centering
    \includegraphics[width=\linewidth]{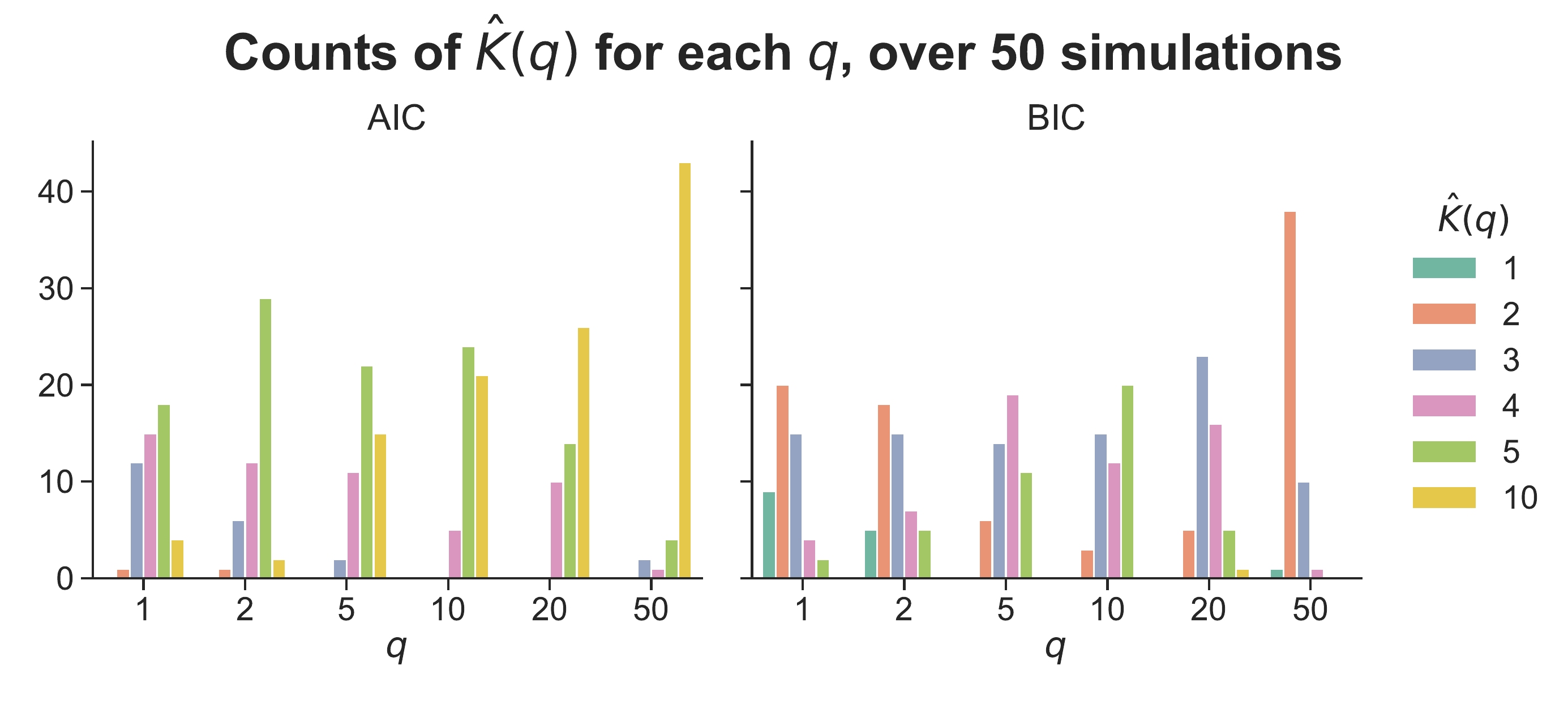}
    \caption{For each value of the embedding space size $q$, number of time each value of $K$ is picked among the candidates by the Information Criterion over the 50 simulations.}
    \label{fig:K_choice}
\end{figure}

\begin{figure}[tbhp]
    \centering
    \includegraphics[width=\linewidth]{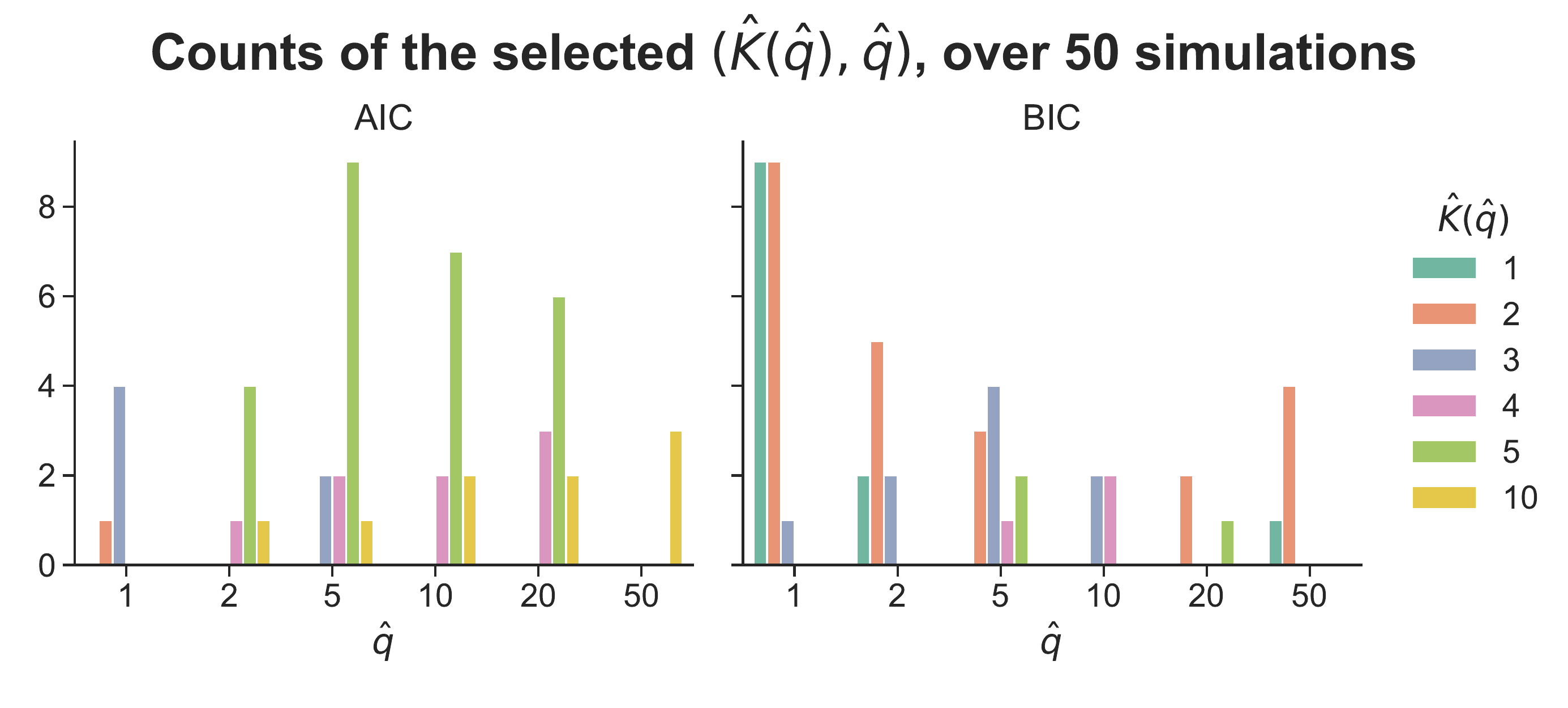}
    \caption{Number of time each pair of $(q, K)$ is picked among the candidates by the Information Criterion over the 50 simulations. BIC is conservative and ofren selects $K=1$. By crossing these results with \figurename~\ref{fig:rand_index_vs_K_q}, we can see that AIC selects models with good average Rand Index.}
    \label{fig:K_q_choice}
\end{figure}

To confirm this, we display on \figurename~\ref{fig:rand_index_vs_pen_IC} the Rand Index realised by projected RJM with $(\widehat{q}, \widehat{K}(\widehat{q}))$ selected adaptively for each simulation. 

In this \figurename, we consider variants of the AIC and BIC where a factor $\eta$ controls the amount of penalisation added by the criterion:
\begin{equation*}
    \text{AIC} = - 2\ln p(\by, \bx ; \widehat{\btheta}) + 2  \eta df  \, ,
\end{equation*}
\begin{equation*}
    \text{BIC} =- 2 \ln p(\by, \bx ; \widehat{\btheta}) + \eta df \ln(n) \, .
\end{equation*}
Up until this point, we have been using the default value $\eta = 1$, corresponding to the theoretical AIC and BIC. The addition of $\eta$ allows the user to exert more control on the selection criterion, at the cost of having a new hyper-parameter to fine-tune. From \figurename~\ref{fig:rand_index_vs_pen_IC}, we conclude that we can confidently keep using the default value $\eta=1$ for AIC.
\begin{figure}[tbhp]
    \centering
    \includegraphics[width=\linewidth]{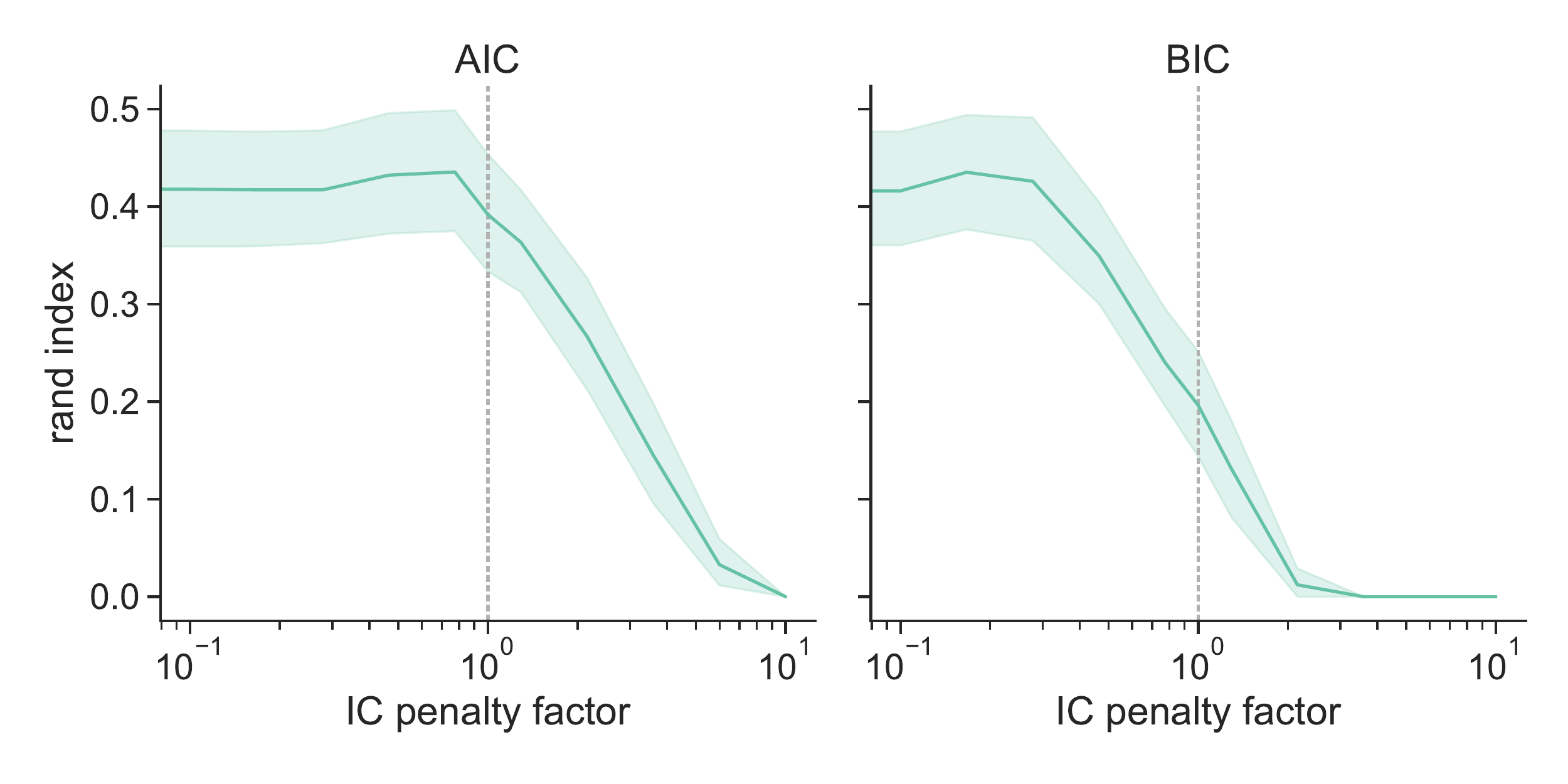}
    \caption{Rand Index of proj.-RJM with adaptive $K, q$ vs the multiplicative factor $\eta$ of the IC.}
    \label{fig:rand_index_vs_pen_IC}
\end{figure}

It may seem from looking at \figurename~\ref{fig:rand_index_vs_pen_IC} that using an IC at all is superfluous since a good Rand Index can be reached with low penalisation. However, we showcase on \figurename~\ref{fig:K_q_choice_vs_pen_IC} how the choice of $(\widehat{q}, \widehat{K}(\widehat{q})$ is affected by the penalty factor. On this \figurename, it is apparent that the presence and intensity of the AIC penalty has an impact on the selected models. In particular, with low AIC penalty, the largest possible $K$ ($K=10$ here) is the most selected. Whereas the selection is more skewed towards $K=3, 4$ or 5 with the regular AIC (factor = 1). This is of course the intended effect of using an information criterion such as AIC in the first place, and serves as a confirmation that, even though the Rand Index is not too different, the model selection is better behaved with AIC than without.

The results of \figurename~\ref{fig:K_q_choice_vs_pen_IC} may also seem paradoxical since, normally in hierarchical models, selection of the number of latent subgroups based on likelihood alone (when the penalty factor $\eta = 0$) should systematically favour the higher number of subgroups possible ($K=10$ here). Yet, this is not what we observe here. This is because the likelihood is non-convex and imperfectly optimised by the EM algorithm. Hence, even with the same data, some runs of the optimiser with lower values of $K$ can sometimes reach higher likelihood values than the run with $K=10$.\\
\\
For the culmination of this experiment, the Rand Index results of all RJM variants with adaptive selection of $(q, K)$, compared to all the other clustering methods, are displayed and discussed on \figurename~\ref{fig:model_selection_full_rand_index}, in the main body of the paper.

\begin{figure}[tbhp]
    \centering
    \includegraphics[width=\linewidth]{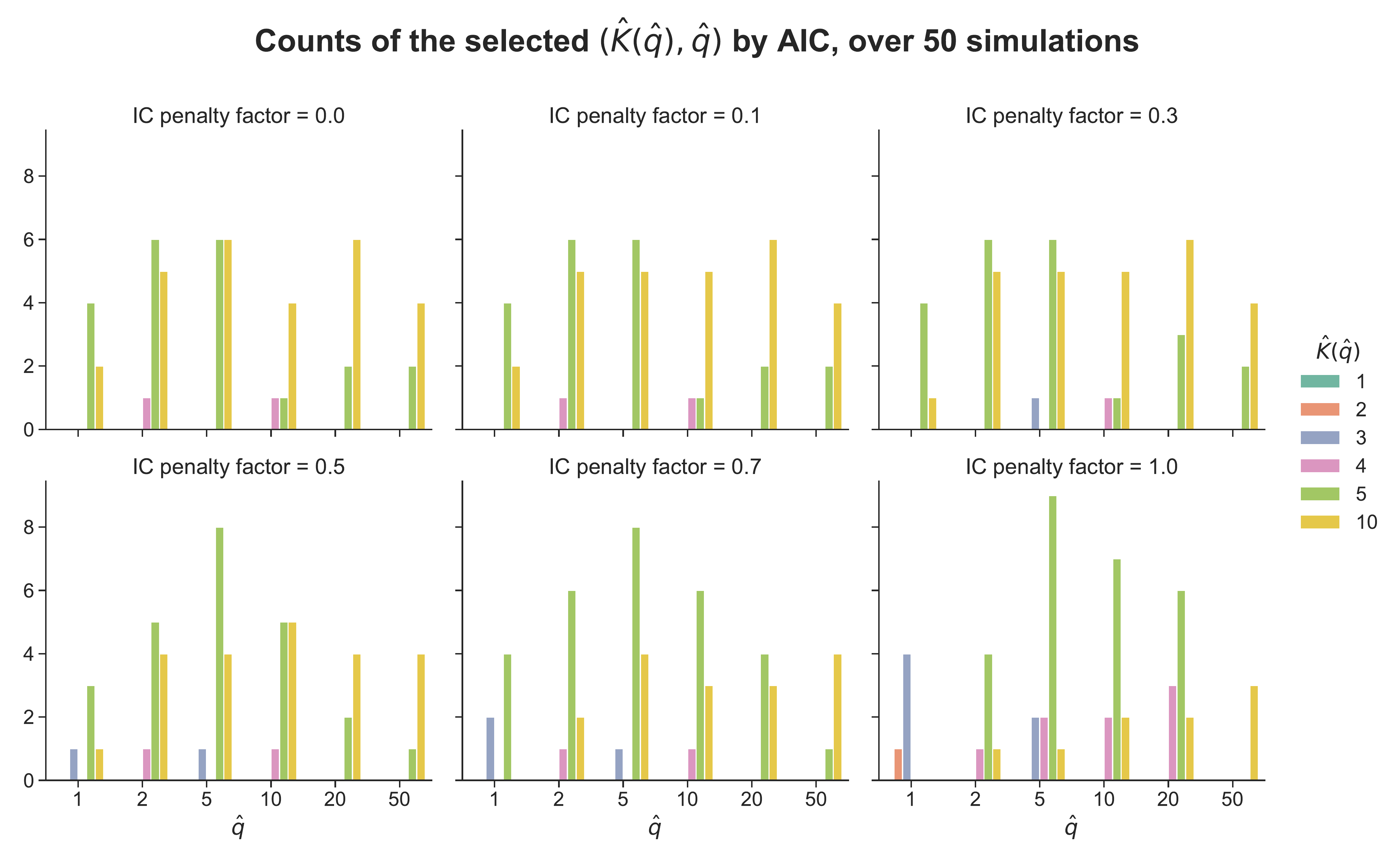}
    \caption{Choice of $K, q$ by the adaptive proj.-RJM vs the multiplicative factor $\eta$ of the IC. A factor $\eta= 1$ is the regular AIC, while a factor $\eta=0$ corresponds to a purely likelihood-based selection of $\widehat{K}(q)$. As we can see, the factor of the AIC penalty affects the selected model. With a factor of 1, the maximum authorised value $K=10$ is picked much less often.}
    \label{fig:K_q_choice_vs_pen_IC}
\end{figure}

\vskip 0.2in
\bibliography{bibliography}

\end{document}